\documentclass[a4paper]{article}
\usepackage{geometry}

\usepackage{epsfig}
\usepackage{algorithm}
\usepackage{algorithmic}
\usepackage{multirow}
\usepackage{graphicx}
\usepackage[cmex10]{amsmath}
\usepackage{amssymb}
\usepackage{amsthm}
\usepackage[tight,footnotesize]{subfigure}
\usepackage{color}

\newcommand{\tb}[1]{\textbf{#1}}
\usepackage{multirow}
\usepackage{appendix}
\usepackage[noblocks]{authblk}
\usepackage{array}
\usepackage{times}

\newcommand{\tabincell}[2]{\begin{tabular}{@{}#1@{}}#2\end{tabular}}
\newtheorem{theorem}{Theorem}

\newtheorem{proposition}[theorem]{Proposition}
\newtheorem{corollary}[theorem]{Corollary}
\newtheorem{definition}[theorem]{Definition}

\newtheorem*{l0}{Proposition 7}
\newtheorem*{l1}{Proposition 8}

\newtheorem*{en}{Proposition 11}
\newtheorem*{evdmin}{Theorem 12}
\newtheorem*{evcos}{Theorem 13}

\begin{document}
\title{Sparse Principal Component Analysis via Rotation and Truncation}

\author[*]{Zhenfang~Hu}
\author[*]{Gang~Pan}
\author[$\dag$]{Yueming~Wang}
\author[*]{Zhaohui~Wu}

\affil[*]{College of Computer Science and Technology, Zhejiang
University, China, \authorcr zhenfhu@gmail.com, gpan@zju.edu.cn,
zjuwuzh@gmail.com}

\affil[$\dag$]{Qiushi Academy for Advanced Studies, Zhejiang
University, China, \authorcr ymingwang@zju.edu.cn}

\maketitle

\begin{abstract}
Sparse principal component analysis (sparse PCA) aims at finding a
sparse basis to improve the interpretability over the dense basis of
PCA, 
{meanwhile the sparse basis should cover the data subspace as much
as possible}. In contrast to most of existing work which deal with
the problem by adding some sparsity penalties on various objectives
of PCA, in this paper, we propose a new method SPCArt, whose
motivation is to find a rotation matrix and a sparse basis such that
the sparse basis approximates the basis of PCA after the rotation.
The algorithm of SPCArt consists of three alternating steps: rotate
PCA basis, truncate small entries, and update the rotation matrix.
Its performance bounds are also given. SPCArt is efficient, with
each iteration scaling linearly with the data dimension. It is easy
to choose parameters in SPCArt, due to its explicit physical
explanations. Besides, we give a unified view to several existing
sparse PCA methods and discuss the connection with SPCArt. Some
ideas in SPCArt are extended to GPower, a popular sparse PCA
algorithm, to overcome its drawback. Experimental results
demonstrate that SPCArt achieves the state-of-the-art performance.
It also achieves a good tradeoff among various criteria, including
sparsity, explained variance, orthogonality, balance of sparsity
among loadings, and computational speed.
\end{abstract}



\section{Introduction}
In many research areas, the data we encountered are usually of very
high dimensions, for examples, signal processing, machine learning,
computer vision, document processing, computer network, and genetics
etc. However, almost all data in these areas have much lower
intrinsic dimensions. Thus, how to handle these data is a
traditional problem.

\subsection{PCA}
Principal component analysis (PCA) \cite{jolliffe2002principal} is
one of the most popular analysis tools to deal with this situation.
Given a set of data, whose mean is removed, PCA approximates the
data by representing them in another orthonormal basis, called
loading vectors. The coefficients of the data when represented using
these loadings are called principal components. They are obtained by
projecting the data onto the loadings, i.e. inner products between
the loading vectors and the data vector. Usually, the loadings are
deemed as a set of ordered vectors, in that the variances of data
explained by them are in a decreasing order, e.g. the leading
loading points to the maximal-variance direction. If the data lie in
a low dimensional subspace, i.e. the distribution mainly varies in a
few directions, a few loadings are enough to obtain a good
approximation; and the original high-dimensional data now can be
represented by the low-dimensional principal components, so
dimensionality reduction is achieved.

Commonly, the dimensions of the original data have some physical
explanations. For example, in financial or biological applications,
each dimension may correspond to a specific asset or gene
\cite{dAspremont2008optimal}. However, the loadings obtained by PCA
are usually dense, so the principal component, got by inner product,
is a mixture of all dimensions, which makes it difficult to
interpret. If most of the entries in the loadings are zeros
(sparse), each principal component becomes a linear combination of a
few non-zero entries. This facilitates the understanding of the
physical meaning of the loadings as well as the principal components
\cite{jolliffe2002principal}. Further, the physical interpretation
would be clearer if different loadings have different non-zero
entries, i.e. corresponding to different dimensions.

\subsection{Sparse PCA}
Sparse PCA aims at finding a sparse basis to make the result more
interpretable \cite{jolliffe2003modified}. At the same time, the
basis is required to represent the data distribution faithfully.
Thus, there is a tradeoff between the statistical fidelity and the
interpretability.

During the past decade, a variety of methods for sparse PCA have
been proposed. Most of them have considered the tradeoff between
sparsity and explained variance. However, there are three points
that have not received enough attentions yet: the orthogonality
between loadings, the balance of sparsity among loadings, and the
pitfall of deflation algorithms.

\begin{itemize}
\item{Orthogonality}.
PCA loadings are orthogonal. But in pursuing sparse loadings, this
property is easy to lose. Orthogonality is desirable in that it
indicates the independence of physical meaning of the loadings. When
the loadings are sufficiently sparse, orthogonality usually implies
non-overlapping of their supports. So under the background of
improving the interpretation of PCA, now each loading is associated
with distinctive physical variables, so are the principal
components. This makes the interpretation much easier. Besides, if
the loadings are not an orthogonal basis, the inner products between
the data and the loadings that are used to compute the components do
not constitute an exact projection. For an extreme example, if two
loadings are very close, the two components would be similar too.
This is meaningless.

\item{Balance of sparsity}.
There should not be any member of the loadings highly dense,
particularly those leading ones that take account of most variance,
otherwise it is meaningless. We emphasize this point, because quite
a few of existing algorithms yield loadings with the leading ones
highly dense (close to those of PCA) while the minor ones highly
sparse; so sparsity is achieved by the minor ones while variance is
explained by the dense ones. This is unreasonable.

\item{Pitfall of deflation}.
Existing work can be categorized into two groups: deflation group
and block group. To obtain $r$ sparse loadings, the deflation group
computes one loading at a time; more are got via removing components
that have been computed \cite{mackey2009deflation}. This follows
traditional PCA. The block group finds all loadings together.
Generally, the optimal loadings found when we restrict the subspace
to be of dimension $r$ may not overlap with the $r+1$ optimal
loadings when the dimension increases to $r+1$
\cite{jolliffe1989rotation}. This problem does not occur for PCA,
whose loadings successively maximize the variance, and the loadings
found via deflation are always globally optimal for any $r$. But it
is not the case for sparse PCA, the deflation method is greedy and
cannot find optimal sparse loadings. However, the block group has
the potential.

\end{itemize}

Finally we mention that by deflation the obtained loadings are
nearly orthogonal, while the block group usually does not equip with
mechanism to ensure the orthogonality.

\subsection{Our Method: SPCArt} In this paper, we
propose a new approach called SPCArt (Sparse PCA via rotation and
truncation). In contrast to most of traditional work which are based
on adding some sparsity penalty on the objective of PCA, the
motivation of SPCArt is distinctive. SPCArt aims to find a rotation
matrix and a sparse basis such that the sparse basis approximates
the loadings of PCA after the rotation. The resulting algorithm
consists of three alternative steps: rotate PCA loadings, truncate
small entries of them, and update the rotation matrix.

SPCArt turns out to resolve or alleviate the previous three points.
It has the following merits. (1) It is able to explain as much
variance as the PCA loadings, since the sparse basis spans almost
the same subspace as the PCA loadings. (2) The new basis is close to
orthogonal, since it approximates the rotated PCA loadings. (3) The
truncation tends to produce more balanced sparsity, since vectors of
the rotated PCA loadings are of equal length. (4) It is not greedy
compared with the deflation group, it belongs to the block group.

The contributions of this paper are four-fold: (1) we propose an
efficient algorithm SPCArt achieving good performance over a series
of criteria, some of which have been overlooked by previous work;
(2) we devise various truncation operations for different situations
and provide performance analysis; (3) we give a unified view for a
series of previous sparse PCA approaches, together with ours; (4)
under the unified view, we find the relation between GPower, rSVD,
and our method, and extend GPower \cite{journee2010generalized} and
rSVD \cite{shen2008sparse} to a new implementation, called rSVD-GP,
to overcome their drawbacks--parameter tuning problem and imbalance
of sparsity among loadings.

The rest of the paper is organized as follows:
Section~\ref{sec:related work} introduces representative work on
sparse PCA. Section~\ref{sec:SPCArt} presents our method SPCArt and
four types of truncation operations, and analyzes their performance.
Section~\ref{sec:unified} gives a unified view for a series of
previous work. Section~\ref{sec:relation to GPower} shows the
relation between GPower, rSVD, and our method, and extends GPower
and rSVD to a new implementation, called rSVD-GP. Experimental
results are provided in Section~\ref{sec:experiment}. Finally, we
conclude this paper in Section~\ref{sec:conclusion}.

\begin{table*}[t]
\caption{Time complexities for computing $r$ loadings from $n$
samples of dimension $p$. $m$ is the number of iterations. $k$ is
the cardinality of a loading. The preprocessing and initialization
overheads are omitted. ST and SPCArt have the additional cost of
PCA. The complexities of SPCArt listed below are of the truncation
types T-$\ell_0$ and T-$\ell_1$. Those of T-sp and T-en are
$O(rp\log p+r^2p+r^3)$. }\label{tab:timeO} \vskip 0.1in
\begin{center}
\begin{scriptsize} %
\begin{tabular}{|c||c|c|c|c|c|c|c|c|}
\hline & PCA \cite{jolliffe2002principal} & ST
\cite{cadima1995loading} & SPCA \cite{zou2006sparse}& PathSPCA
\cite{dAspremont2008optimal}& ALSPCA \cite{lu2009augmented}&
\tabincell{c}{GPower \cite{journee2010generalized}, \\rSVD-GP,\\
TPower \cite{yuan2013truncated}} & \tabincell{c}{GPowerB
\cite{journee2010generalized},\\ rSVD-GPB} & SPCArt\\\hline\hline

$n>p$ & $O(np^2)$ & $O(rp)$ & $mO(r^2p+rp^3)$ & $O(rkp^2+rk^3)$ &
$mO(rp^2)$ & $mO(rp^2)$ & $mO(rpn+r^2n)$ & $mO(r^2p+r^3)$\\\hline

$n<p$ & $O(pn^2)$ & $O(rp)$ & $mO(r^2p+rnp)$ & $O(rknp+rk^3)$ &
$mO(rnp)$ & $mO(rnp)$ & $mO(rpn+r^2n)$ & $mO(r^2p+r^3)$\\\hline
\end{tabular}
\end{scriptsize}
\end{center}
\vskip -0.1in
\end{table*}

\section{Related Work}\label{sec:related work}
Various sparse PCA methods have been proposed during the past
decade. We give a brief review below.

\emph{1. Post-processing PCA}. In early days, interpretability is
gained via post-processing the PCA loadings. Loading rotation (LR)
\cite{jolliffe1989rotation} applies various criteria to rotate the
PCA loadings so that 'simple structure' emerges, e.g. varimax
criterion drives the entries to be either small or large, which is
close to a sparse structure. Simple thresholding (ST)
\cite{cadima1995loading} instead obtains sparse loadings via
directly setting the entries of PCA loadings below a small threshold
to zero.

\emph{2. Covariance matrix maximization}. More recently, systematic
approaches based on solving explicit objectives were proposed,
starting from SCoTLASS \cite{jolliffe2003modified} which optimizes
the classical objective of PCA, i.e. maximizing the quadratic form
of covariance matrix, while additionally imposing a sparsity
constraint on each loading.

\emph{3. Matrix approximation}. SPCA \cite{zou2006sparse} formulates
the problem as a regression-type optimization, so as to facilitate
the use of LASSO \cite{tibshirani1996regression} or elastic-net
\cite{zou2005regularization} techniques to solve the problem. rSVD
\cite{shen2008sparse} and SPC \cite{witten2009penalized} obtain
sparse loadings by solving a sequence of rank-1 matrix
approximations, with sparsity penalty or constraint imposed.

\emph{4. Semidefinite convex relaxation}. Most of the methods
proposed so far are local ones, which suffer from getting trapped in
local minima. DSPCA \cite{aspremont2007direct} transforms the
problem into a semidefinite convex relaxation problem, thus global
optimality of solution is guaranteed. This distinguishes it from
most of the other local methods. Unfortunately, its computational
complexity is as high as $O(p^4\sqrt{\log p})$ ($p$ is the number of
variables), which is expensive for most applications. Later, a
variable elimination method \cite{zhang2011large} of complexity
$O(p^3)$ was developed in order to make the application on large
scale problem feasible.

\emph{5. Greedy methods}.  In \cite{moghaddam2006spectral}, greedy
search and branch-and-bound methods are used to solve small
instances of the problem exactly. Each step of the algorithm has a
complexity $O(p^3)$, leading to a total complexity of $O(p^4)$ for a
full set of solutions (solutions of cardinality ranging from 1 to
$p$). Later, this bound is improved in the classification setting
\cite{moghaddam2006generalized}. In another way, a greedy algorithm
PathSPCA \cite{dAspremont2008optimal} was presented to further
approximate the solution process of \cite{moghaddam2006spectral},
resulting in a complexity of $O(p^3)$ for a full set of solutions.
For a review of DSPCA, PathSPCA, and their applications, see
\cite{zhang2012sparse}.

\emph{6. Power methods}. The GPower method
\cite{journee2010generalized} formulates the problem as maximization
of a convex objective function and the solution is obtained by
generalizing the power method \cite{golub1996matrix} that is used to
compute the PCA loadings. Recently, a new power method TPower
\cite{yuan2013truncated}, and a somewhat different but related power
method ITSPCA \cite{ma2013sparse} that aims at recovering sparse
principal subspace, were proposed.

\emph{7. Augmented lagrangian optimization}. ALSPCA
\cite{lu2009augmented} solves the problem based on an augmented
lagrangian optimization. The most special feature of ALSPCA is that
it simultaneously considers the explained variance, orthogonality,
and correlation among principal components.

Among them only LR \cite{jolliffe1989rotation}, SCoTLASS
\cite{jolliffe2003modified}, ALSPCA \cite{lu2009augmented} have
considered the orthogonality of loadings. SCoTLASS, rSVD
\cite{shen2008sparse}, SPC \cite{witten2009penalized}, the greedy
methods \cite{moghaddam2006spectral, dAspremont2008optimal}, one
version of GPower \cite{journee2010generalized}, and TPower
\cite{yuan2013truncated} belong to the deflation group. Only DSPCA's
solution \cite{aspremont2007direct} is ensured to be globally
optimal.

The computational complexities of some of the above algorithms are
summarized in Table~\ref{tab:timeO}.

\section{SPCArt: Sparse PCA via Rotation and Truncation}\label{sec:SPCArt}
We first give a brief overview of SPCArt, next introduce the
motivation, and then the objective and optimization, and then the
truncation types, and finally provide performance analysis.

The idea of SPCArt is simple. Since any rotation of the $r$ PCA
loadings $[V_1,\dots,V_r]\in\mathbb{R}^{p\times r}$ constitutes an
orthogonal basis spanning the same subspace, $X=VR$
($R\in\mathbb{R}^{r\times r}$, $R^TR=I$), we want to find a rotation
matrix $R$ through which $V$ is transformed to a sparsest basis $X$.
It is difficult to solve this problem directly, so instead we would
find a rotation matrix and a sparse basis such that the sparse basis
approximates the PCA loadings after the rotation $V\approx XR$.


The major notations used are listed in Table~\ref{tab:notations}.

\begin{table*}[t]
\caption{Major notations.}\label{tab:notations} \vskip 0.1in
\begin{center}
\begin{small} %
\begin{tabular}{|c||m{12.5cm}|}

\hline notation & interpretation \\\hline\hline

$A\in \mathbb{R}^{n\times p}$ & data matrix with $n$ samples of $p$
variables\\\hline

$V=[V_1,V_2,\dots]$ & PCA loadings arranged column-wise. $V_i$
denotes the $i$th column. $V_{1:r}$ denotes the first $r$
columns\\\hline

$R$ & rotation matrix\\\hline

$Z$ & rotated PCA loadings, i.e. $VR^T$\\\hline

$X$ & spare loadings arranged column-wise, similar to $V$\\\hline

$Polar(\cdot)$ & for a matrix $B\in \mathbb{R}^{n\times p}$, $n\geq
p$, let the thin SVD be $WDQ^T$, $D\in \mathbb{R}^{p\times p}$, then
$Polar(B)=WQ^T$\\\hline

$S_\lambda(\cdot)$ & $0\leq \lambda <1$. For a vector $x$,
$S_\lambda(x)$ is entry-wise soft thresholding:
$S_\lambda(x_i)=sign(x_i)(|x_i|-\lambda)_+$, where $[y]_+= y$ if
$y\geq 0$ and $[y]_+=0$ otherwise\\\hline

$H_\lambda(\cdot)$ & $0\leq \lambda <1$. For a vector $x$,
$H_\lambda(x)$ is entry-wise hard thresholding:
$H_\lambda(x_i)=x_i[sign(|x_i|-\lambda)]_+$, i.e. $H_\lambda(x_i)=0$
if $|x_i|\leq \lambda$, $H_\lambda(x_i)=x_i$ otherwise\\\hline

$P_\lambda(\cdot)$ & $\lambda\in\{0,1,2,\cdots\}$. For a vector $x$,
$P_\lambda(x)$ sets the smallest $\lambda$ entries (absolute value)
to be zero\\\hline

$E_\lambda(\cdot)$ & $0\leq \lambda <1$. For a vector $x$,
$E_\lambda(x)$ sets the smallest $k$ entries, whose energy take up
at most $\lambda$, to be zero. $k$ is found as following: sort
$|x_1|,|x_2|,\dots$ in ascending order to be
$\bar{x}_1,\bar{x}_2,\dots$, $k=\max_i i,\,s.t.\sum_{j=1}^i
\bar{x}_{j}^2/\|x\|_2^2\leq\lambda$\\\hline

\end{tabular}
\end{small}
\end{center}
\vskip -0.1in
\end{table*}

\subsection{Motivation}\label{sec:motivation}
Our method is motivated by the solution of the Eckart-Young theorem
\cite{eckart1936approximation}. This theorem considers the problem
of approximating a matrix by the product of two low-rank ones.

\begin{theorem}\label{theo:Eckart-Young}
(\textbf{Eckart-Young Theorem}) Assume the SVD of a matrix
$A\in\mathbb{R}^{n\times p}$ is $A=U\Sigma V^T$, in which
$U\in\mathbb{R}^{n\times m}$, $m\leq min\{n,p\}$,
$\Sigma\in\mathbb{R}^{m\times m}$ is diagonal with
$\Sigma_{11}\geq\Sigma_{22}\geq\cdots\geq\Sigma_{mm}$, and
$V\in\mathbb{R}^{p\times m}$. A rank $r$ ($r\leq m$) approximation
of $A$ is to solve the following problem:
\begin{equation}\label{equ:Eckart-Young}
\min_{Y,X} \|A-YX^T\|^2_F,\;s.t.\,X^TX=I,
\end{equation}
where $Y\in\mathbb{R}^{n\times r}$ and $X\in\mathbb{R}^{p\times r}$.
A solution is
\begin{equation}\label{equ:Y=AX}
X^*=V_{1:r},\;Y^*=AX^*,
\end{equation}
where $V_{1:r}$ is the first $r$ columns of $V$.
\end{theorem}
Alternatively, the solution can be expressed as
\begin{equation}\label{equ:X=Polar}
Y^*=U_{1:r}\Sigma_{1:r},\;X^*=Polar(A^TY^*),
\end{equation}
where $Polar(\cdot)$ is the orthonormal component of the polar
decomposition of a matrix \cite{journee2010generalized}. From the
more familiar SVD perspective, its equivalent definition is provided
in Table~\ref{tab:notations}.

Note that if the row vectors of $A$ have been centered to have mean
zero, $V_{1:r}$ are the loadings obtained by PCA. Clearly, $\forall
R\in\mathbb{R}^{r\times r}$ and $R^TR=I$, $X^*=V_{1:r}R$ and
$Y^*=AX^*=U_{1:r}\Sigma_{1:r}R$ is also a solution of
(\ref{equ:Eckart-Young}). This implies that any rotation of the $r$
orthonormal leading eigenvectors $V_{1:r}\in\mathbb{R}^{p\times r}$
is a solution of the best rank $r$ approximation of $A$. That is,
any orthonormal basis in the corresponding eigen-subspace is capable
of representing the original data distribution as well as the
original basis. Thus, a natural idea for sparse PCA is to find a
rotation matrix $R$ so that $X=V_{1:r}R$ becomes sparse, i.e.,
\begin{equation}\label{equ:motivation}
\min_{R\in\mathbb{R}^{r\times r}} \|V_{1:r}R\|_0,\;s.t.\;R^TR=I,
\end{equation}
where $\| \cdot \|_0$ denotes the sum of $\ell_0$ 
{(pseudo)} norm of the columns of a matrix, i.e. it counts the
non-zeros of a matrix.

\subsection{Objective and optimization}\label{sec:objective}
Unfortunately, the above problem is hard to solve. So we approximate
it instead. Since $X=V_{1:r}R\Leftrightarrow V_{1:r}=XR^T$, we want
to find a rotation matrix $R$ through which a sparse basis $X$
approximates the original PCA loadings. Without confusion, we use
$V$ to denote $V_{1:r}$ hereafter. For simplicity, the $\ell_0$
version will be postponed to next section, we consider the $\ell_1$
version first:
\begin{equation}\label{equ:SPCArtl1}
\begin{split}
&\min_{X,R}\;
\frac{1}{2}\|V-XR\|^2_F+\lambda\sum_i\|X_i\|_1,\,\\
&s.t.\,\forall i,\, \|X_i\|_2=1,\,R^TR=I.
\end{split}
\end{equation}
$\| \cdot \|_1$ is the $\ell_1$ norm of a vector, i.e. sum of
absolute values. It is well-known that $\ell_1$ norm is sparsity
inducing, which is a convex surrogate of the $\ell_0$ norm
\cite{donoho2006most}. Under this objective, the solution may not be
orthogonal, and may deviate from the eigen-subspace spanned by $V$.
However, if the approximation is accurate enough, i.e., $V\approx
XR$, then $X\approx VR^T$ would be nearly orthogonal and explain
similar variance as $V$. Note that the above objective turns out to
be a matrix approximation problem as Eckart-Young theorem. The key
difference is that sparsity penalty is added. But the solutions
still share some common features.

There is no closed-form solutions for $R$ and $X$ simultaneously. We
can solve the problem by fixing one and optimizing the other
alternately. Both subproblems have closed-form solutions.

\subsubsection{Fix $X$, solve $R$}
When $X$ is fixed, it becomes a Procrustes problem
\cite{zou2006sparse}:
\begin{equation}\label{equ:SPCArtl1-R}
\min_{R} \|V-XR\|^2_F,\;s.t.\,R^TR=I.
\end{equation}
$R^*=Polar(X^TV)$. It has the same form as the right of
(\ref{equ:X=Polar}).

\subsubsection{Fix $R$, solve $X$}
When $R$ is fixed, it becomes
\begin{equation}\label{equ:SPCArtl1-X}
\min_{X}
\frac{1}{2}\|VR^T-X\|^2_F+\lambda\sum_i\|X_i\|_1,\,s.t.\,\forall
i,\, \|X_i\|_2=1.
\end{equation}
There are $r$ independent subproblems, one for each column:
$\min_{X_i}
1/2\|Z_i-X_i\|^2_2+\lambda\|X_i\|_1,\,s.t.\,\|X_i\|_2=1$, where
$Z=VR^T$. It is equivalent to
$\max_{X_i}\;Z_i^TX_i-\lambda\|X_i\|_1,\,s.t.\,\|X_i\|_2=1$. The
solution is $X_i^*=S_\lambda (Z_i)/\|S_\lambda (Z_i)\|_2$
\cite{journee2010generalized}. $S_\lambda (\cdot)$ is entry-wise
soft thresholding, defined in Table~\ref{tab:notations}. This is
truncation type \textbf{T-$\ell_1$: soft thresholding}.

It has the following physical explanations. $Z$ is rotated PCA
loadings, it is orthonormal. $X$ is obtained via truncating small
entries of $Z$. On one hand, because of the unit length of each
column in $Z$, a single threshold $0\leq \lambda < 1$ is feasible to
make the sparsity distribute evenly among the columns in $X$;
otherwise we have to apply different thresholds for different
columns which are hard to determine. On the other hand, because of
the orthogonality of $Z$ and small truncations, $X$ is still
possible to be nearly orthogonal. These are the most distinctive
features of SPCArt. They enable easy analysis and parameter setting.

The algorithm of SPCArt is presented in Algorithm~\ref{alg:SPCArt},
where the truncation in line 7 can be any type, including T-$\ell_1$
and the others that will be introduced in next section.

The computational complexity of SPCArt is shown in
Table~\ref{tab:timeO}. Except the computational cost of PCA
loadings, SPCArt scales linearly about data dimension. When the
number of loadings is not too large, it is efficient.

\begin{algorithm}[tb]
   \caption{SPCArt}
   \label{alg:SPCArt}
\begin{algorithmic}[1]
   \STATE {\bfseries Input:} data matrix $A\in\mathbb{R}^{n\times p}$, number of loadings $r$, truncation type
$T$, truncation parameter $\lambda$.
   \STATE {\bfseries Output:} sparse loadings $X=[X_1,\dots,X_r]\in\mathbb{R}^{p\times r}$.
   \STATE PCA: compute rank-$r$ SVD of $A$: $U\Sigma V^T$, $V\in\mathbb{R}^{p\times r}$.
   \STATE Initialize $R$: $R=I$.
   \REPEAT
   \STATE Rotation: $Z=VR^T$.
   \STATE Truncation: $\forall {i}$, $X_i=T_\lambda(Z_i)/\|T_\lambda(Z_i)\|_2$.
   \STATE Update $R$: thin SVD of $X^TV$: $WDQ^T$, $R=WQ^T$.
   \UNTIL{convergence}
\end{algorithmic}
\end{algorithm}

\subsection{Truncation Types}\label{sec:truncation types}

In this section, given rotated PCA loadings $Z$, we introduce the
truncation operation $T_\lambda(Z_i)$, where $T_\lambda$ is any of
the following four types: T-$\ell_1$ soft thresholding $S_\lambda$,
T-$\ell_0$ hard thresholding $H_\lambda$, T-sp truncation by
sparsity $P_\lambda$, and T-en truncation by energy $E_\lambda$.
T-$\ell_1$ has been introduced in last section, which is resulted
from $\ell_1$ penalty.

\textbf{T-$\ell_0$: hard thresholding}. Set the entries below
threshold $\lambda$ to be zero: $X_i^*=H_\lambda (Z_i)/\|H_\lambda
(Z_i)\|_2$. $H_\lambda(\cdot)$ is defined in
Table~\ref{tab:notations}. It is resulted from $\ell_0$ penalty:
\begin{equation}\label{equ:SPCArtl0}
\min_{X,R} \|V-XR\|^2_F+\lambda^2\sum_i\|X_i\|_0,\,s.t.\,R^TR=I,
\end{equation}
The optimization is similar to the $\ell_1$ case. Fixing $X$,
$R^*=Polar(X^TV)$. Fixing $R$, the problem becomes $\min_{X}
\|VR^T-X\|^2_F+\lambda^2\|X\|_0$. Let $Z=VR^T$, it can be decomposed
to $p\times r$ entry-wise subproblems, and the solution is apparent:
if $|Z_{ji}|\leq \lambda$, then $X_{ji}^*=0$, otherwise
$X_{ji}^*=Z_{ji}$. Hence the solution can be expressed as
$X_i^*=H_\lambda (Z_i)$.

There is no normalization for $X^*$ compared with the $\ell_1$ case.
This is because if unit length constraint $\|X_i\|_2=1$ is added,
there will be no closed form solution. However, in practice, we
still let $X_i^*=H_\lambda (Z_i)/\|H_\lambda (Z_i)\|_2$ for
consistency, since empirically no significant difference is
observed.

Note that both $\ell_0$ and $\ell_1$ penalties only result in
thresholding operation on $Z$ and nothing else (only make line 7 of
Algorithm~\ref{alg:SPCArt} different). Hence, we may devise other
heuristic truncation types irrespective of explicit objective:

\textbf{T-sp: truncation by sparsity}. Truncate the smallest
$\lambda$ entries: $X_i=P_\lambda(Z_i)/\|P_\lambda(Z_i)\|_2$,
$\lambda\in\{0,1,\dots,p-1\}$. Table~\ref{tab:notations} gives the
precise definition of $P_\lambda(\cdot)$. It can be shown that this
heuristic type is resulted from the $\ell_0$ constraint:
\begin{equation}\label{equ:SPCArtsp}
\begin{split}
&\min_{X,R}\;\|V-XR\|^2_F,\\
&s.t.\,\forall i,\, \|X_i\|_0\leq p-\lambda,\;\|X_i\|_2=1,\,R^TR=I.
\end{split}
\end{equation}
When $X$ is fixed, the solution is the same as $\ell_0$ and $\ell_1$
cases above. When $R$ is fixed, the solution is
$X_i^*=P_\lambda(Z_i)/\|P_\lambda(Z_i)\|_2$, where $Z=VR^T$. The
proof is put in Appendix~\ref{sec:app T-sp}.

\textbf{T-en: truncation by energy}. Truncate the smallest entries
whose energy (sum of square) take up $\lambda$ percentage:
$X_i=E_\lambda(Z_i)/\|E_\lambda(Z_i)\|_2$. $E_\lambda$ is described
in Table~\ref{tab:notations}. However, we are not aware of any
objective associated with this type.

Algorithm~\ref{alg:SPCArt} describes the complete algorithm of
SPCArt with any truncation type.

SPCArt promotes the seminal ideas of simple thresholding
\cite{cadima1995loading} and loading rotation
\cite{jolliffe1989rotation}. When using T-$\ell_0$, the first
iteration of SPCArt, i.e. $X_i=H_\lambda(V_i)$, corresponds to the
ad-hoc simple thresholding ST, which is frequently used in practice
and sometimes produced good results \cite{zou2006sparse,
moghaddam2006spectral}. In another way, the motivation of SPCArt,
i.e. (\ref{equ:motivation}), is similar to the loading rotation,
whereas SPCArt explicitly seeks sparse loadings via $\ell_0$
pseudo-norm, loading rotation seeks 'simple structure' via various
criteria, e.g. the varimax criterion, which maximizes the variances
of squared loadings $\sum_i [\sum_j Z_{ji}^4-1/p(\sum_k Z_{ki}^2)]$,
where $Z=VR$, drives the entries to distribute unevenly, either
small or large (see Section 7.2 in \cite{jolliffe2002principal}).

\subsection{Performance Analysis}\label{sec:analysis}
This section discusses the performance bounds of SPCArt with each
truncation type. For $X_i=T_\lambda(Z_i)/\|T_\lambda(Z_i)\|_2$,
$i=1,\dots,r$, we study the following problems:

(1) How much sparsity of $X_i$ is guaranteed?

(2) How much $X_i$ deviates from $Z_i$?

(3) How is the orthogonality of $X$?

(4) How much variance is explained by $X$?

The performance bounds derived are functions of $\lambda$. Thus, we
can directly or indirectly control sparsity, orthogonality, and
explained variance via $\lambda$.\footnote{Theorem~\ref{theo:cosev}
is specific to SPCArt, which concerns the important explained
variance. The other results apply to more general situations:
Proposition 6-11 apply to any orthonormal $Z$,
Theorem~\ref{theo:dminev} applies to any matrix $X$. To obtain
results specific to SPCArt, we may have to make assumption of the
data distribution. Nevertheless, they are still the absolute
performance bounds of SPCArt and can guide us to set $\lambda$ for
some performance guarantee.} We give some definitions first.

\begin{definition}
$\forall x\in\mathbb{R}^p$, the \textbf{sparsity} of $x$ is the
proportion of zero entries: $s(x)=1-\|x\|_0/p.$
\end{definition}

\begin{definition}
$\forall z\in\mathbb{R}^{p}$, $z\neq 0$,
$x=T_\lambda(z)/\|T_\lambda(z)\|_2$, the \textbf{deviation} of $x$
from $z$ is $\sin(\theta(x,z))$, where $\theta(x,z)$ is the included
angle between $x$ and $z$, $0\leq \theta(x,z)\leq \pi/2$. If $x=0$,
$\theta(x,y)$ is defined to be $\pi/2$.
\end{definition}

\begin{definition}
$\forall x,\,y\in\mathbb{R}^p$, $x\neq 0$, $y\neq 0$, the
\textbf{nonorthogonality} between $x$ and $y$ is
$|\cos(\theta(x,y))|=|x^Ty|/(\|x\|_2\cdot\|y\|_2)$, where
$\theta(x,y)$ is the included angle between $x$ and $y$.
\end{definition}

\begin{definition}
Given data matrix $A\in\mathbb{R}^{n\times p}$ containing $n$
samples of dimension $p$, $\forall$ basis $X\in\mathbb{R}^{p\times
r}$, $r\leq p$, the \textbf{explained variance} of $X$ is
$EV(X)=tr(X^TA^TAX)$. Let $U$ be any orthonormal basis in the
subspace spanned by $X$, then the \textbf{cumulative percentage of
explained variance} is $CPEV(X)=tr(U^TA^TAU)/tr(A^TA)$
\cite{shen2008sparse}.
\end{definition}

Intuitively, larger $\lambda$ leads to higher sparsity and larger
deviation. When two truncated vectors deviate from their originally
orthogonal vectors, in the worst case, the nonorthogonality of them
degenerates as the `sum' of their deviations. In another way, if the
deviations of a sparse basis from the rotated loadings are small, we
expect the sparse basis still represents the data well, and the
explained variance or cumulative percentage of explained variance
maintains similar level to that of PCA. So, both the
nonorthogonality and the explained variance depend on the
deviations, and the deviation and sparsity in turn are controlled by
$\lambda$. We now go into details. The proofs of some of the
following results are included in Appendix~\ref{sec:app bounds}.

\vspace{2ex}
\subsubsection{Orthogonality}
\begin{proposition}\label{theo:nonortho}
The relative upper bound of nonorthogonality between $X_i$ and
$X_j$, $i\neq j$, is
\begin{equation}
\begin{aligned}
&|\cos(\theta(X_i,X_j))|\leq\\
&
\begin{cases}
  \sin(\theta(X_i,Z_i)+\theta(X_j,Z_j))&,\theta(X_i,Z_i)+\theta(X_j,Z_j)\leq \frac{\pi}{2}, \\
  1&,\text{otherwise.}
\end{cases}
\end{aligned}
\end{equation}
\end{proposition}

The bounds can be obtained by considering the two conical surfaces
generated by axes $Z_i$ with rotational angles $\theta(X_i,Z_i)$.
The proposition implies the nonorthogonality is determined by the
sum of deviated angles. When the deviations are small, the
orthogonality is good. The deviation depends on $\lambda$, which is
analyzed below.

\vspace{2ex}
\subsubsection{Sparsity and Deviation}
The following results only concern a single vector of the basis. We
will denote $Z_i$ by $z$, and $X_i$ by $x$ for simplicity, and
derive bounds of sparsity $s(x)$ and deviation $\sin(\theta(x,z))$
for each $T$. They depend on a key value $1/\sqrt p$, which is the
entry value of a uniform vector.

\begin{proposition}
For T-$\ell_0$, the sparsity bounds are
\begin{equation}
\begin{cases}
  0\leq s(x) \leq 1-\frac{1}{p} & \text{, $\lambda<\frac{1}{\sqrt{p}}$,} \\
  1-\frac{1}{p\lambda^2}< s(x) \leq 1 & \text{, $\lambda\geq\frac{1}{\sqrt{p}}$.}
\end{cases}
\end{equation}
Deviation $\sin(\theta(x,z))=\|\bar{z}\|_2$, where $\bar{z}$ is the
truncated part: $\bar{z}_i=z_i$ if $x_i=0$, and $\bar{z}_i=0$
otherwise. The absolute bounds are:
\begin{equation}
0\leq \sin(\theta(x,z)) \leq
\begin{cases}
  \sqrt{p-1}\lambda & \text{, $\lambda<\frac{1}{\sqrt{p}}$,} \\
  1 & \text{, $\lambda\geq\frac{1}{\sqrt{p}}$.}
\end{cases}
\end{equation}
All the above bounds are achievable.
\end{proposition}

Because when $\lambda<1/\sqrt{p}$, there is no sparsity guarantee,
$\lambda$ is usually set to be $1/\sqrt{p}$ in practice. Generally
it works well.

\begin{proposition}
For T-$\ell_1$, the bounds of $s(x)$ and lower bound of
$\sin(\theta(x,z))$ are the same as T-$\ell_0$. In addition, there
are relative deviation bounds
\begin{equation}
\|\bar{z}\|_2\leq\sin(\theta(x,z)) <
\sqrt{\|\bar{z}\|^2_2+\lambda^2\|x\|_0 }.
\end{equation}
\end{proposition}

It is still an open question that whether T-$\ell_1$ has the same
upper bound of $\sin(\theta(x,z))$ as T-$\ell_0$. By the relative
lower bounds, we have

\begin{corollary}
The deviation due to soft thresholding is always larger than that of
hard thresholding, if the same $\lambda$ is applied.
\end{corollary}

This implies that results got by T-$\ell_1$ have potentially greater
sparsity and less explained variance than those of T-$\ell_0$.

\begin{proposition}\label{theo:sp}
For T-sp, $\lambda/p\leq s(z) < 1$, and
\begin{equation}
0\leq \sin(\theta(x,z))\leq \sqrt{\lambda/p}\;.
\end{equation}
\end{proposition}

{Except the unusual case that $x$ has many zeros, $s(z)=\lambda/p$.}
The main advantage of T-sp lies in its direct control on sparsity.
If specific sparsity is wanted, it can be applied.

\begin{proposition}\label{theo:en}
For T-en, $0\leq \sin(\theta(x,z))\leq \sqrt{\lambda}$. In addition
\begin{equation}\label{equ:en}
\lfloor{\lambda p}\rfloor/p\leq s(x)\leq 1-1/p.
\end{equation}
If $\lambda<1/p$, there is no sparsity guarantee. When $p$ is
moderately large, $\lfloor{\lambda p}\rfloor/p\approx \lambda$.
\end{proposition}

Due to the discrete nature of operand, the actually truncated energy
can be less than $\lambda$. But in practice, especially when $p$ is
moderately large, the effect is negligible. So we usually have
$\sin(\theta(x,z))\approx \sqrt{\lambda}$. The main advantage of
T-en is that it has direct control on deviation. Recall that the
deviation has direct influence on the explained variance. Thus, if
it is desirable to gain specific explained variance, T-en is
preferable. Besides, if $p$ is moderately large, T-en also gives
nice control on sparsity.

\vspace{2ex}
\subsubsection{Explained Variance}

Finally, we derive bounds on the explained variance $EV(X)$. Two
results are provided. The first one is general and is applicable to
any basis $X$ not limited to sparse ones. The second one is tailored
to SPCArt.

\begin{theorem}\label{theo:dminev}
Let rank-$r$ SVD of $A\in\mathbb{R}^{n\times p}$ be $U\Sigma V^T$,
$\Sigma\in\mathbb{R}^{r\times r}$. Given $X\in\mathbb{R}^{p\times
r}$, assume SVD of $X^TV$ is $WDQ^T$, $D\in\mathbb{R}^{r\times r}$,
$d_{min}=\min_i D_{ii}$. Then
\begin{equation}\label{equ:dminEV}
d_{min}^2\cdot EV(V)\leq EV(X),
\end{equation}
and $EV(V)=\sum_i \Sigma^2_{ii}$.
\end{theorem}

The theorem can be interpreted as follows. If $X$ is a basis that
approximates rotated PCA loadings well, then $d_{min}$ will be close
to one, and so the variance explained by $X$ is close to that
explained by PCA. Note that variance explained by PCA loadings is
the largest value that is possible to be achieved by orthonormal
basis. Conversely, if $X$ deviates much from the rotated PCA
loadings, then $d_{min}$ tends to zero, so the variance explained by
$X$ is not guaranteed to be much. We see that the less the sparse
loadings deviates from rotated PCA loadings, the more variance they
explain.

When SPCArt converges, i.e.
$X_i=T_\lambda(Z_i)/\|T_\lambda(Z_i)\|_2$, $Z=VR^T$, and
$R=Polar(X^TV)$ hold simultaneously, we have another estimation. It
is mainly valid for T-en.

\begin{theorem}\label{theo:cosev}
Let $C=Z^TX$, i.e. $C_{ij}=\cos(\theta(Z_i,X_j))$, and let $\bar{C}$
be $C$ with diagonal elements removed. Assume
$\theta(Z_i,X_i)=\theta$ and $\sum^r_{j}C_{ij}^2\leq 1$, $\forall
i$, then
\begin{equation}
(\cos^2(\theta)-\sqrt{r-1}\sin(2\theta))\cdot EV(V)\leq EV(X).
\end{equation}
When $\theta$ is sufficiently small,
\begin{equation}\label{equ:cosEV}
(\cos^2(\theta)-O(\theta))\cdot EV(V)\leq EV(X).
\end{equation}
\end{theorem}

Since the sparse loadings are obtained by truncating small entries
of the rotated PCA loadings, and $\theta$ is the deviation angle
between these sparse loadings and the rotated PCA loadings, the
theorem implies, if the deviation is small then the variance
explained by the sparse loadings is close to that of PCA, as
$\cos^2(\theta)\approx 1$. For example, if the truncated energy
$\|\bar{z}\|^2_2=\sin^2 (\theta)$ is about 0.05, then 95\% EV of PCA
loadings is guaranteed.

The assumptions $\theta(Z_i,X_i)=\theta$ and $\sum^r_{j}C_{ij}^2\leq
1$, $\forall i$, are roughly satisfied by T-en using small
$\lambda$. Uniform deviation $\theta(Z_i,X_i)=\theta$, $\forall i$,
can be achieved by T-en as indicated by Proposition~\ref{theo:en}.
$\sum^r_{j}C_{ij}^2\leq 1$ means the sum of projected length is less
than 1, when $Z_i$ is projected onto each $X_j$. It must be
satisfied if $X$ is exactly orthogonal, whereas it is likely to be
satisfied if $X$ is nearly orthogonal (note $Z_i$ may not lie in the
subspace spanned by $X$), which can be achieved by setting small
$\lambda$ according to Proposition~\ref{theo:nonortho}. In this
case, about $(1-\lambda)EV(V)$ is guaranteed.

In practice, we prefer CPEV \cite{shen2008sparse} to EV. CPEV
measures the variance explained by subspace rather than basis. Since
it is also the projected length of $A$ onto the subspace spanned by
$X$, the higher CPEV is, the better $X$ represents the data. If $X$
is not an orthogonal basis, EV may overestimates or underestimates
the variance. However, if $X$ is nearly orthogonal, the difference
is small, and it is nearly proportional to CPEV.

\section{A Unified View to Some Prior Work}\label{sec:unified}
A series of methods: PCA \cite{jolliffe2002principal}, SCoTLASS
\cite{jolliffe2003modified}, SPCA \cite{zou2006sparse}, GPower
\cite{journee2010generalized}, rSVD \cite{shen2008sparse}, TPower
\cite{yuan2013truncated}, SPC \cite{witten2009penalized}, and
SPCArt, though proposed independently and formulated in various
forms, can be derived from the common source of Theorem
\ref{theo:Eckart-Young}, the Eckart-Young Theorem. Most of them can
be seen as the problems of matrix approximation
(\ref{equ:Eckart-Young}), with different sparsity penalties. Most of
them have two matrix variables, and the solutions of them usually
can be obtained by an alternating scheme: fix one and solve the
other. Similar to SPCArt, the two subproblems are a sparsity
penalized/constrained regression problem and a Procrustes problem.

\textbf{PCA} \cite{jolliffe2002principal}. Since $Y^*=AX^*$,
substituting $Y=AX$ into (\ref{equ:Eckart-Young}) and optimizing
$X$, the problem is equivalent to
\begin{equation}\label{equ:PCA}
\max_{X} tr(X^TA^TAX),\;s.t.\,X^TX=I.
\end{equation}
The solution is provided by Ky Fan theorem
\cite{fan1961generalization}: $X^*=V_{1:r}R$, $\forall R^TR=I$. If
$A$ has been centered to have mean zero, the special solution
$X^*=V_{1:r}$ are exactly the $r$ loadings obtained by PCA.

\textbf{SCoTLASS} \cite{jolliffe2003modified}. Constraining $X$ to
be sparse in (\ref{equ:PCA}), we get SCotLASS
\begin{equation}\label{equ:SCoTLASS}
\max_{X} tr(X^TA^TAX),\,s.t.\,X^TX=I,\,\forall i,\, \|X_i\|_1\leq
\lambda.
\end{equation}
However, the problem is not easy to solve.

\textbf{SPCA} \cite{zou2006sparse}. If we substitute $Y=AX$ into
(\ref{equ:Eckart-Young}) and separate the two $X$'s into two
independent variables $X$ and $Z$ (so as to solve the problem via
alternating), and then impose some penalties on $Z$, we get SPCA
\begin{equation}\label{equ:SPCA}
\begin{split}
&\min_{Z,\,X}\;
\|A-AZX^T\|^2_F+\lambda\|Z\|^2_F+\sum_i\lambda_{1i}\|Z_i\|_1,\,\\
&s.t.\,X^TX=I,
\end{split}
\end{equation}
where $Z$ is treated as target sparse loadings and $\lambda$'s are
weights. When $X$ is fixed, the problem is equivalent to $r$
elastic-net problems: $\min_{Z_i}
\|AX_i-AZ_i\|^2_F+\lambda\|Z_i\|^2_2+\lambda_{1i}\|Z_i\|_1$. When
$Z$ is fixed, it is a Procrustes problem: $\min_{X}
\|A-AZX^T\|^2_F,\;s.t.\,X^TX=I$, and $X^*=Polar(A^TAZ)$.

\textbf{GPower} \cite{journee2010generalized}. Except some
artificial factors, the original GPower solves the following
$\ell_0$ and $\ell_1$ versions of objectives:
\begin{equation}\label{equ:ori-GPower-l0}
\begin{split}
&\max_{Y,W}
\sum_i(Y_i^TAW_i)^2-\lambda_{i}\|W_i\|_0,s.t.Y^TY=I,\forall i,
\|W_i\|_2=1,
\end{split}
\end{equation}
\begin{equation}\label{equ:ori-GPower-l1}
\begin{split}
&\max_{Y,W} \sum_i
Y_i^TAW_i-\lambda_{i}\|W_i\|_1,s.t.\,Y^TY=I,\,\forall i,\,
\|W_i\|_2=1.
\end{split}
\end{equation}
They can be seen as derived from the following more fundamental ones
(details are included in Appendix~\ref{sec:app gpower}).
\begin{equation}\label{equ:GPower-l0}
\min_{Y,X} \|A-YX^T\|^2_F+\sum_i\lambda_{i}\|X_i\|_0,\;s.t.\,Y^TY=I,
\end{equation}
\begin{equation}\label{equ:GPower-l1}
\min_{Y,X}
\frac{1}{2}\|A-YX^T\|^2_F+\sum_i\lambda_{i}\|X_i\|_1,\;s.t.\,Y^TY=I.
\end{equation}

These two objectives can be seen as derived from
(\ref{equ:Eckart-Young}): a mirror version of
Theorem~\ref{theo:Eckart-Young} exists $\min_{Y,X} \|A-YX^T\|^2_F$,
$s.t.$ $Y^TY=I$, where $A\in \mathbb{R}^{n\times p}$ is still seen
as a data matrix containing $n$ samples of dimension $p$. The
solution is $X^*=V_{1:r}\Sigma_{1:r}R$ and
$Y^*=Polar(AX^*)=U_{1:r}R$. Adding sparsity penalties to $X$, we get
(\ref{equ:GPower-l0}) and (\ref{equ:GPower-l1}).

Following the alternating optimization scheme. When $X$ is fixed, in
both cases $Y^*=Polar(AX)$. When $Y$ is fixed, the $\ell_0$ case
becomes $\min_{X} \|A^TY-X\|^2_F+\sum_i\lambda_{i}\|X_i\|_0$. Let
$Z=A^TY$, then $X^*_i=H_{\sqrt {\lambda_i}}(Z_i)$; the $\ell_1$ case
becomes $\min_{X} 1/2\|A^TY-X\|^2_F+\sum_i\lambda_{i}\|X_i\|_1$,
$X^*_i=S_\lambda(Z_i)$. The $i$th loading is obtained by normalizing
$X_i$ to unit length.

{The iterative steps combined together produce essentially the same
solution processes to the original ones in
\cite{journee2010generalized}. But, the matrix approximation
formulation makes the relation of GPower to SPCArt and others
apparent. 
{The three methods rSVD, TPower, and SPC below can be seen as
special cases of GPower.}

\textbf{rSVD} \cite{shen2008sparse}. rSVD can be seen as a special
case of GPower, i.e. the single component case $r=1$. Here
$Polar(\cdot)$ reduces to unit length normalization. More loadings
can be got via deflation \cite{mackey2009deflation, shen2008sparse},
e.g. update $A\leftarrow A(I-x^*x^{*T})$ and run the procedure
again. Now, since $Ax^*=0$, the subsequent loadings obtained are
nearly orthogonal to $x^*$.

If the penalties in rSVD are replaced with constraints, we obtain
TPower and SPC.

\textbf{TPower} \cite{yuan2013truncated}. The $\ell_0$ case is
\begin{equation}\label{equ:TPower}
\min_{y\in\mathbb{R}^n,x\in\mathbb{R}^p}
\|A-yx^T\|^2_F,\;s.t.\,\|x\|_0\leq \lambda,\,\|y\|_2=1.
\end{equation}
There are closed form solutions $y^*=Ax/\|Ax\|_2$,
$x^*=P_{p-\lambda}(A^Ty)$. $P_\lambda(\cdot)$ sets the smallest
$\lambda$ entries to zero.\footnote{\cite{shen2008sparse} did
implement this version for rSVD, but using as a heuristic trick.} By
iteration, $x^{(t+1)}\propto P_{p-\lambda}(A^TAx^{(t)})$, which
indicates equivalence to the original TPower.

\textbf{SPC} \cite{witten2009penalized}. The $\ell_1$ case is
$\min_{y,d,x}$ $ \|A-ydx^T\|^2_F$, $s.t.$ $\|x\|_1\leq \lambda$,
$\|y\|_2=1$, $\|x\|_2=1$, $d\in\mathbb{R}$. $d$ serves as the length
of $x$ in (\ref{equ:TPower}). If the other variables are fixed,
$d^*=y^TAx$. If $d$ is fixed, the problem is:
$\max_{y,x}\;tr(y^TAx)$, $s.t.$ $\|x\|_1\leq \lambda$, $\|y\|_2=1$,
$\|x\|_2=1$. A small modification leads to SPC:
\begin{equation*}\label{equ:SPC}
\max_{y,x}\;tr(y^TAx),\;s.t.\,\|x\|_1\leq \lambda,\,\|y\|_2\leq
1,\,\|x\|_2\leq 1,
\end{equation*}
which is biconvex. $y^*=Ax/\|Ax\|_2$. However, there is no analytic
solution for $x$; it is solved by linear searching.

\section{Relation of GPower to SPCArt and an Extension}\label{sec:relation
to GPower}
\subsection{Relation of GPower to SPCArt} Note that
(\ref{equ:GPower-l0}) and (\ref{equ:GPower-l1}) are of similar forms
to (\ref{equ:SPCArtl0}) and (\ref{equ:SPCArtl1}) respectively. There
are two important points of differences. First, SPCArt deals with
orthonormal PCA loadings rather than original data. Second, SPCArt
takes rotation matrix rather than merely orthonormal matrix as
variable. These differences are the key points for the success of
SPCArt.

Compared with SPCArt, GPower has some drawbacks. GPower can work on
both the deflation mode ($r=1$, i.e. rSVD) and the block mode
($r>1$). In the block mode, there is no mechanism to ensure the
orthogonality of the loadings. Here $Z=A^TY$ is not orthogonal, so
after thresholding, {$X$ also does not tend to be orthogonal}.
Besides, it is not easy to determine the weights, since lengths of
$Z_i$'s usually vary in great range. E.g., if we initialize
$Y=U_{1:r}$, then $Z=A^TY=V_{1:r}\Sigma_{1:r}$, which are scaled PCA
loadings whose lengths usually decay exponentially. Thus, if we
simply set the thresholds $\lambda_i$'s uniformly, it is easy to
lead to unbalanced sparsity among loadings, in which leading
loadings may be highly denser. This deviates from the goal of sparse
PCA. {For the deflation mode, though it produces nearly orthogonal
loadings, the greedy scheme makes its solution not optimal. And
there still exists a problem of how to set the weights
appropriately.\footnote{Even if $y$ is initialized with the
maximum-length column of $A$ as \cite{journee2010generalized} does,
it is likely to align with $U_1$.} Besides, for both modes,
performance analysis may be difficult to obtain}.

\subsection{Extending rSVD and GPower to
rSVD-GP}\label{sec:rSVD-GP}

A major drawback of rSVD and GPower is that they cannot use uniform
thresholds when applying thresholding $x=T_\lambda(z)$. The problem
does not exist in SPCArt since the inputs are of unit length. But,
we can extend the similar idea to GPower and rSVD: let
$x=\|z\|_2\cdot T_\lambda(z/\|z\|_2)$, which is equivalent to
truncating $z$ according to its length, or using adaptive thresholds
$x=T_{\lambda\|z\|_2}(z)$. The other truncation types T-en and T-sp
can be introduced into GPower too. T-sp is insensitive to length, so
there is no trouble in parameter setting; and the deflation version
happens to be TPower.

The deflation version of the improved algorithm rSVD-GP is shown in
Algorithm~\ref{alg:rSVD-GP}, and the block version rSVD-GPB is shown
in Algorithm~\ref{alg:rSVD-GPB}. rSVD-GPB follows the optimization
described in Section~\ref{sec:unified}. For rSVD-GP, since
$Polar(\cdot)$ reduces to normalization of vector, and the extended
truncation is insensitive to the length of input, we can combine the
$Polar$ step with the $Z=A^TY$ step and ignore the length during the
iterations. Besides, it is more efficient to work with the
covariance matrix, if $n>p$.

\begin{algorithm}[h]
   \caption{rSVD-GP (deflation version)}
   \label{alg:rSVD-GP}
\begin{algorithmic}[1]
   \STATE {\bfseries Input:} data matrix $A\in\mathbb{R}^{n\times p}$ (or covariance matrix
$C\in\mathbb{R}^{p\times p}$), number of loadings $r$, truncation
type $T$, parameter $\lambda$.
   \STATE {\bfseries Output:} $r$ sparse loading vectors $x_i\in\mathbb{R}^{p}$.
   \FOR{$i=1$ {\bfseries to} $r$}
   \STATE Initialize $x_i$: $j=\arg\max_k \|A_k\|_2$ (or $\arg\max_k C_{kk}$), set $x_{ij}=1$, $x_{ik}=0$, $\forall k\neq j$.

   \REPEAT
   \STATE $z=A^TAx_i$ (or $z=Cx_i$).
   \STATE Truncation: $x_i=T_\lambda(z/\|z\|_2)$.
   \UNTIL{convergence}
   \STATE Normalization: $x_i=x_i/\|x_i\|_2$.
   \STATE Deflation: $A=A(I-x_ix^T_i)$ (or $C=(I-x_ix^T_i)C(I-x_ix^T_i)$).
   \ENDFOR
\end{algorithmic}
\end{algorithm}

\begin{algorithm}[h]
   \caption{rSVD-GPB (block version)}
   \label{alg:rSVD-GPB}
\begin{algorithmic}[1]
   \STATE {\bfseries Input:} data matrix $A\in\mathbb{R}^{n\times p}$, number of loadings $r$, truncation type $T$,
parameter $\lambda$.
   \STATE {\bfseries Output:} sparse loadings $X=[X_1,\dots,X_r]\in\mathbb{R}^{p\times r}$.
   \STATE PCA: compute rank-$r$ SVD of $A$: $Y\Sigma V^T$.
   \REPEAT
   \STATE $Z=A^TY$.
   \STATE Truncation: $\forall i$, $X_i=\|Z_i\|_2\cdot T_\lambda(Z_i/\|Z_i\|_2)$.
   \STATE Update $Y$: thin SVD of $AX$: $WDQ^T$, $Y=WQ^T$.
   \UNTIL{convergence}
   \STATE Normalize $X$: $\forall i$, $X_i=X_i/\|X_i\|_2$.
\end{algorithmic}
\end{algorithm}

\section{Experiments} \label{sec:experiment}
The data sets used include: (1) a synthetic data with some
underlying sparse loadings \cite{zou2006sparse}; (2) the classical
Pitprops data \cite{jeffers1967two}; (3) a natural image data with
moderate dimension and relatively large sample size, on which
comprehensive evaluations are conducted; (4) a gene data with high
dimension and small sample size \cite{golub1999molecular}; (5) a set
of random data with increasing dimensions for the purpose of speed
test.

We compare our methods with five methods: SPCA \cite{zou2006sparse},
PathSPCA \cite{dAspremont2008optimal}, ALSPCA
\cite{lu2009augmented}, GPower \cite{journee2010generalized}, and
TPower \cite{yuan2013truncated}. For SPCA, we use toolbox
\cite{sjstrand2005matlab}, which implements $\ell_0$ and $\ell_1$
constraint versions. We use GPowerB to denote the block version of
GPower, as rSVD-GPB. We use SPCArt(T-$\ell_0$) to denote SPCArt
using T-$\ell_0$; the other methods use the similar abbreviations.
Note that, rSVD-GP(T-sp) is equivalent to TPower
\cite{yuan2013truncated}. Except our SPCArt and rSVD-GP(B), the
codes of the others are downloaded from the authors' websites.

There are mainly five criteria for the evaluation. (1) SP: mean of
sparsity of loadings. (2) STD: standard deviation of sparsity of
loadings. (3) CPEV: cumulative percentage of explained variance
(that of PCA loadings is CPEV(V)). (4) NOR: nonorthogonality of
loadings, $1/(r(r-1))\sum_{i\neq j} |\cos \theta (X_i,X_j)|$ where
$r$ is the number of loadings. (5) Time cost, including the
initialization. Sometimes we may use the worst sparsity among
loadings, $min_i (1-\|X_i\|_0/p)$, instead of STD, when it is more
appropriate to show the imbalance of sparsity.

All methods involved in the comparison have only one parameter
$\lambda$ that induces sparsity. For those methods that have direct
control on sparsity, we view them as belonging to T-sp and let
$\lambda$ denote the number of zeros of a vector. GPowerB is
initialized with PCA, and its parameters are set as $\mu_j=1$,
$\forall j$ and $\lambda$'s are uniform for all loadings.\footnote
{The original random initialization for r-1 vectors
\cite{journee2010generalized} may fall out of data subspace and
result in zero solution. When using PCA as initialization, distinct
$\mu_j$ setting in effect artificially alters data variance.} For
ALSPCA, since we do not consider correlation among principal
components, we set $\Delta_{ij}=+\infty$, $\epsilon_I=+\infty$,
$\epsilon_E=0.03$, and $\epsilon_O=0.1$. In SPCArt, for T-$\ell_0$
and T-$\ell_1$ we set $\lambda=1/\sqrt{p}$ by default, since it is
the minimal threshold to ensure sparsity and the maximal threshold
to avoid truncating to zero vector. The termination conditions of
SPCArt, SPCA are the relative change of loadings
$\|X^{(t)}-X^{(t-1)}\|_F/\sqrt r<0.01$ or iterations exceed $200$.
{rSVD-GP(B) uses similar setting.} All codes are implemented using
MATLAB, run on a computer with 2.93GHz duo core CPU and 2GB memory.

\begin{table}[t]
\caption{Recovering of sparse loadings on a synthetic data. $r=2$.
CPEV(V) = 0.9973. Loading pattern 5-10; 1-4,9-10 means the nonzero
support of the first loading vector is 5 to 10, and the second is 1
to 4 and 9 to 10.} \label{synthetic data} \vskip 0.15in
\begin{centering}
\begin{small} %
\begin{tabular}{|c|c||c|c|c|}
\hline \multicolumn{2}{|c||}{{ algorithm}} & $\lambda$ &
\tabincell{c}{loading\\ pattern}  & CPEV\\\hline\hline
 \multirow{2}{*}{SPCA}
& T-sp  & 4  & 5-10; 1-4  & 0.9848\\\cline{2-5}
 & $\ell_1$  & 2.2  & 1-4,9-10; 5-8  & 0.8286\\\hline

PathSPCA & T-sp & 4  & 5-10; 1-4,9-10 &0.9960\\\hline

\multicolumn{2}{|c||}{ALSPCA } & 0.7  & 5-10; 1-4 & 0.9849\\\hline

 \multirow{4}{*}{rSVD-GP} & T-$\ell_0$  & $1/\sqrt{p}$ & 5-10;
1-4  & 0.9849\\\cline{2-5}
 & T-$\ell_1$  & $1/\sqrt{p}$  & 5-10; 1-4  & 0.9808\\\cline{2-5}
 & T-sp  & 4  & 5-10; 1-4,9-10  & 0.9960\\\cline{2-5}
 & T-en  & 0.1  & 5-10; 1-4  & 0.9849\\\hline

\multirow{4}{*}{SPCArt} & T-$\ell_0$  & $1/\sqrt{p}$  & 5-10; 1-4 &
0.9848\\\cline{2-5}
 & T-$\ell_1$  & $1/\sqrt{p}$  & 5-10; 1-4  & 0.9728\\\cline{2-5}
 & T-sp  & 4  & 5-10; 1-4,9-10  & 0.9968\\\cline{2-5}
 & T-en  & 0.1  & 5-10; 1-4  & 0.9848\\
\hline
\end{tabular}
\end{small}
\par\end{centering}
\vskip -0.1in
\end{table}

\begin{table}[h]
\caption{A comparison of algorithms on the Pitprops data. $r=6$,
CPEV(V) = 0.8700. Loading patterns here describe the cardinality of
each loading vector.}\label{pitprops} \vskip 0.1in
\begin{centering}
\begin{small}
\begin{tabular}{|c|c||c|c|c|c|c|c|}
\hline \multicolumn{2}{|c||}{{ algorithm}} & $\lambda$ & NZ &
\tabincell{c}{loading\\ patterns}  & STD  & NOR & CPEV\\\hline\hline

\multicolumn{2}{|c||}{ALSPCA } & 0.65 & 17 & 722213 & 0.1644 &
0.0008 & 0.8011\\\hline\hline

\multirow{5}{*}{T-$\ell_0$}

& GPower & 0.1 & 19 & 712162 & 0.2030 & 0.0259 & 0.8111\\\cline{2-8}

& rSVD-GP & 0.27 & 17 & 612422 & 0.1411 & 0.0209 &
\tb{0.8117}\\\cline{2-8}

& GPowerB & 0.115 & 17 & 724112 & 0.1782 & 0.0186  &
0.8087\\\cline{2-8}

& rSVD-GPB & 0.3 & 18 & 534132 & 0.1088 & 0.0222 &
0.7744\\\cline{2-8}

& SPCArt & $1/\sqrt{p}$ & 18 & 424332 & \tb{0.0688} & \tb{0.0181} &
0.8013\\\cline{2-8}\hline\hline

\multirow{5}{*}{T-sp}

& SPCA & 10  & 18 & 333333 & 0 & \tb{0.0095} & 0.7727\\\cline{2-8}

& PathSPCA & 10 & 18 & 333333 & 0 & 0.0484 &
\tb{0.7840}\\\cline{2-8}

& rSVD-GP & 10 & 18 & 333333  & 0 & 0.0455 & 0.7819\\\cline{2-8}

& rSVD-GPB & 10 & 18 & 333333  & 0 & 0.0525  & 0.7610\\\cline{2-8}

& SPCArt & 10 & 18 & 333333  & 0 & 0.0428  &
0.7514\\\cline{2-8}\hline
\end{tabular}
\end{small}
\par\end{centering}

\vskip -0.1in
\end{table}

\begin{figure*}[h]
\center{ \subfigure[Relative change of
$X$]{\label{fig:SPCArt-l0:dx}\includegraphics[width=4cm]{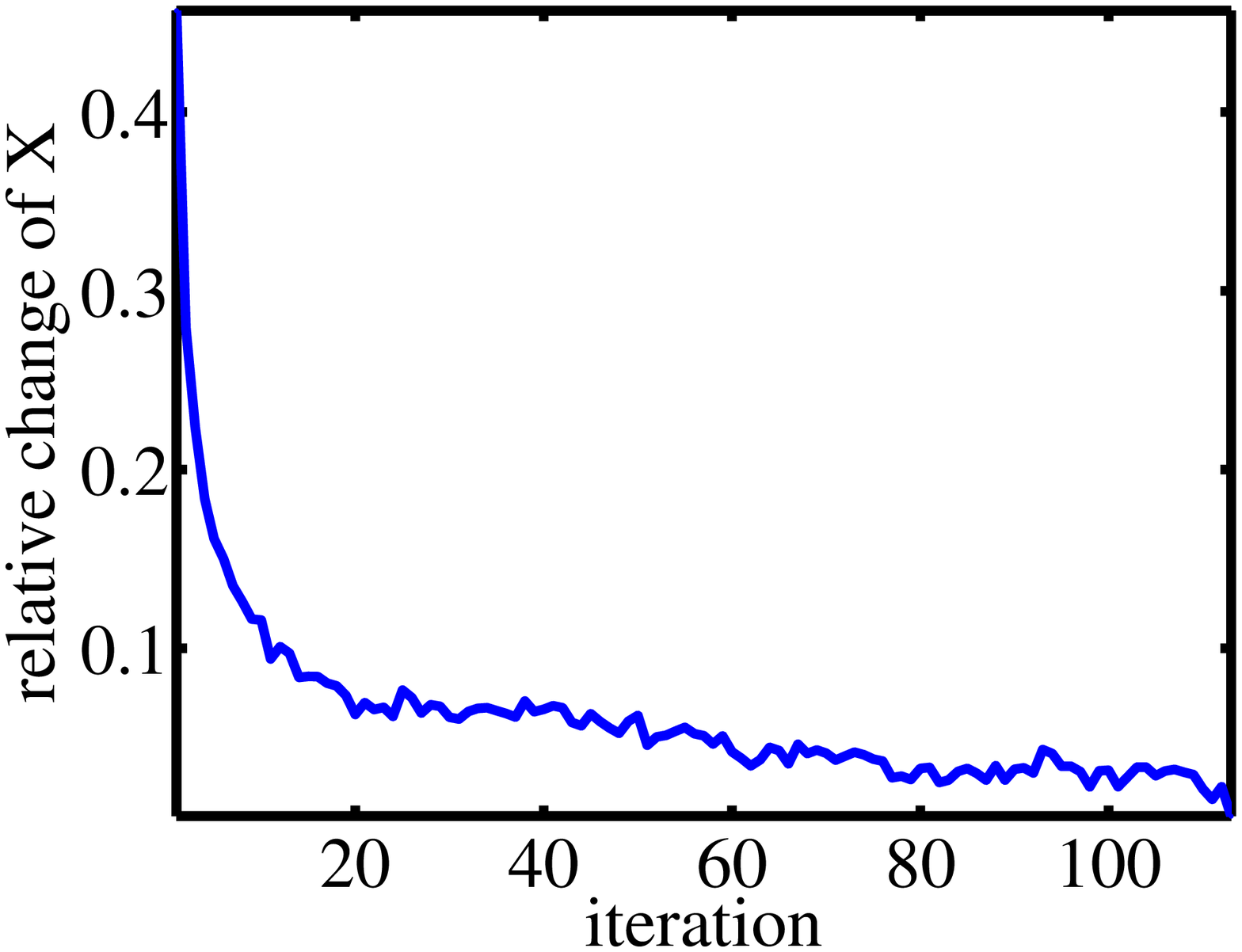}}
\subfigure[Truncated
energy]{\label{fig:SPCArt-l0:en}\includegraphics[width=4cm]{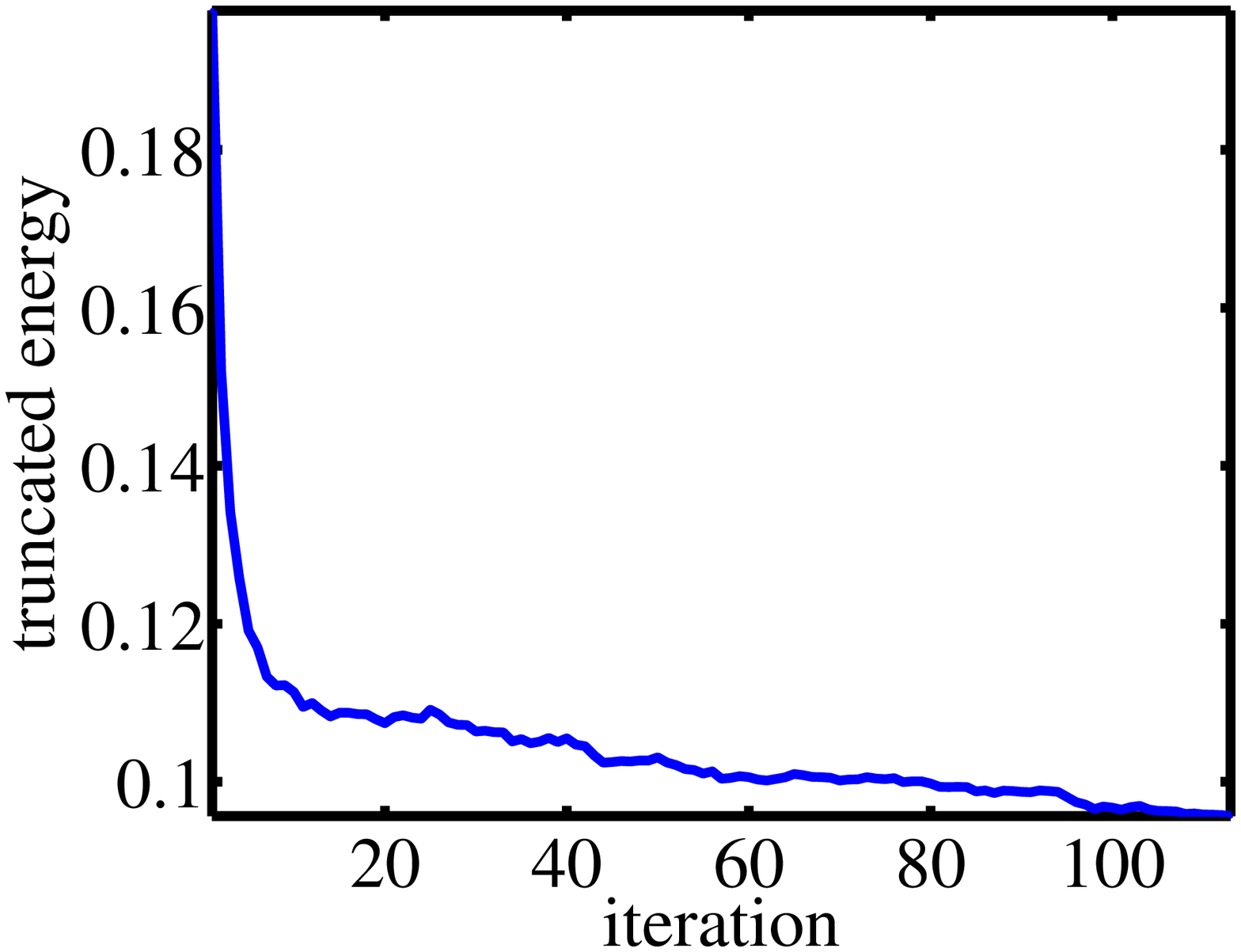}}\\
\subfigure[SP]{\label{fig:SPCArt-l0:sp}\includegraphics[width=4cm]{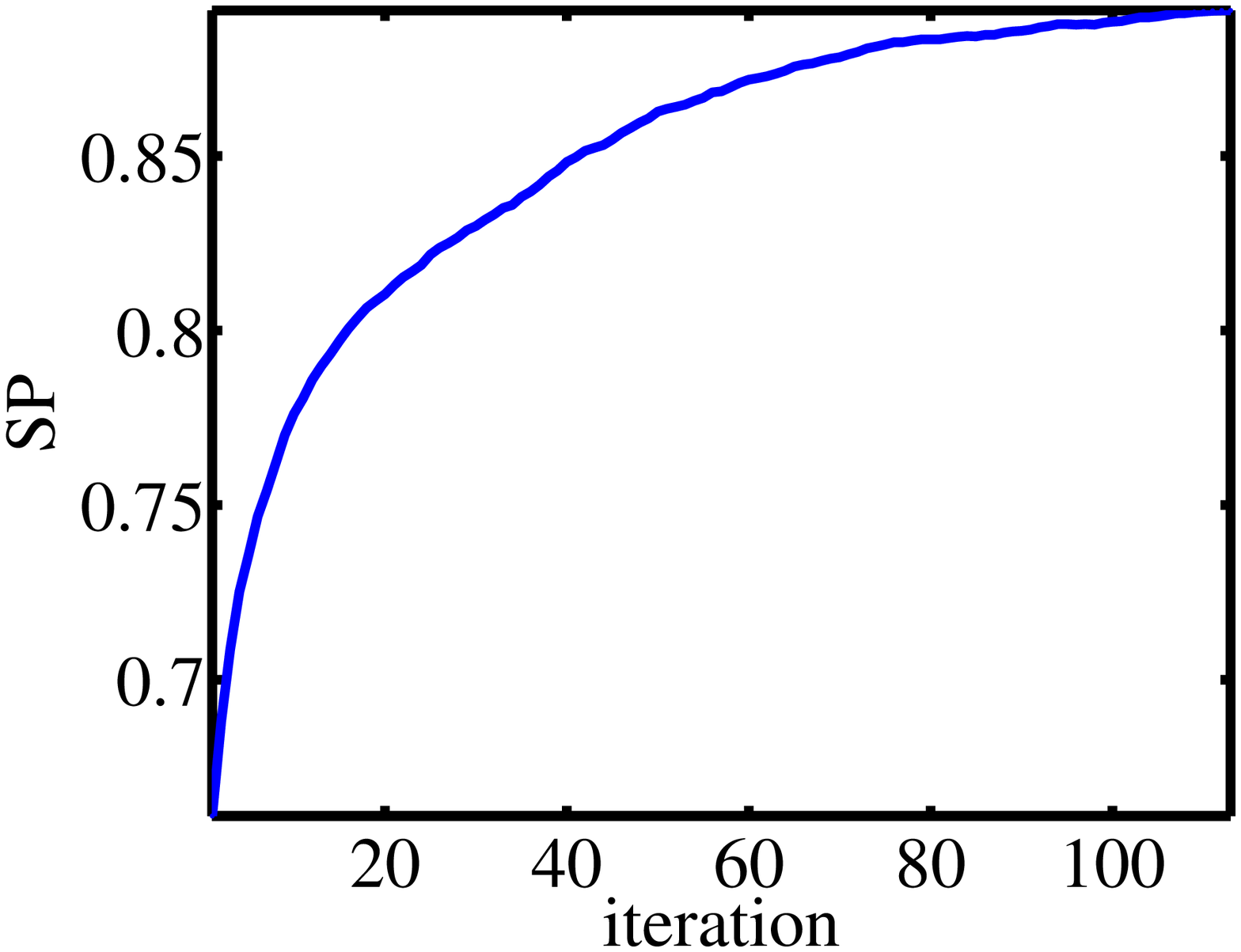}}
\subfigure[CPEV]{\label{fig:SPCArt-l0:va}\includegraphics[width=4cm]{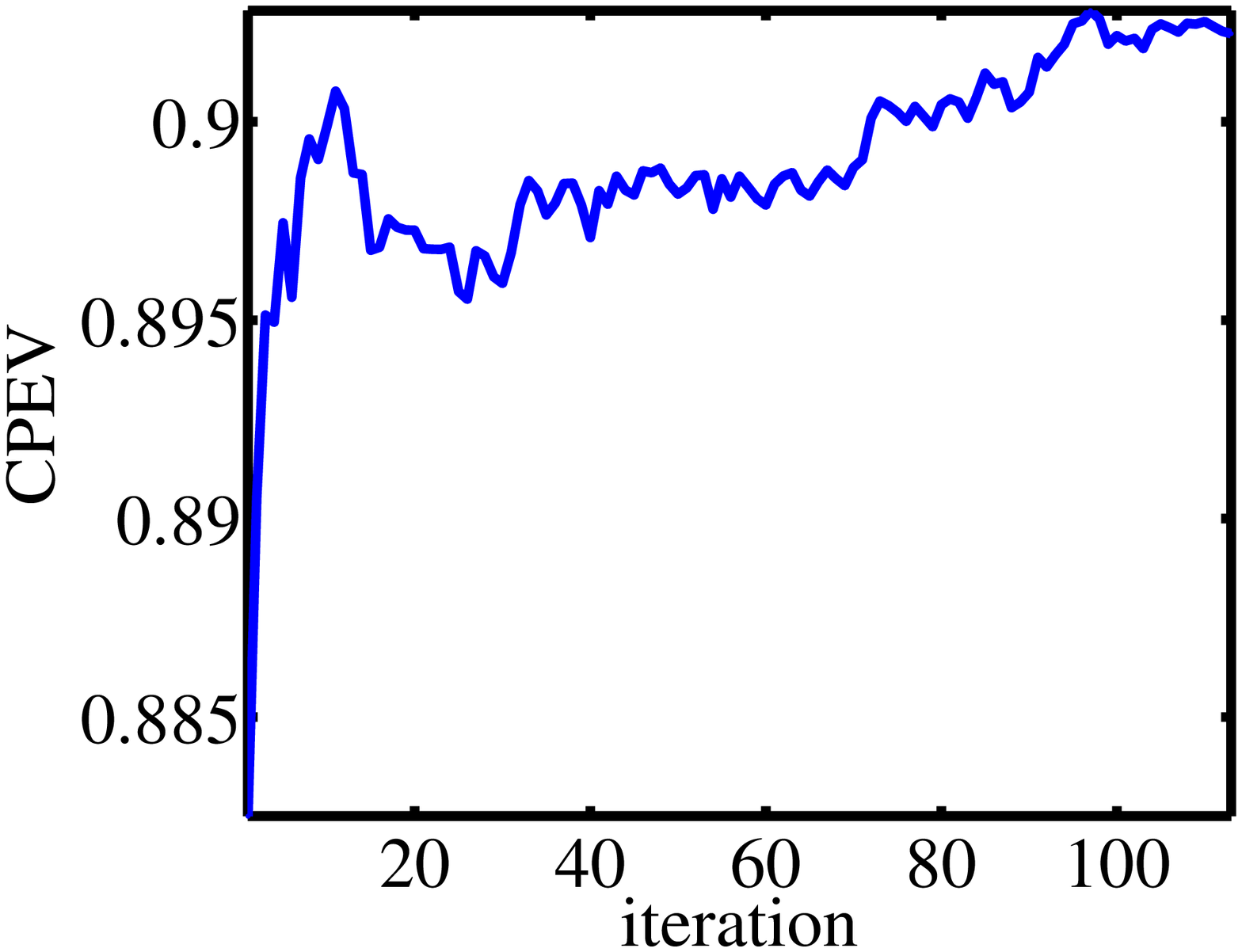}}
\subfigure[NOR]{\label{fig:SPCArt-l0:or}\includegraphics[width=4cm]{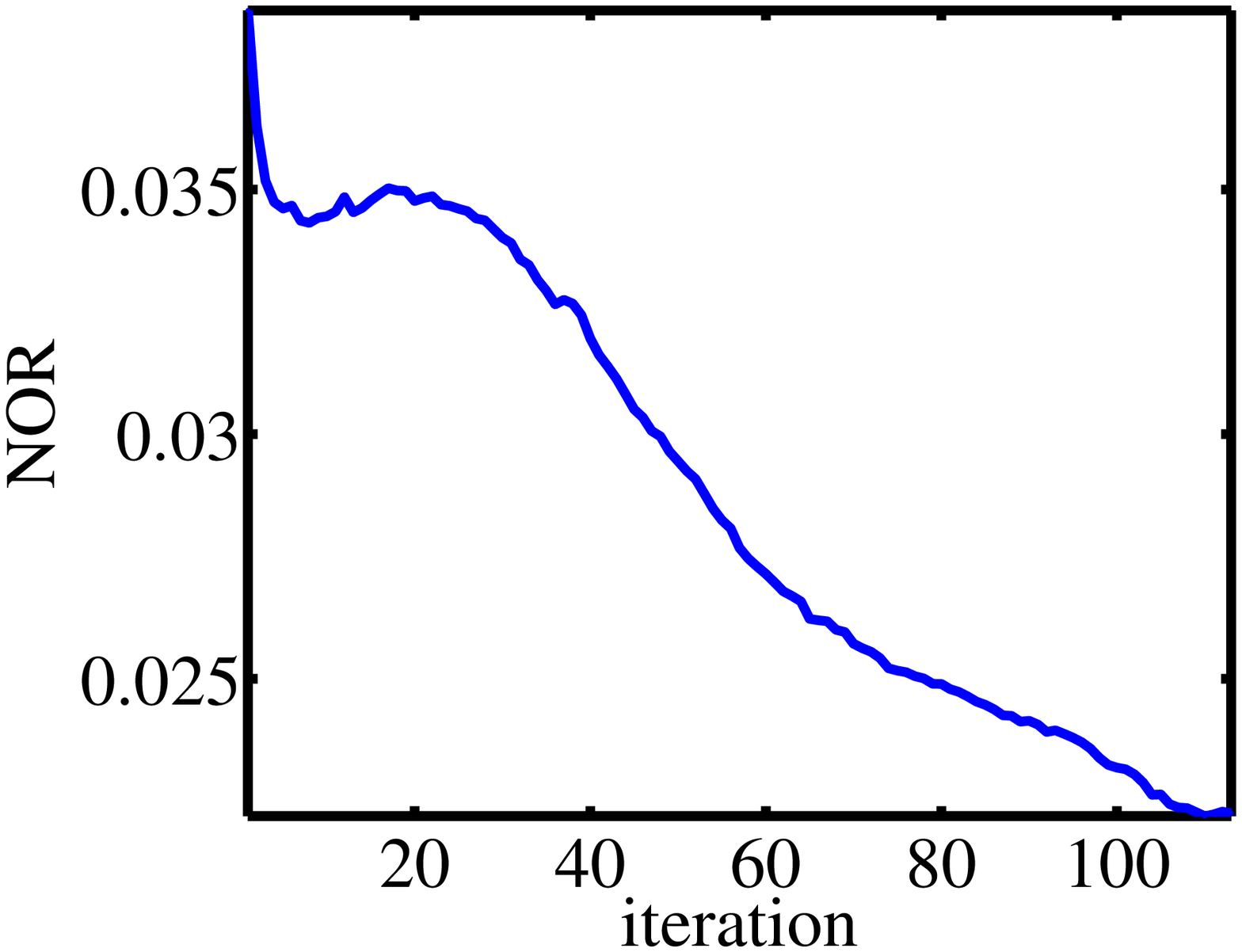}}
} \caption{Convergence of SPCArt(T-$\ell_0$) on image data. The
convergence is relatively stable, and the criteria improve along
with the iteration. }\label{fig:SPCArt-l0}
\end{figure*}

\begin{figure*}[h]
\centering{
\subfigure[SP]{\label{fig:SPCArt-en:sp}\includegraphics[width=5cm]{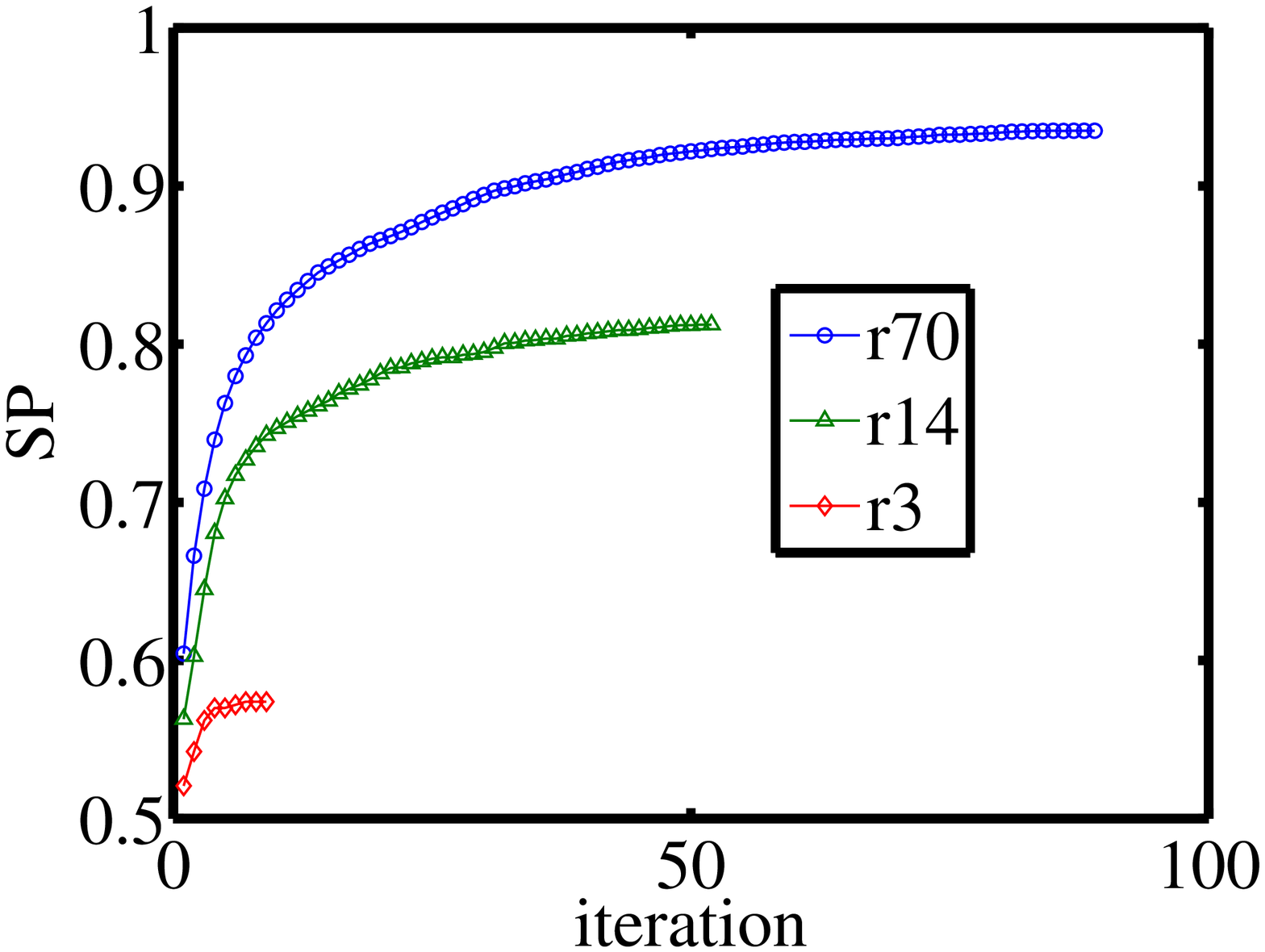}}
\subfigure[NOR]{\label{fig:SPCArt-en:or}\includegraphics[width=5cm]{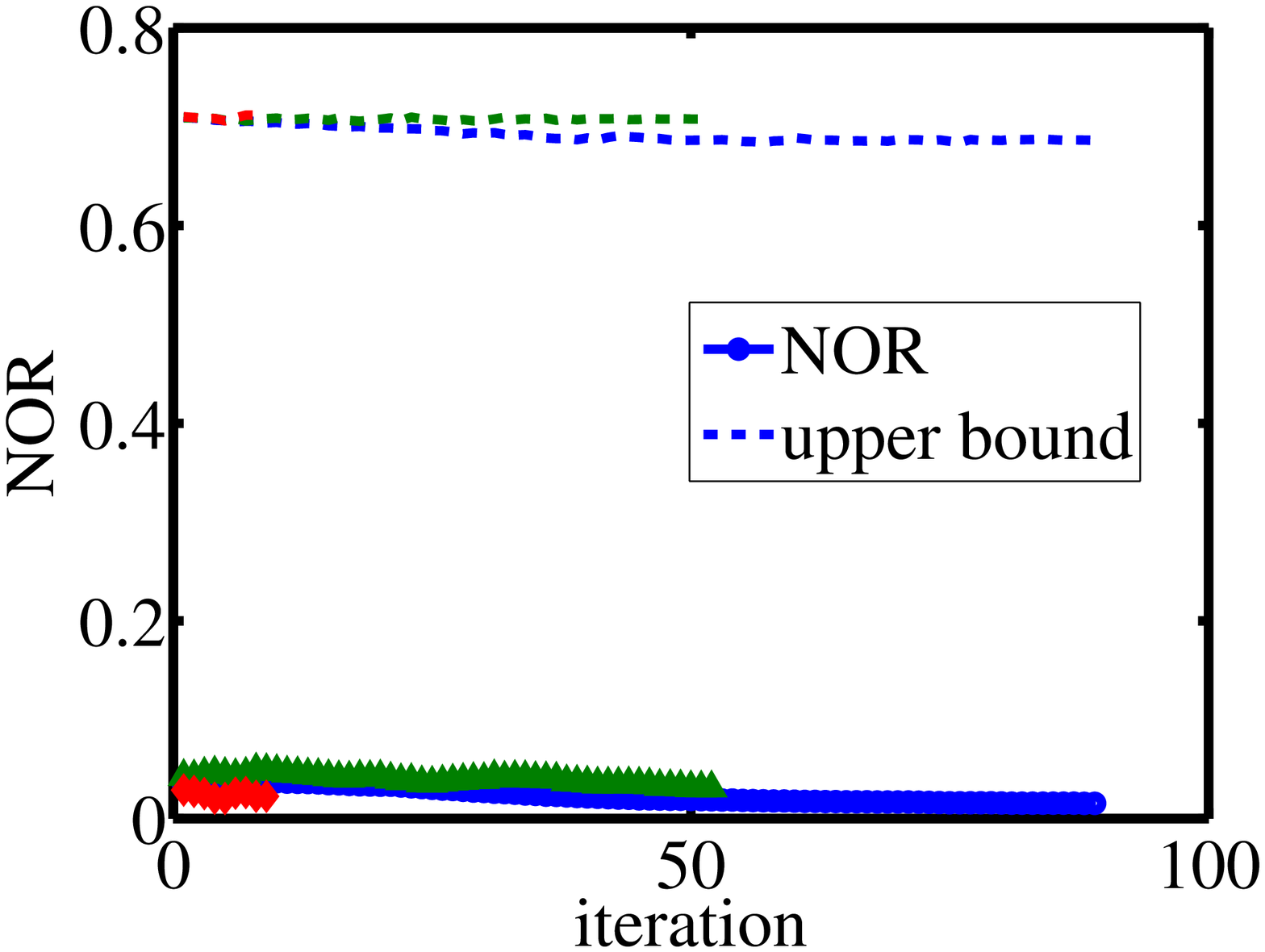}}\\
\subfigure[CPEV,
$r=3$]{\label{fig:SPCArt-en:va3}\includegraphics[width=5cm]{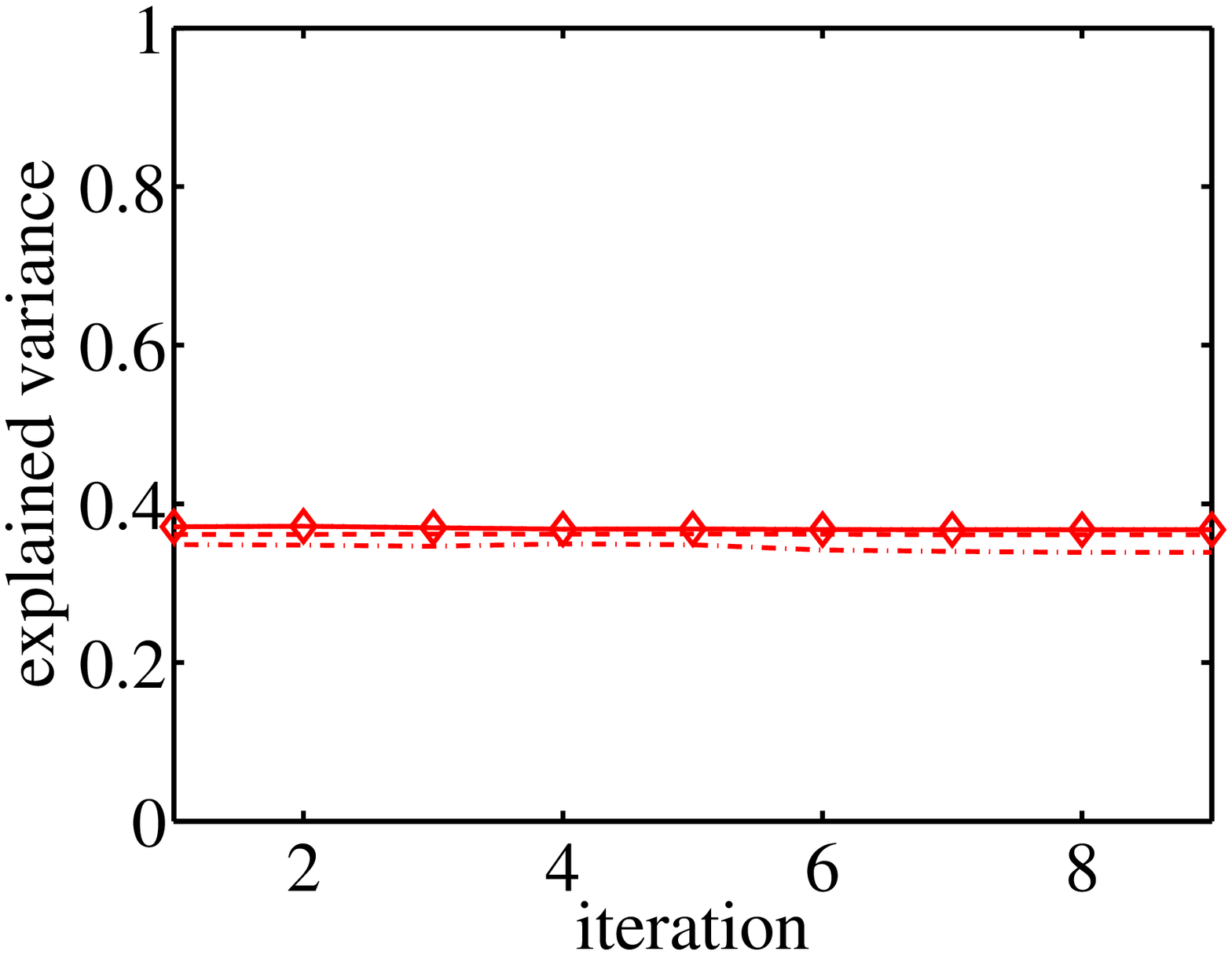}}
\subfigure[CPEV,
$r=14$]{\label{fig:SPCArt-en:va14}\includegraphics[width=5cm]{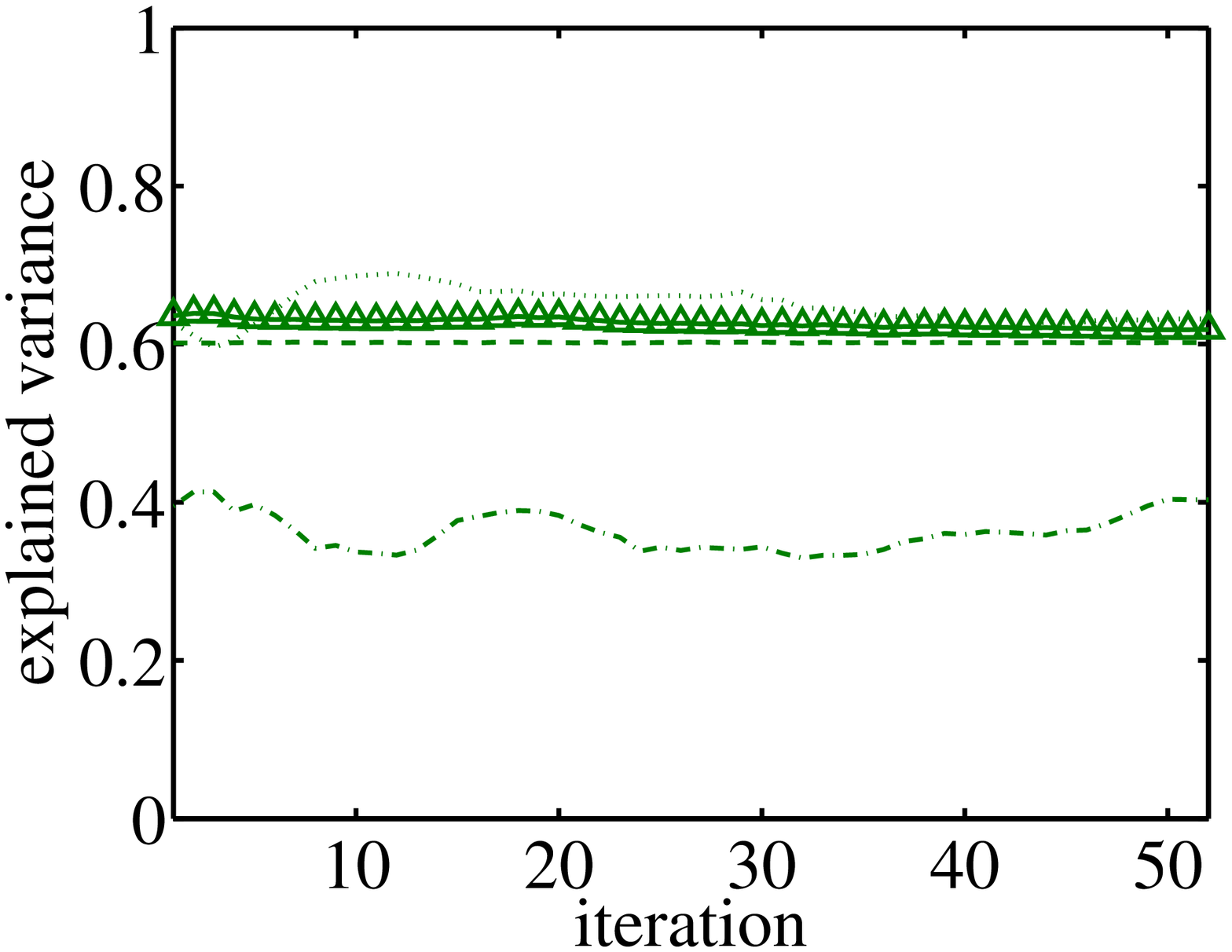}}
\subfigure[CPEV,
$r=70$]{\label{fig:SPCArt-en:va70}\includegraphics[width=5cm]{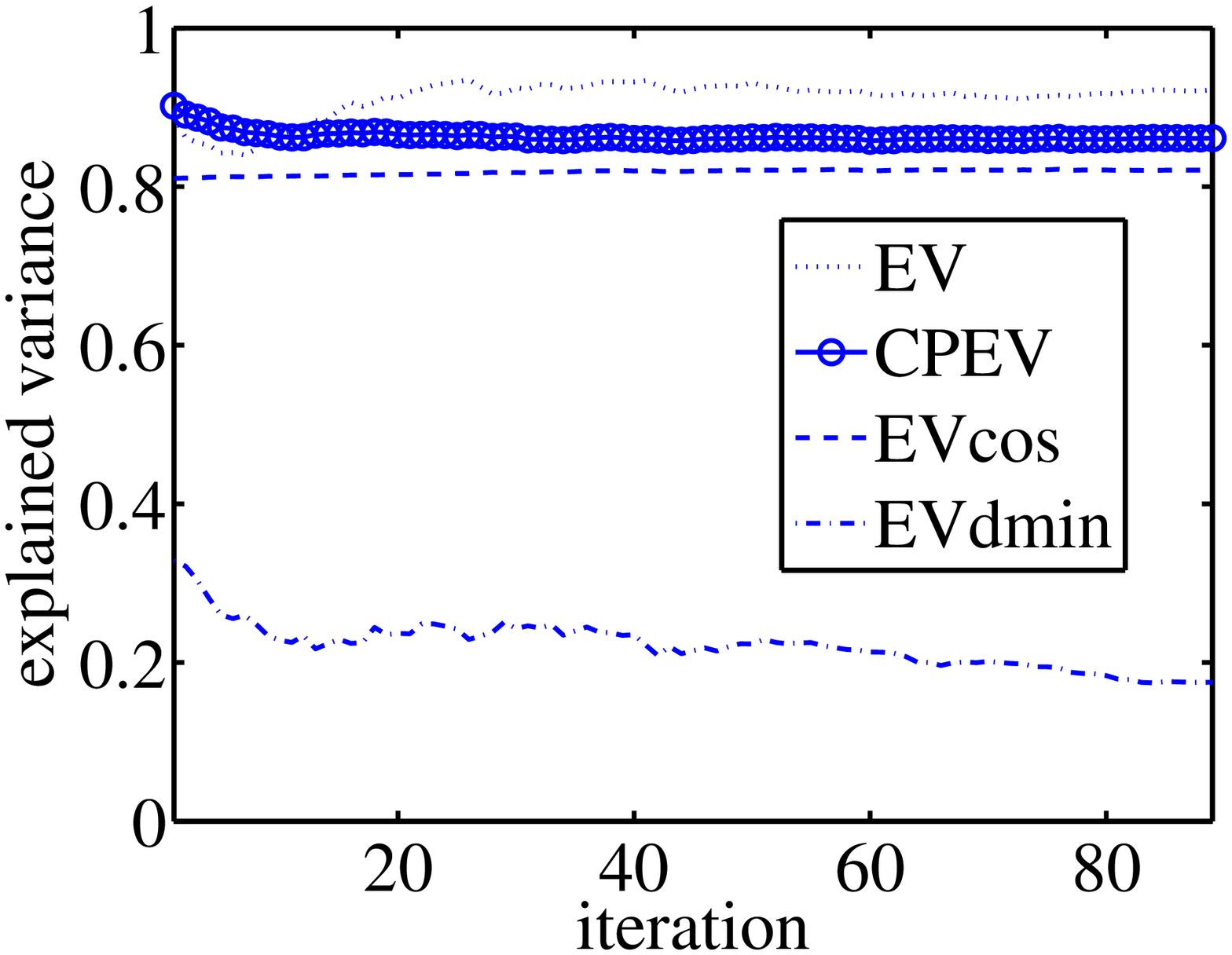}}
} \caption{Performance bounds of SPCArt(T-en) for 3 levels of $r$ on
image data. The legends of (c) and (d) are similar to that of (e).
EV is a normalized version EV = $tr(X^TA^TAX)/tr(A^TA)$ so that it
can be compared with CPEV. EVdmin = $d_{min}^2EV(V)/tr(A^TA)$ and
EVcos = $\cos^2(\theta)EV(V)/tr(A^TA)$. We see EVcos is better than
EVdmin in estimation, and EVcos meets empirical performance well.
For NOR, the algorithm performs far optimistic than those upper
bounds. Owning to the good orthogonality, EV are comparable to CPEV.
For each $r$, as iteration goes, SP improves a lot while CPEV
sacrifices little.}\label{fig:SPCArt-en}
\end{figure*}

\subsection{Synthetic Data}
We will test whether SPCArt and rSVD-GP can recover some underlying
sparse loadings. The synthetic data was introduced by
\cite{zou2006sparse} and became classical for sparse PCA problem. It
considers three hidden Gaussian factors: $h_{1}\sim N(0,290)$,
$h_{2}\sim N(0,300)$, $h_{3}=-0.3h_{1}+0.925h_{2}+\epsilon$,
$\epsilon\sim N(0,1)$. Then ten variables are generated:
$a_{i}=h_{j}+\epsilon_{i}$, $\epsilon_{i}\sim N(0,1)$,
$i=1,\dots,10$, with $j=1$ for $i=1,\dots,4$, $j=2$ for
$i=5,\dots,8$, $j=3$ for $i=9,10$. In words, $h_1$ and $h_2$ are
independent, while $h_3$ has correlations with both of them,
particularly $h_2$. The first 1-4 variables are generated by $h_1$,
while the 5-8 variables are generated by $h_2$. So these two sets of
variables are independent. The last variables 9-10 are generated by
$h_3$, so they have correlations with both of the 1-4 and 5-8
variables, particularly the latter. The covariance matrix $C$
determined by $a_i$'s is fed into the algorithms. For those
algorithms that only accept data matrix, an artificial data
$\tilde{A}=V\Sigma^{-1/2}V^T$ is made where $V\Sigma V^T=C$ is the
SVD of $C$. This is reasonable since they share the same loadings.

The algorithms are required to find two sparse loadings. Besides
CPEV, the nonzero supports of the loadings are recorded, which
should be consistent with the above generating model. The results
are reported in Table~\ref{synthetic data}. Except SPCA(T-$\ell_1$),
the others, including SPCArt and rSVD-GP, successfully recovered the
two most acceptable loading patterns 1-4,9-10; 5-10 and 1-4; 5-10,
as can be seen from the CPEV. \footnote{Setting $\lambda=6$ for
T-sp, all recover another well accepted 5-8; 1-4 pattern, see
\cite{shen2008sparse} for detail.}

\subsection{Pitprops Data}

The Pitprops data is a classical data to test sparse PCA
\cite{jeffers1967two}. There is a covariance matrix of 13 variables
available: $C\in\mathbb{R}^{13\times13}$. For those algorithms that
only accept data matrix as input, an artificial data matrix
$A=V\Sigma^{-1/2}V^T$ is made where $V\Sigma V^T=C$. The algorithms
are tested to find $r=6$ sparse loadings. For fairness, $\lambda$'s
are tuned so that each algorithm yields total cardinality of all
loadings, denoted by NZ, about 18; and mainly T-sp and T-$\ell_0$
algorithms are tested. Criteria STD, NOR, and CPEV are reported. The
results are shown in Table~\ref{pitprops}. For T-$\ell_0$, SPCArt
does best overall, although its CPEV is not the best. The others,
especially GPower(B), suffer from unbalanced cardinality, as can be
seen from the loading patterns and STD; their CPEV may be high but
they are mainly contributed by the dense leading vector, which
aligns with the direction of maximal variance, i.e. leading PCA
loading. The improvements of rSVD-GP(B) over GPower(B) on this point
is significant, as can be seen from the tradeoff between STD and
CPEV. For T-sp, focusing on NOR and CPEV, the performance of rSVD-GP
is good while that of SPCArt is somewhat bad, for the CPEV is the
worst although the NOR ranks two.

\subsection{Natural Image Data}
The investigation of the distribution of natural image patches is
important for computer vision and pattern recognition communities.
On this data set, we will evaluate the convergence of SPCArt, the
performance bounds, and make a comprehensive comparisons between
different algorithms. We randomly select 5000 gray-scale
$13\times13$ patches from BSDS \cite{Martin01}. Each patch is
reshaped to a vector of dimension 169. The DC component of each
patch is removed first, and then the mean of the data set is
removed.

\subsubsection{Convergence of SPCArt}
We will show the stability of convergence and the improvement of
SPCArt over simple thresholding \cite{cadima1995loading}. We take
T-$\ell_0$, $r=70$ as example. CPEV(V) = 0.95. The results are shown
in Figure~\ref{fig:SPCArt-l0}. Gradually, SPCArt has found a local
optimal rotation such that less truncated energy from the rotated
PCA loadings is needed (Figure~\ref{fig:SPCArt-l0:en}) to get a
sparser (Figure~\ref{fig:SPCArt-l0:sp}), more variance explained
(Figure~\ref{fig:SPCArt-l0:va}), and more orthogonal
(Figure~\ref{fig:SPCArt-l0:or}) basis. Note that, the results in the
first iteration are equal to those of simple thresholding
\cite{cadima1995loading}. The final solution of SPCArt significantly
improves over simple thresholding.

\subsubsection{Performance Bounds of SPCArt}
We now compare the theoretical bounds provided in
Section~\ref{sec:analysis} with empirical performance. T-en with
$\lambda=0.15$ is taken as example, in which about 85\% EV(V) is
guaranteed. To achieve a more systematic evaluation, three levels of
subspace dimension are tested: $r$ = [3 14 70] with the
corresponding CPEV(V) = [0.42 0.71 0.95]. The results are shown in
Figure~\ref{fig:SPCArt-en}. Note that most of the theoretical bounds
are the absolute bounds without assuming the specific distribution
of data, so they may be very different from the empirical
performance.

First, for sparsity, the lower bound given in (\ref{equ:en}) is
about 15\%. But as seen in Figure~\ref{fig:SPCArt-en:sp}, the
empirical sparsity is far better than expectation, especially when
$r$ is larger.

Similar situation occurs for nonorthogonality, as seen in
Figure~\ref{fig:SPCArt-en:or}. The upper bounds are far too
pessimistic to be practical. It may be caused by the high dimension
of data.

Finally, for explained variance, it can be found from
Figure~\ref{fig:SPCArt-en:va3}, \ref{fig:SPCArt-en:va14},
\ref{fig:SPCArt-en:va70} that there is no large discrepancy between
EV and CPEV, owning to the near orthogonality of the sparse basis as
indicated in Figure~\ref{fig:SPCArt-en:or}. On the other hand, the
specific bound EVcos is better than the universal bound EVdmin. In
contrast to sparsity and nonorthogonality, EVcos meets the empirical
performance well, as analyzed in Section~\ref{sec:analysis}.

\begin{figure*}[h]
\centering{
\subfigure[CPEV]{\label{fig:GP-GPower:va}\includegraphics[width=4cm]{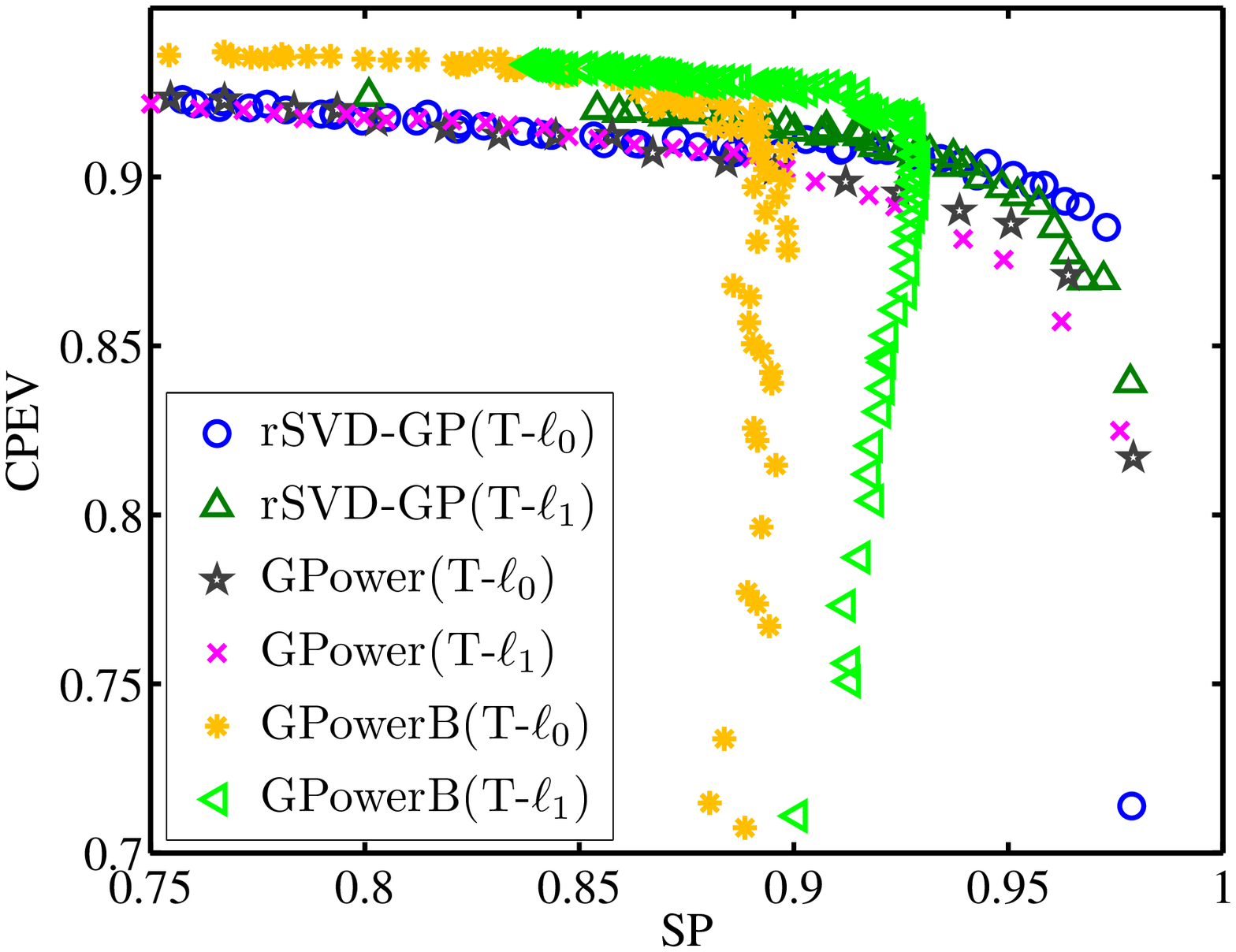}}
\subfigure[NOR]{\label{fig:GP-GPower:or}\includegraphics[width=4cm]{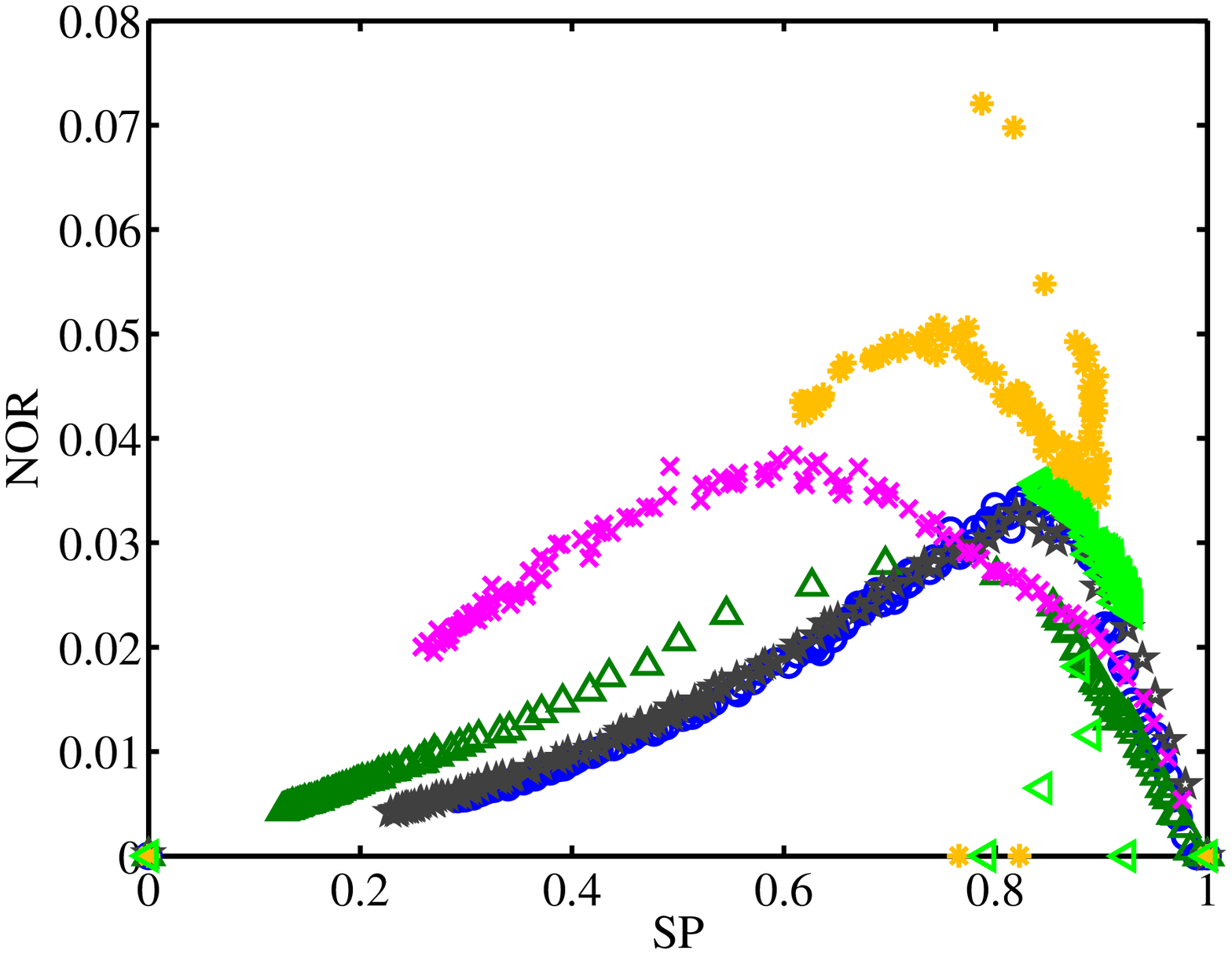}}
\subfigure[Worst
sparsity]{\label{fig:GP-GPower:std}\includegraphics[width=4cm]{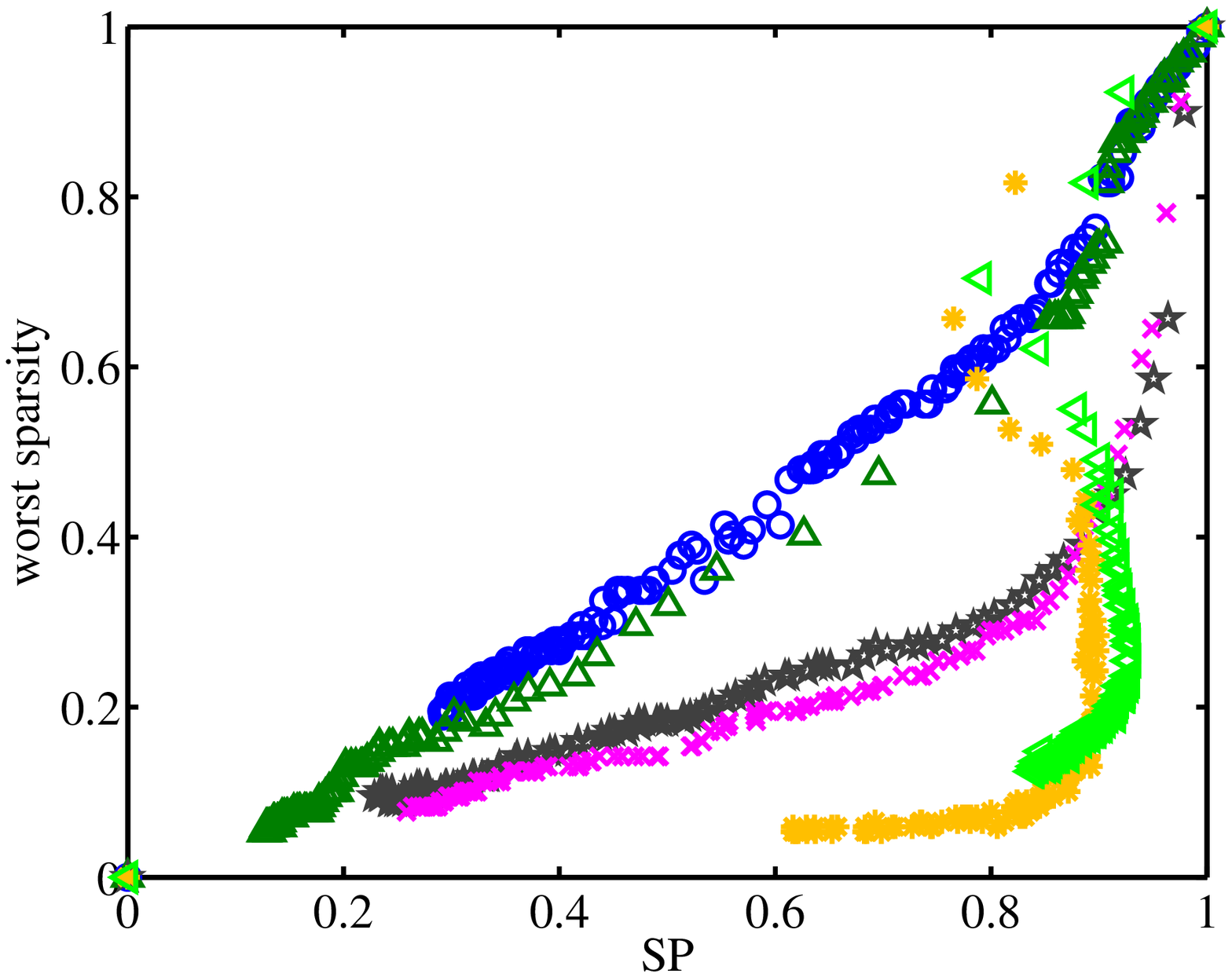}}
\subfigure[Time
cost]{\label{fig:GP-GPower:tm}\includegraphics[width=4cm]{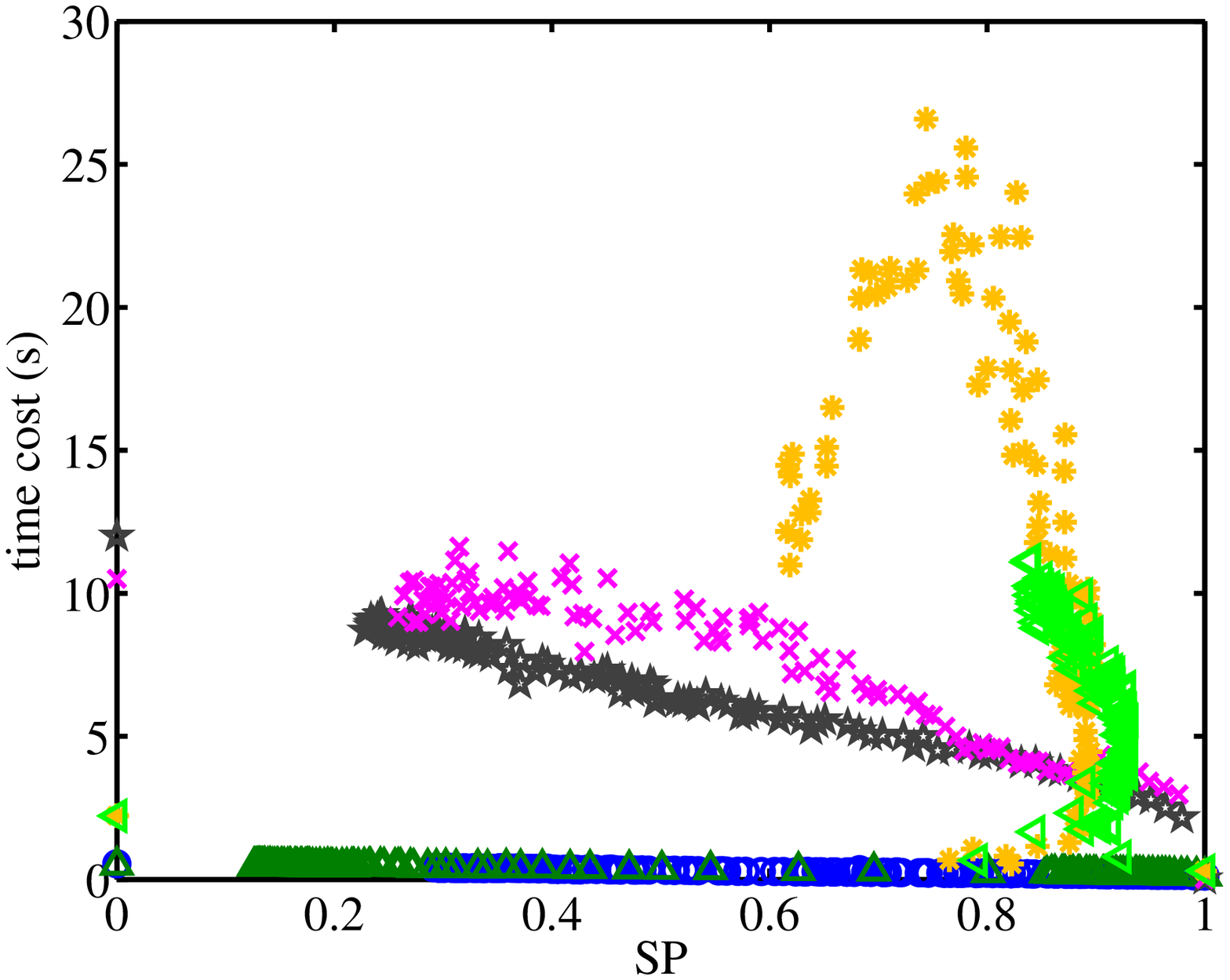}}
} \caption{rSVD-GP v.s. GPower(B) on image data. From (c), we see
that for GPower(B), the uniform parameter setting leads to
unbalanced sparsity, the worst case is rather dense. rSVD-GP
significantly improves over GPower(B) on the balance of sparsity as
well as the other criteria.}\label{fig:GP-GPower}
\end{figure*}


\begin{figure*}[h]
\centering{
\subfigure[CPEV]{\label{fig:GP-GPB:va}\includegraphics[width=4cm]{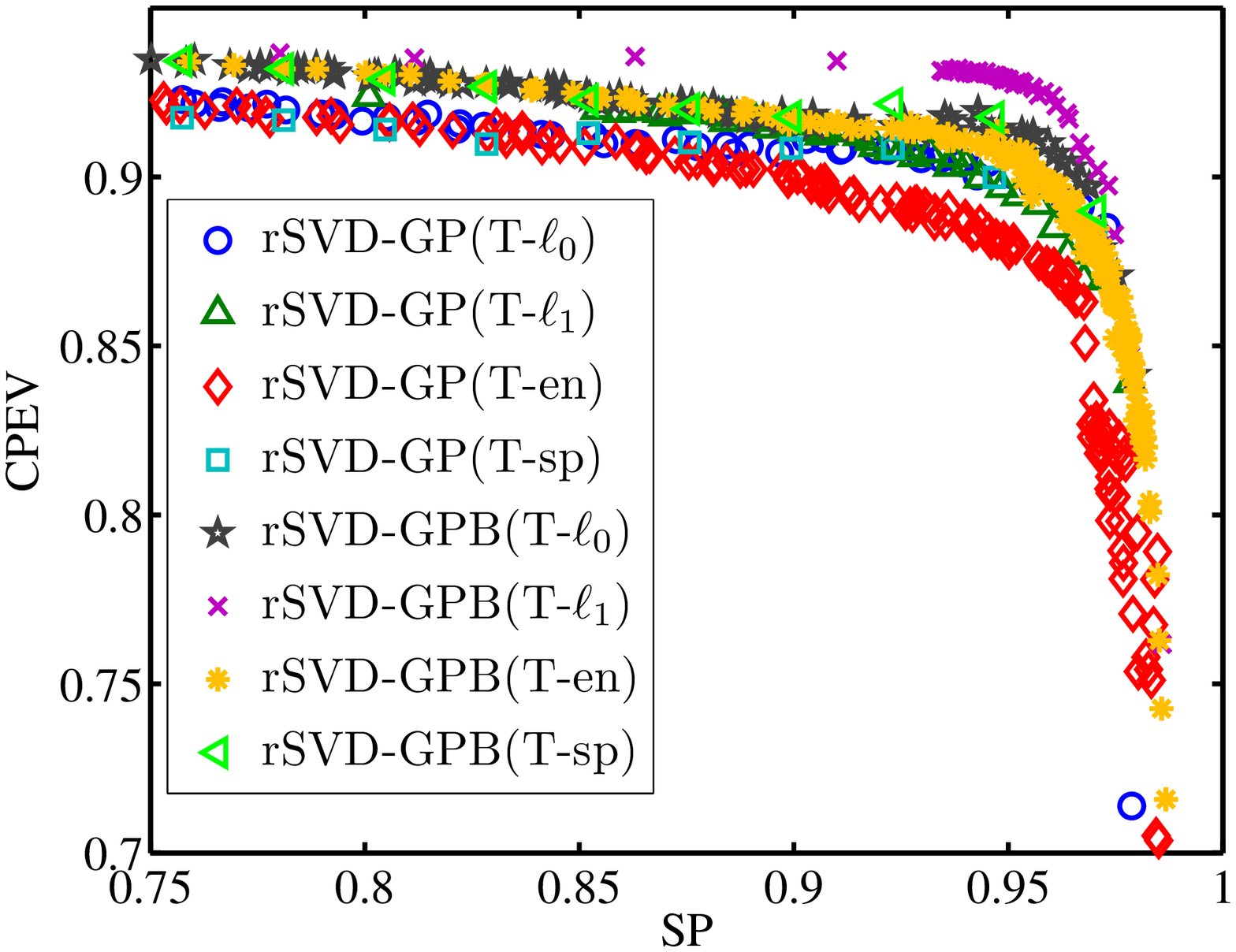}}
\subfigure[NOR]{\label{fig:GP-GPB:or}\includegraphics[width=4cm]{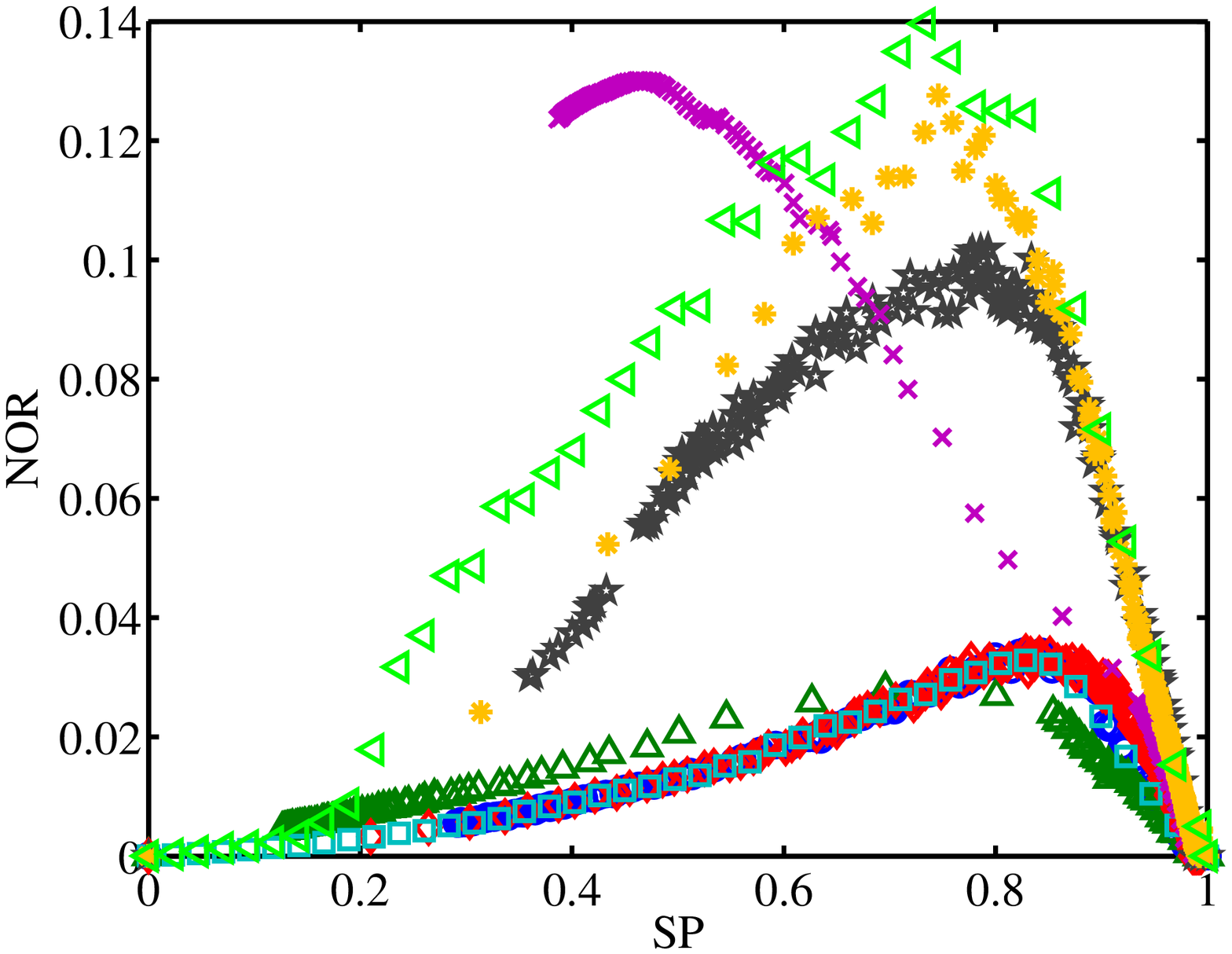}}
\subfigure[STD]{\label{fig:GP-GPB:std}\includegraphics[width=4cm]{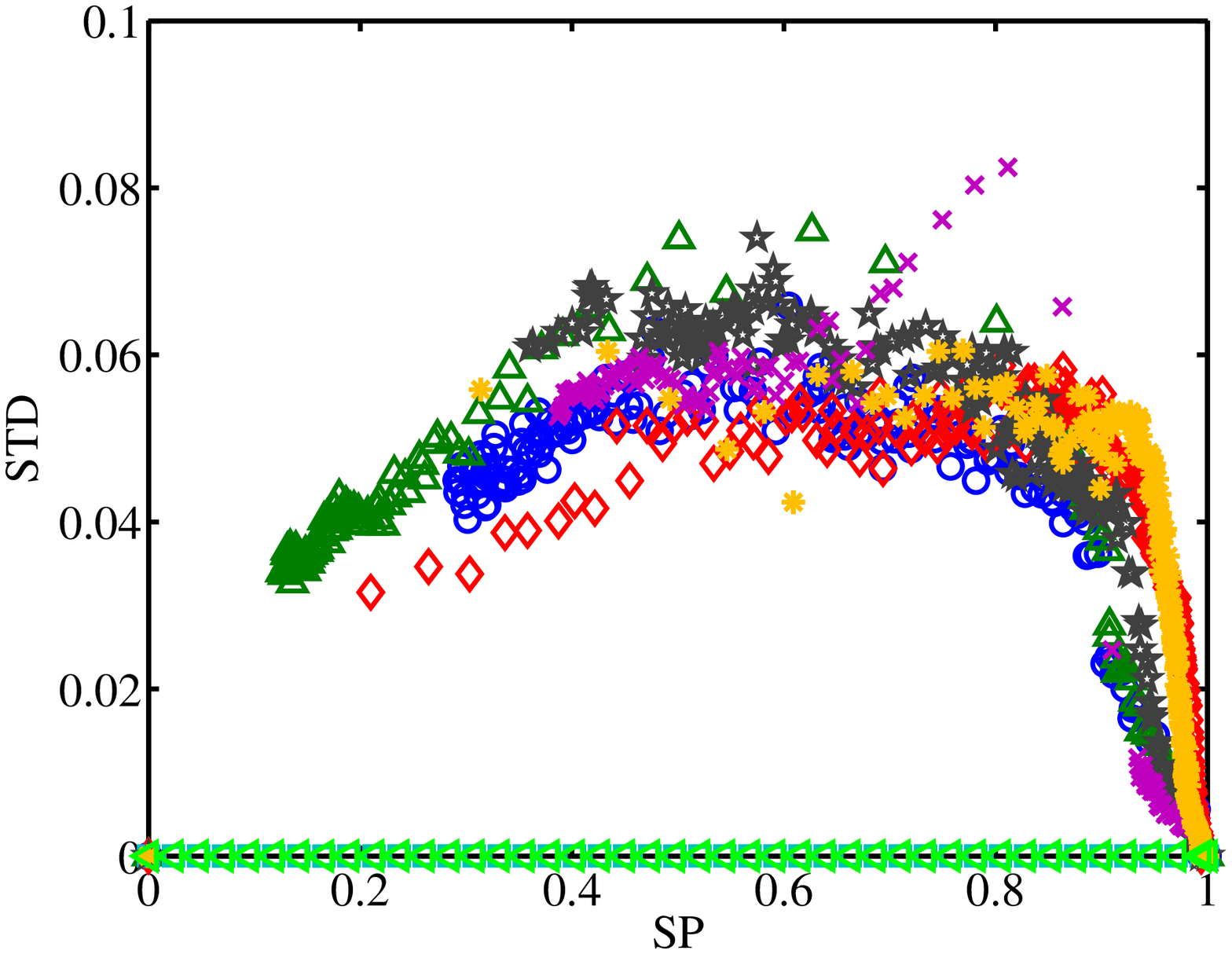}}
\subfigure[Time
cost]{\label{fig:GP-GPB:tm}\includegraphics[width=4cm]{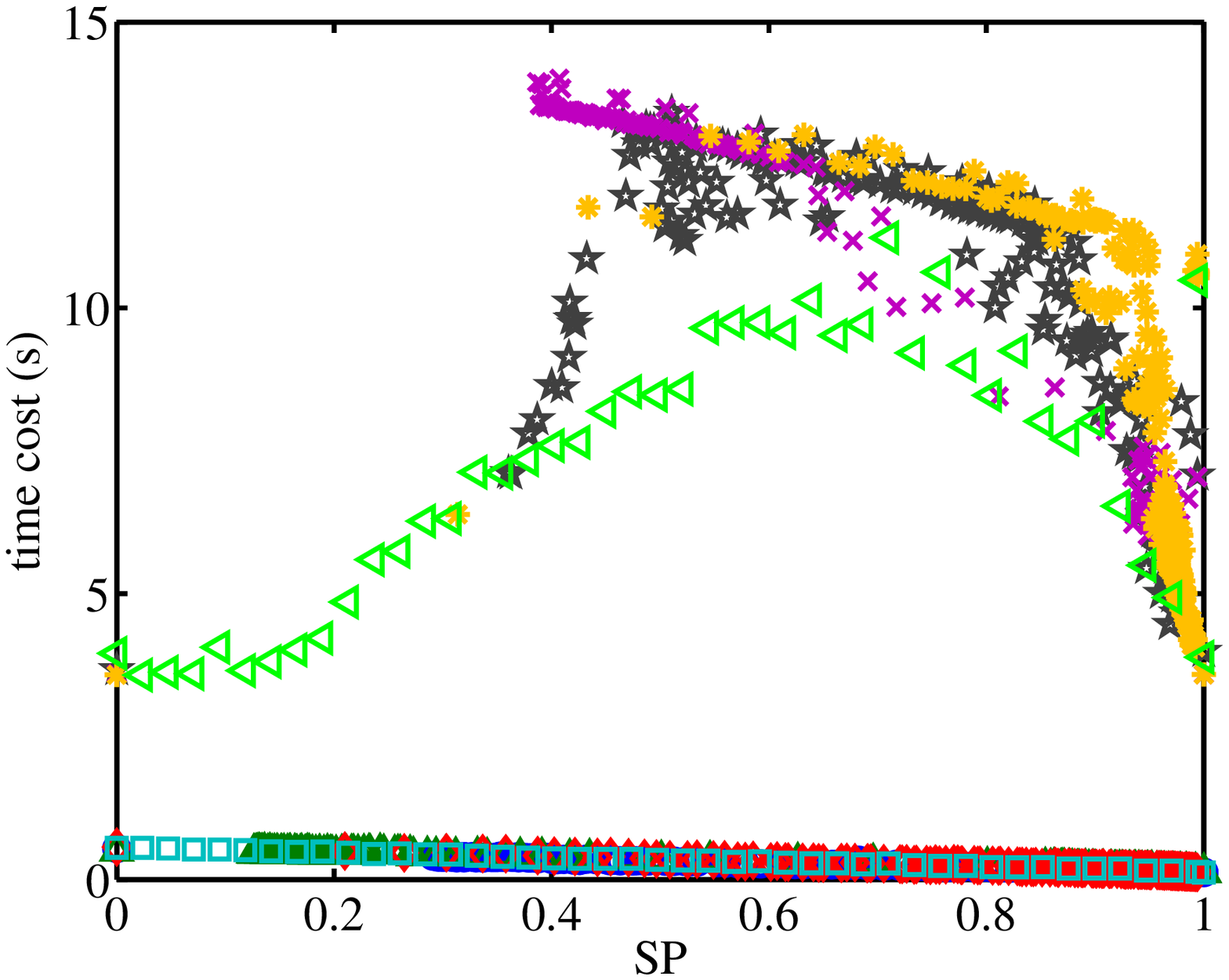}}
} \caption{rSVD-GP v.s. rSVD-GPB on image data. Both are initialized
with PCA. From (b), we see the block version gets much worse
orthogonality than the deflation version. The other criteria are
comparable except time cost.}\label{fig:GP-GPB}
\end{figure*}

\begin{figure*}[h]
\centering{
\subfigure[CPEV]{\label{fig:SPCArt-GP:va}\includegraphics[width=4cm]{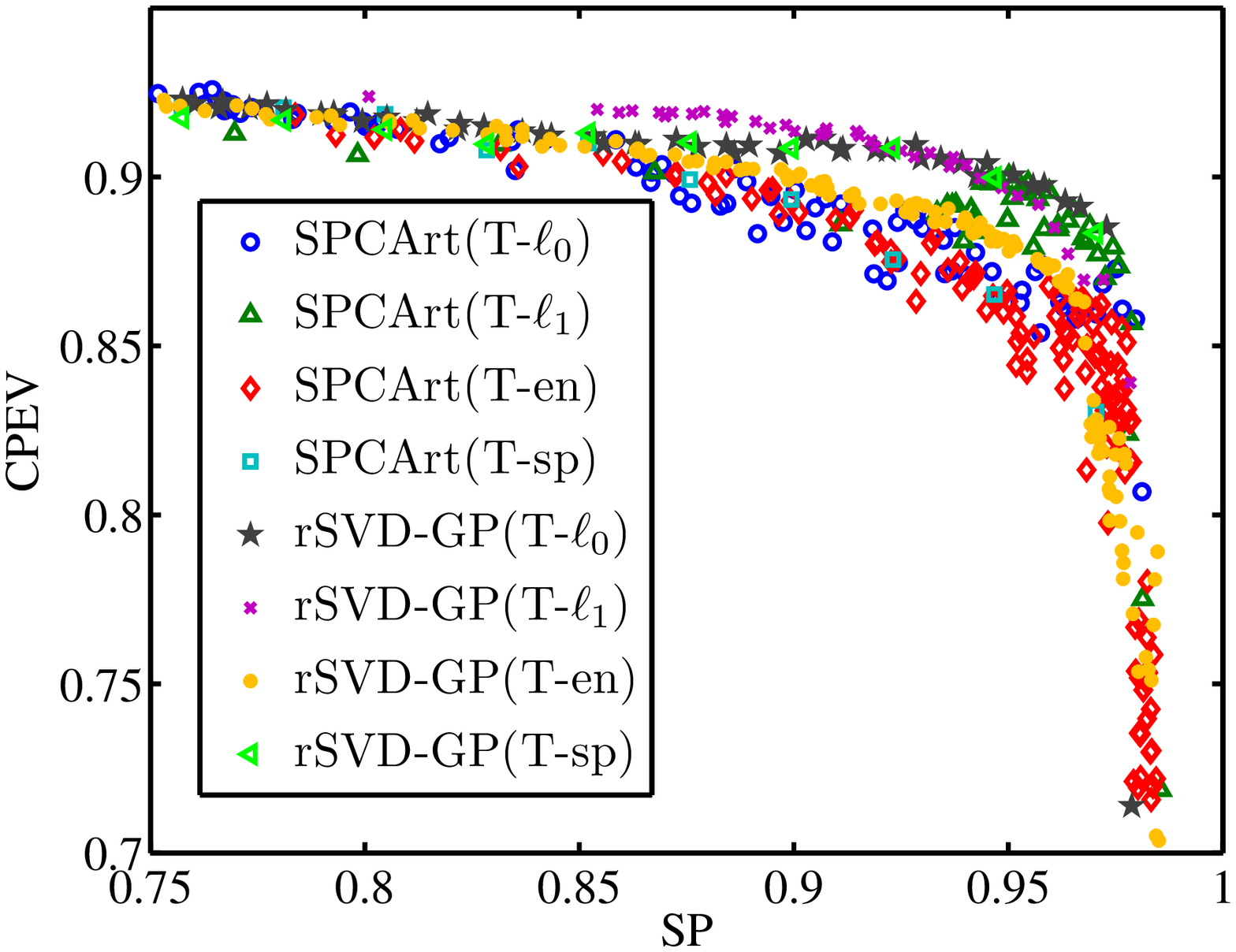}}
\subfigure[NOR]{\label{fig:SPCArt-GP:or}\includegraphics[width=4cm]{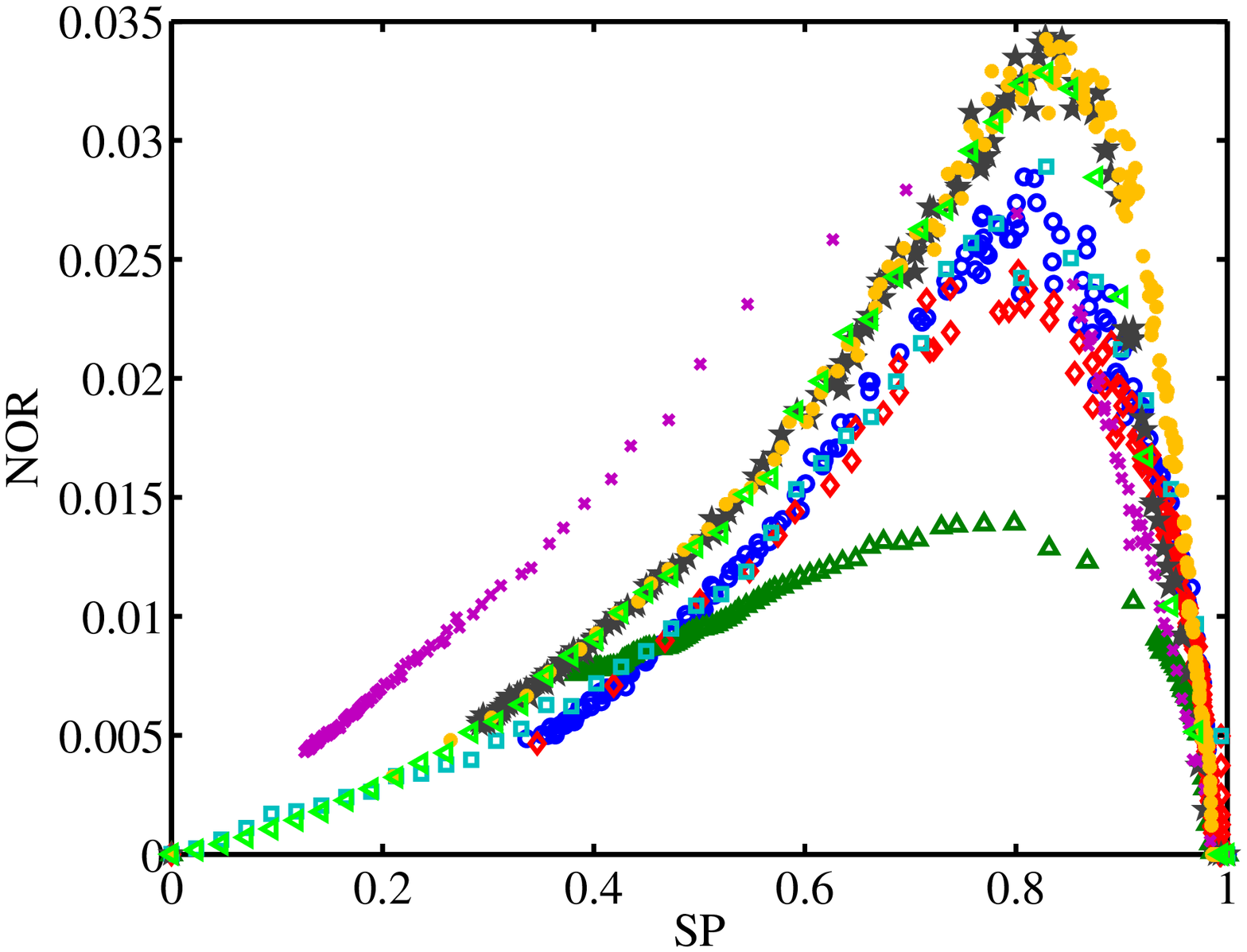}}
\subfigure[STD]{\label{fig:SPCArt-GP:std}\includegraphics[width=4cm]{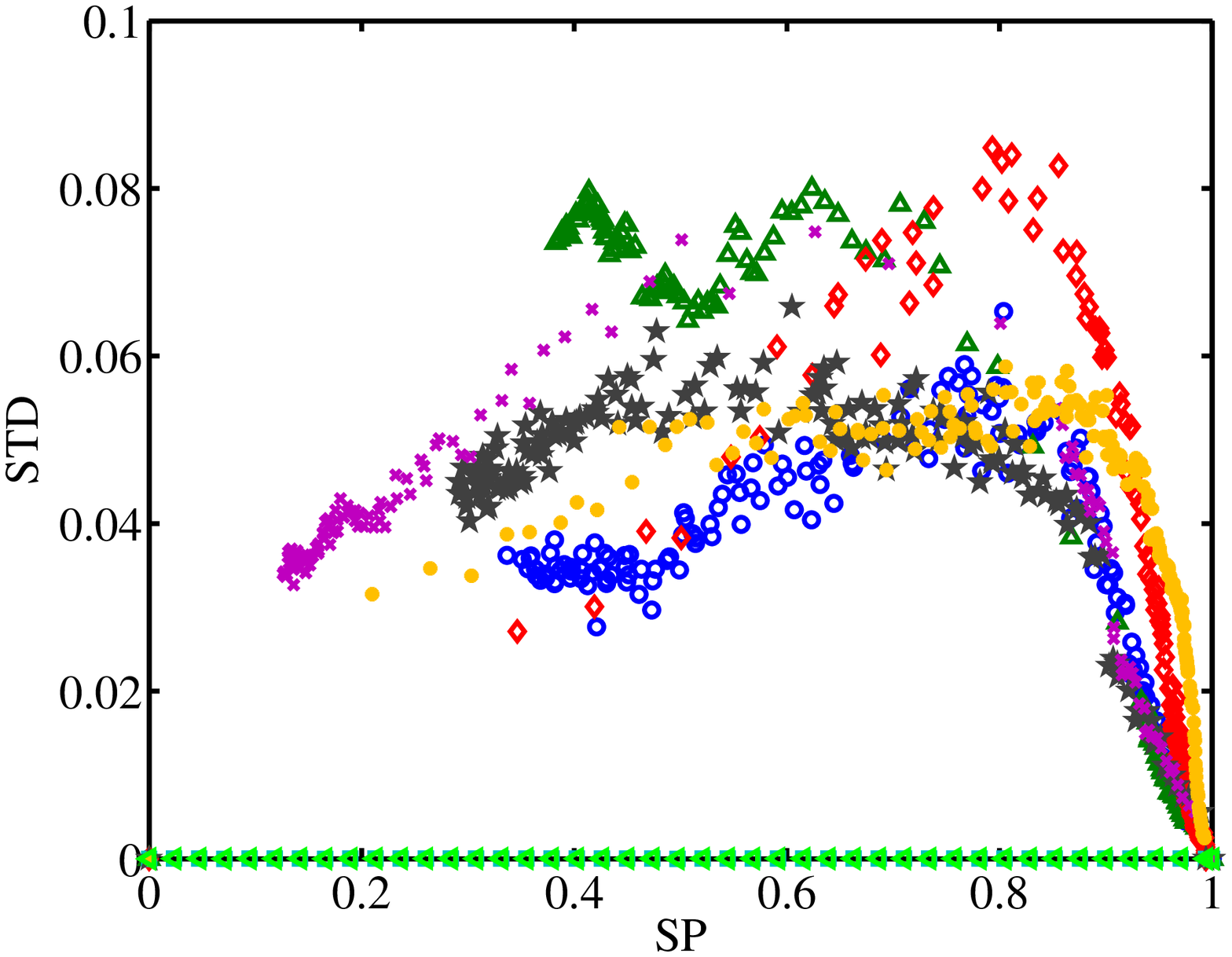}}
\subfigure[Time
cost]{\label{fig:SPCArt-GP:tm}\includegraphics[width=4cm]{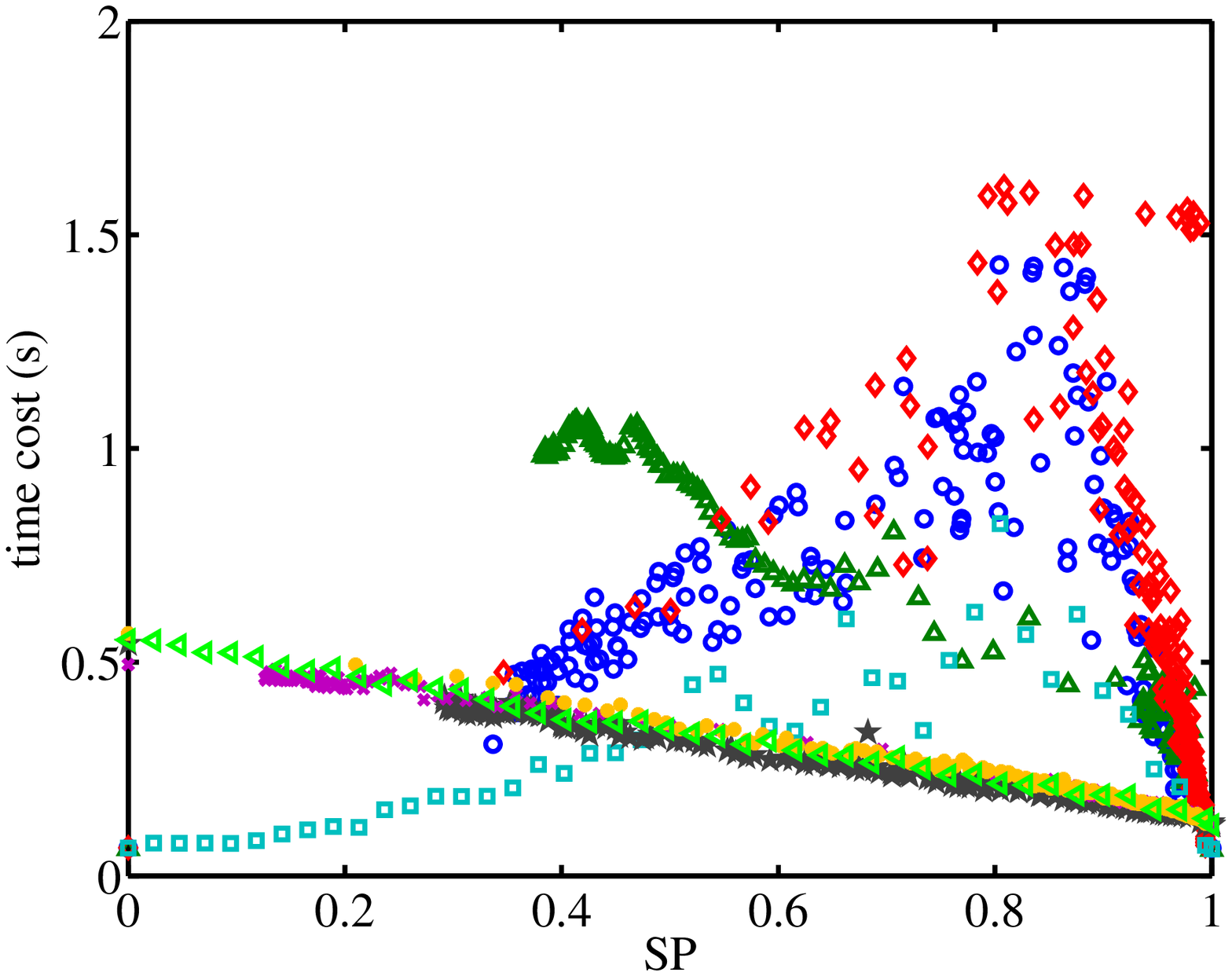}}
} \caption{SPCArt v.s. rSVD-GP on image data. The two methods
{obtain} comparable results on these criteria.}\label{fig:SPCArt-GP}
\end{figure*}

\begin{figure*}[h]
\centering{
\subfigure[CPEV]{\label{fig:SPCArt-other:va}\includegraphics[width=4cm]{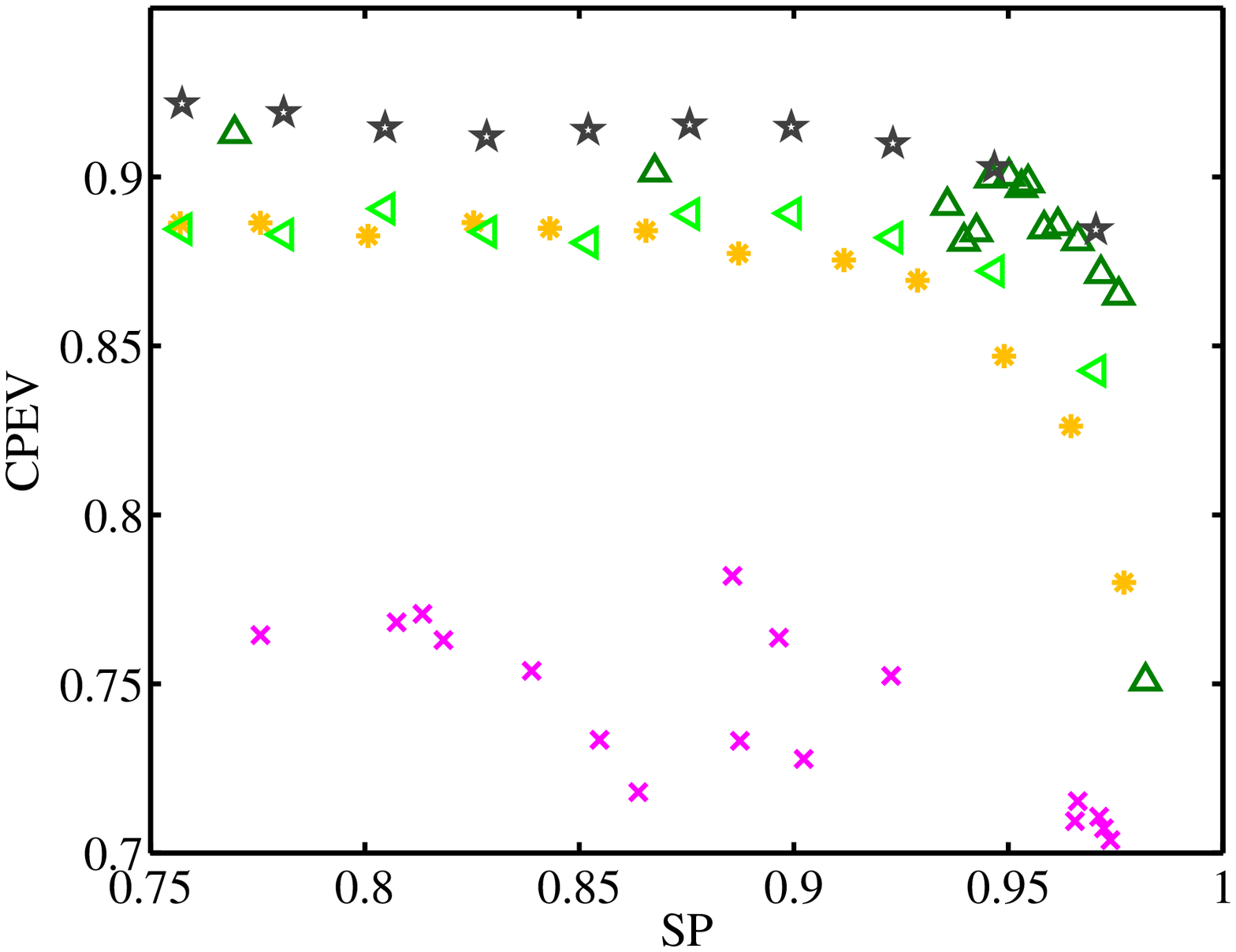}}
\subfigure[NOR]{\label{fig:SPCArt-other:or}\includegraphics[width=4cm]{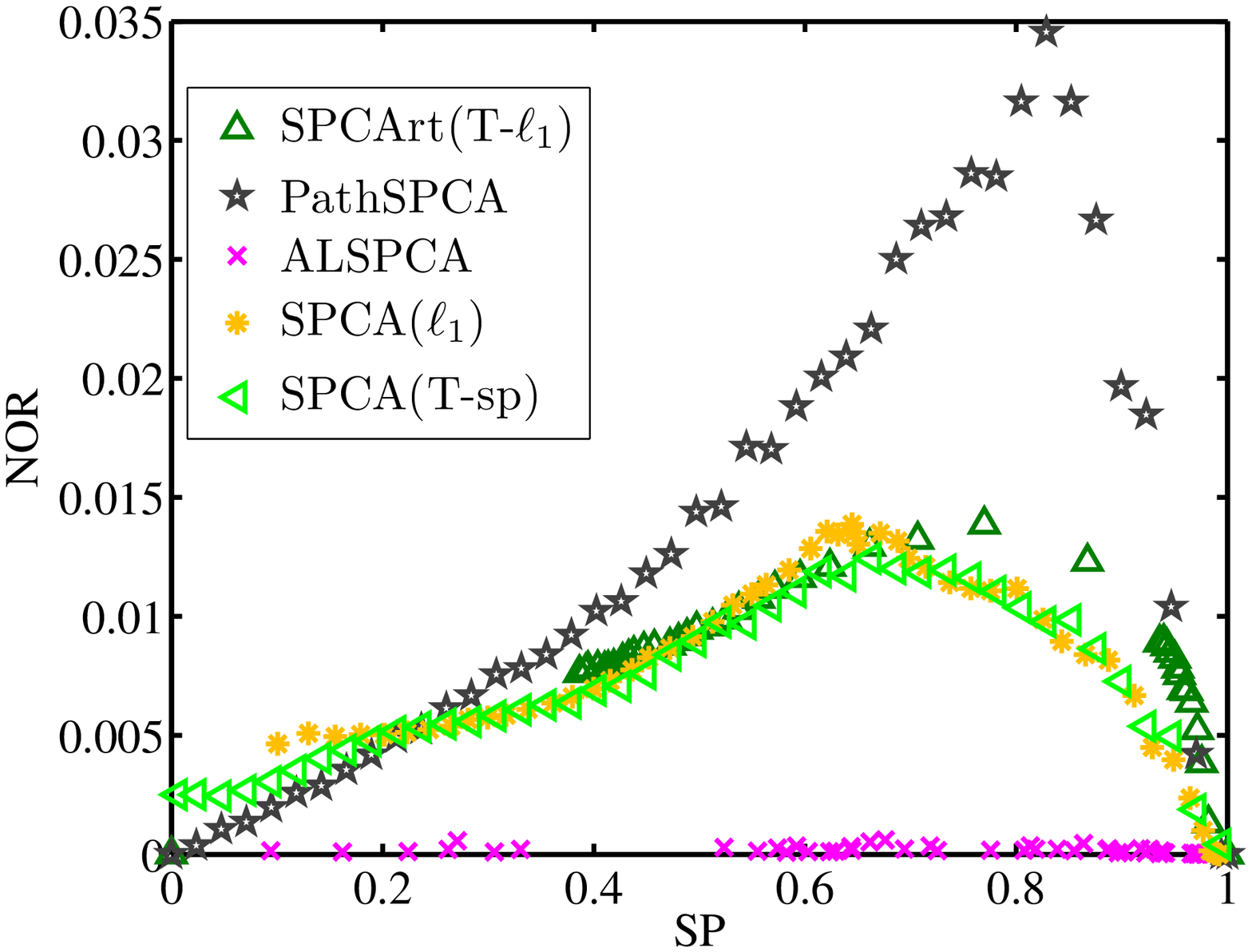}}
\subfigure[STD]{\label{fig:SPCArt-other:std}\includegraphics[width=4cm]{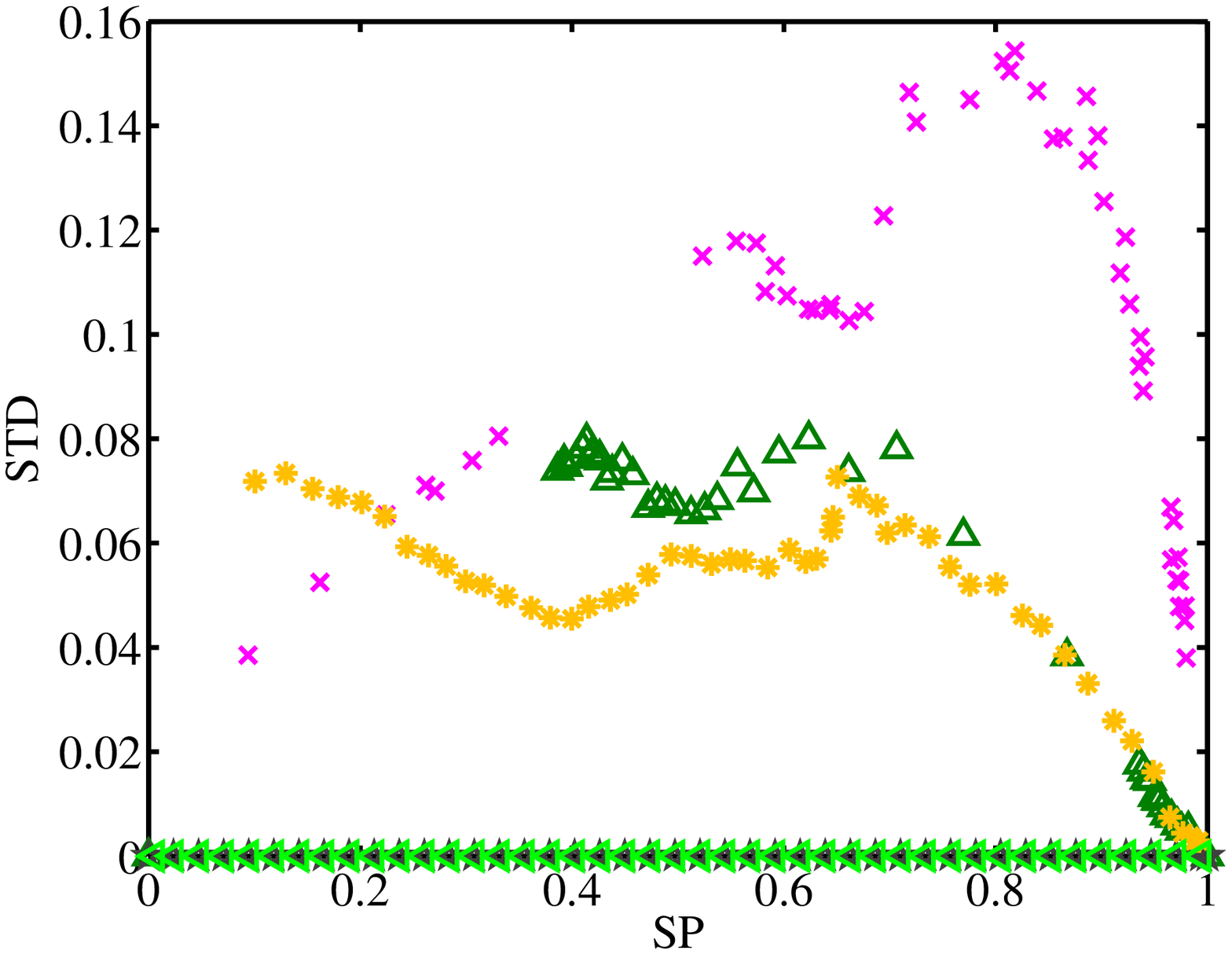}}
\subfigure[Time
cost]{\label{fig:SPCArt-other:tm}\includegraphics[width=4cm]{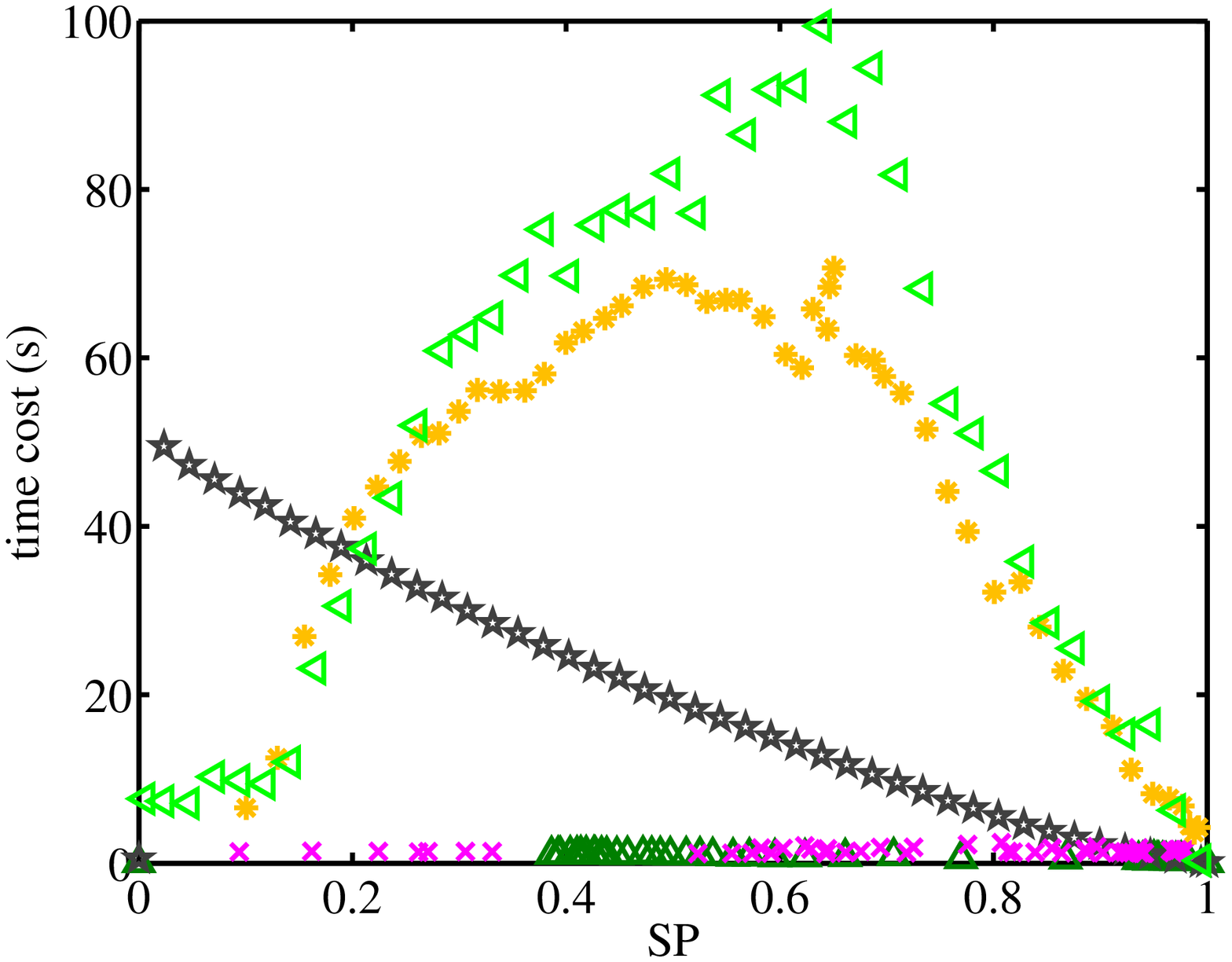}}
} \caption{SPCArt v.s. SPCA, PathSPCA, ALSPCA on image data. To make
the figures less messy, T-$\ell_1$ is taken as representative for
SPCArt. SPCArt performs best overall, while PathSPCA performs best
at CPEV. ALSPCA and SPCA are unstable. PathSPCA and SPCA are time
consuming.}\label{fig:SPCArt-other}
\end{figure*}

\subsubsection{Performance Comparisons between Algorithms}
We fix $r=70$ and run the algorithms over a range of parameter
$\lambda$ to produce a series of results, then the algorithms are
compared based on the same sparsity. We first verify the improvement
of rSVD-GP over GPower(B) on the balance of sparsity, and take
rSVD-GP(B) as example to show that the block group produces worse
orthogonality than the deflation group. Then we compare SPCArt with
the other algorithms.

(1) rSVD-GP v.s. GPower(B), see Figure~\ref{fig:GP-GPower}. For
GPower(B), the uniform parameter setting leads to unbalanced
sparsity. In fact, the worst case is usually achieved by the leading
loadings. rSVD-GP significantly improves over GPower(B) on this
criterion as well as the others.

(2) rSVD-GP v.s. rSVD-GPB, see Figure~\ref{fig:GP-GPB}. The block
version always gets worse orthogonality. This is because there is no
mechanism in it to ensure orthogonality.

(3) SPCArt v.s. rSVD-GP, see Figure~\ref{fig:SPCArt-GP}. The two
methods obtain comparable results on these criteria.

(4) SPCArt v.s. SPCA, PathSPCA, and ALSPCA, see
Figure~\ref{fig:SPCArt-other}. SPCArt performs best overall.
Generally, PathSPCA performs best at CPEV, but its time cost
increases with cardinality. ALSPCA is unstable and sensitive to
parameter, so is SPCA. Besides, SPCA is time consuming.

\begin{figure*}[h]
\centering{
\subfigure[SP]{\label{fig:varyr-other:sp}\includegraphics[width=4cm]{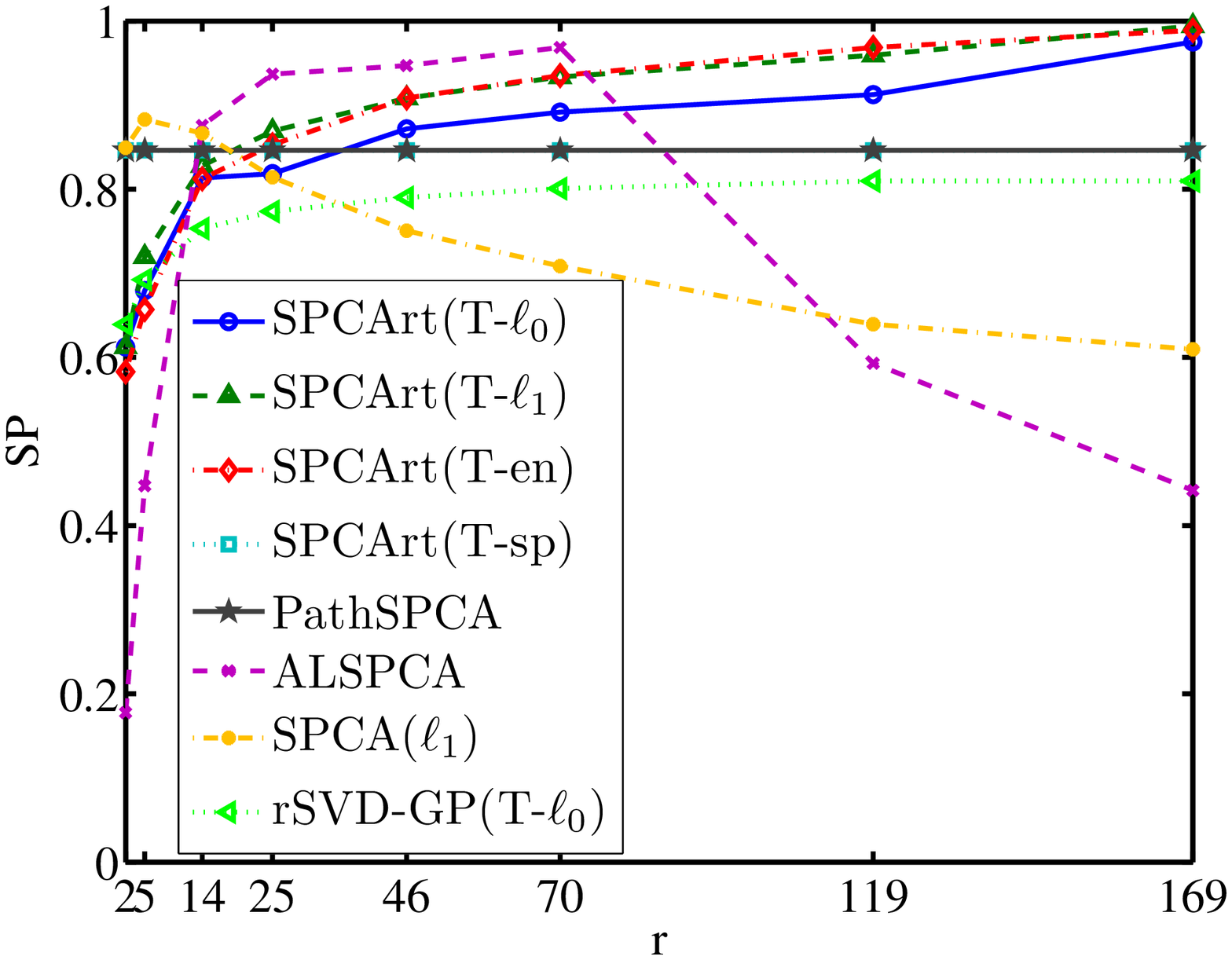}}
\subfigure[CPEV]{\label{fig:varyr-other:va}\includegraphics[width=4cm]{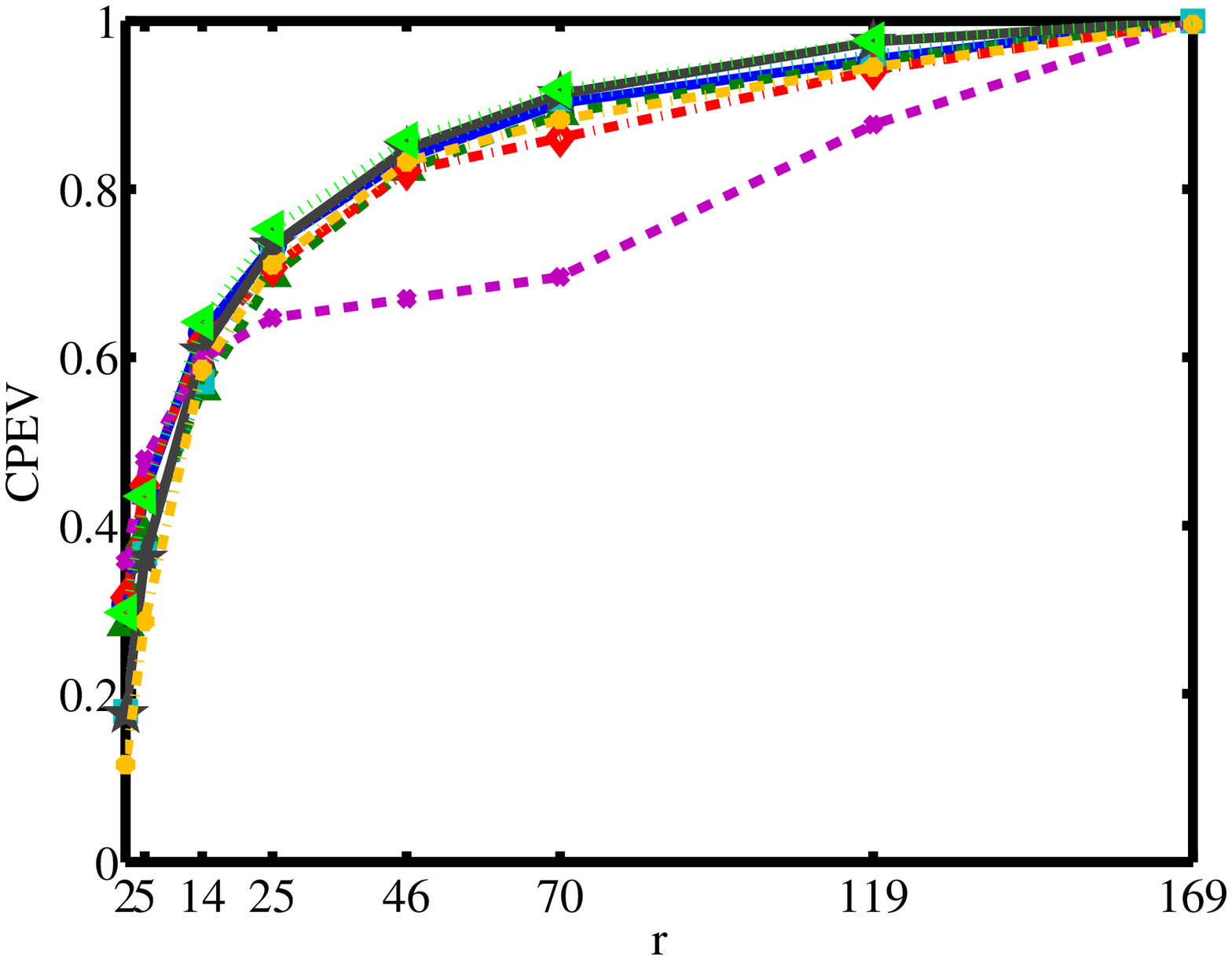}}
\subfigure[NOR]{\label{fig:varyr-other:or}\includegraphics[width=4cm]{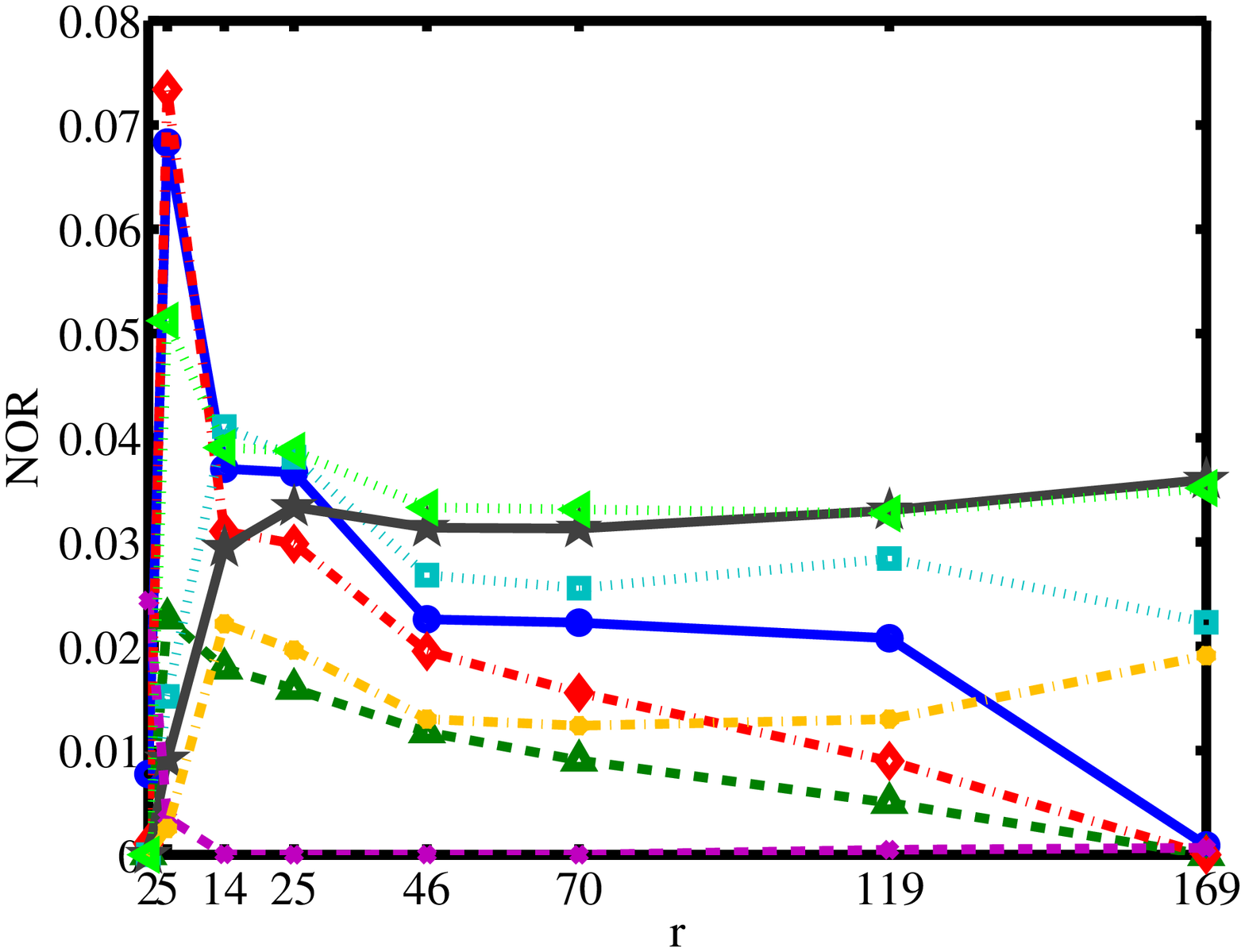}}\\
\subfigure[Worst
sparsity]{\label{fig:varyr-other:std}\includegraphics[width=4cm]{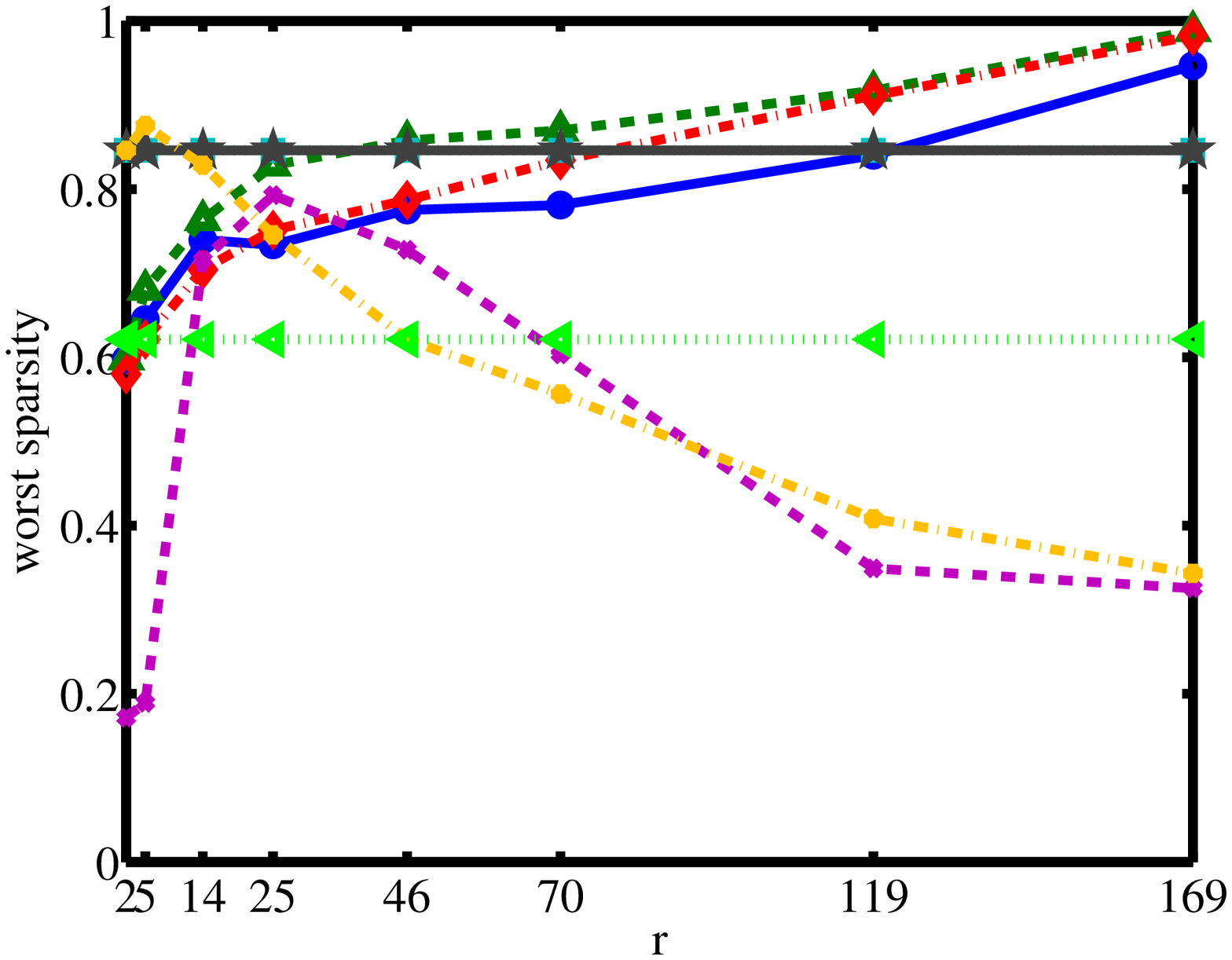}}
\subfigure[Time
cost]{\label{fig:varyr-tm:tm1}\includegraphics[width=4cm]{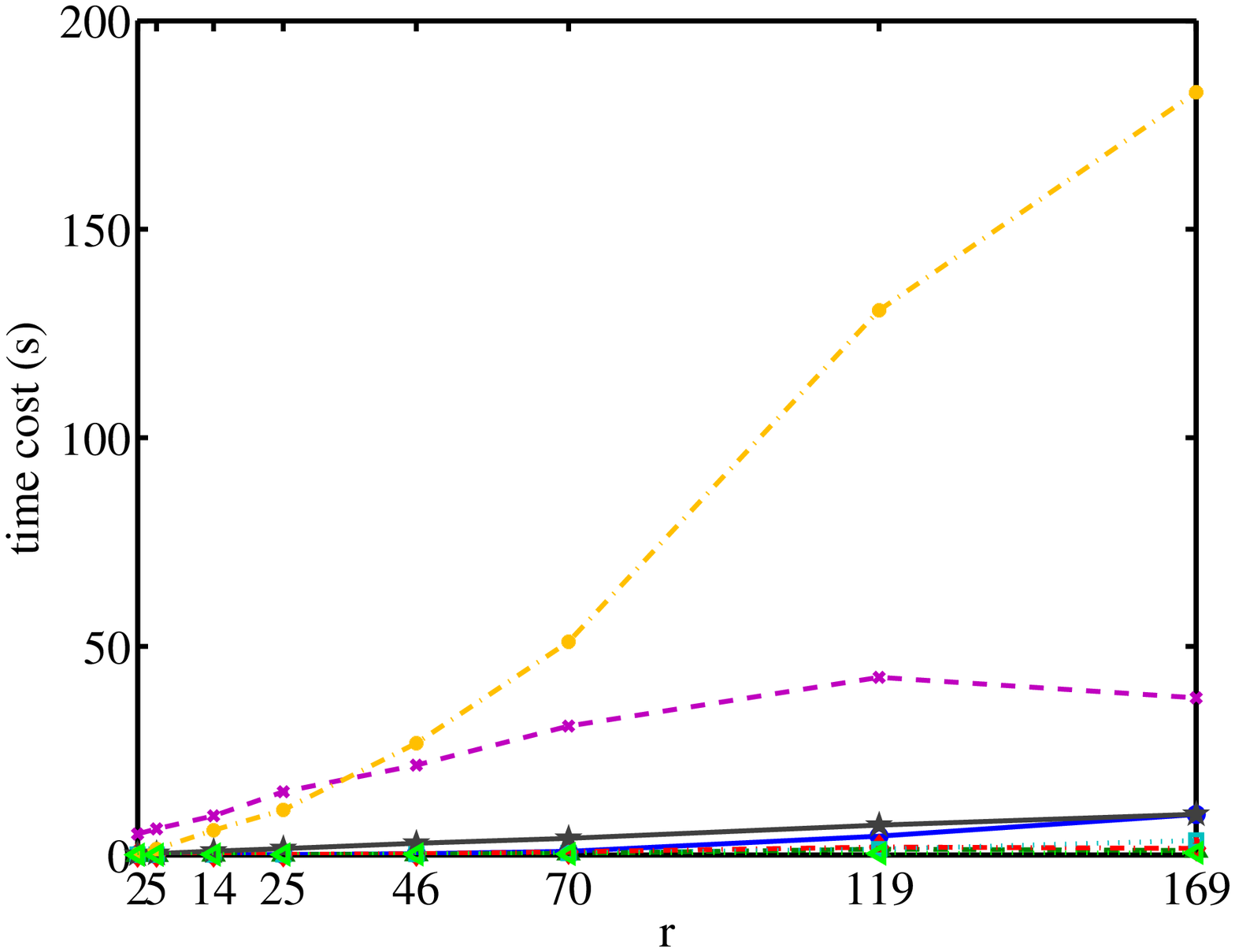}}
\subfigure[Time cost per
iteration]{\label{fig:varyr-tm:tm2}\includegraphics[width=4cm]{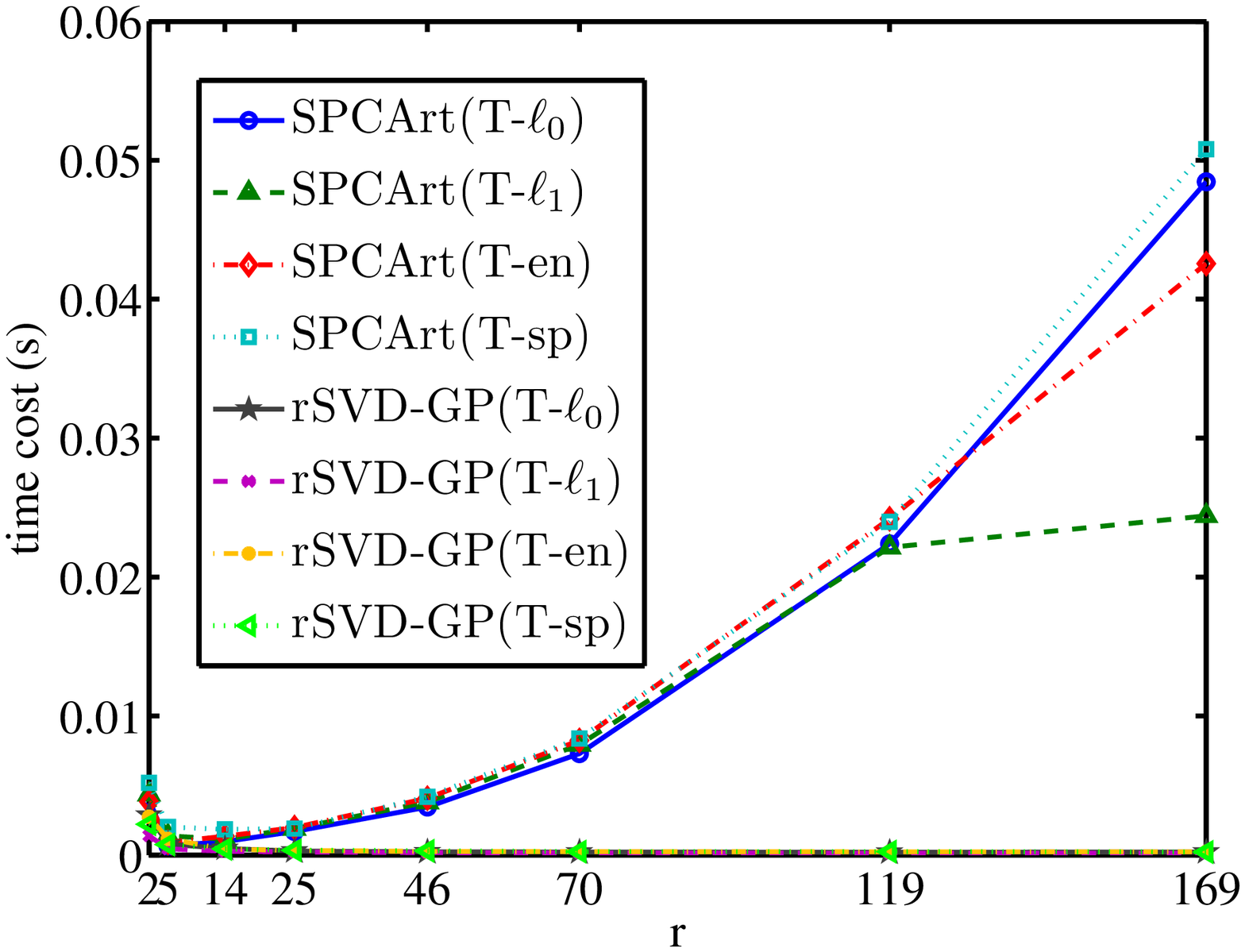}}
} \caption{Evolution of solution as $r$ increases on image data,
SPCArt v.s. PathSPCA, ALSPCA, SPCA, and rSVD-GP(T-$\ell_0$). In (f),
only SPCArt and rSVD-GP are shown. SPCArt is insensitive to
parameter setting. Compared with the deflation algorithms (PathSPCA,
rSVD-GP), the loadings of SPCArt are adaptive with $r$, whose
properties gradually improve. When $r$ becomes the full dimension,
T-$\ell_1$ perfectly recovers the natural basis which is globally
optimal, as can be seen from SP, worst sparsity, and NOR. Both the
two sparsity criteria reach $(p-1)/p=0.994$ and NOR touches bottom
0. T-en achieves similar results.}\label{fig:varyr-other}
\end{figure*}

\begin{figure}[h]
\center{
\includegraphics[width=8.5cm]{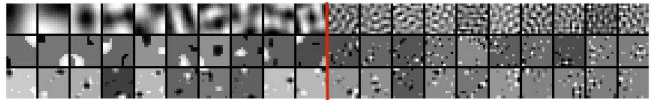}
}\caption{Images of the first 10 and the last 10 loadings among the
total 70 loadings on image data. 1st line: PCA; 2nd line:
rSVD-GP(T-sp); 3rd line: SPCArt(T-sp).
$\lambda=\lfloor0.85p\rfloor$. rSVD-GP is greedy, and the results of
it are more confined to those of PCA, while SPCArt is more
flexible.}\label{fig:Ximage}
\end{figure}

\begin{figure*}[h]
\centering{
\subfigure[CPEV]{\label{fig:gene:va}\includegraphics[width=4cm]{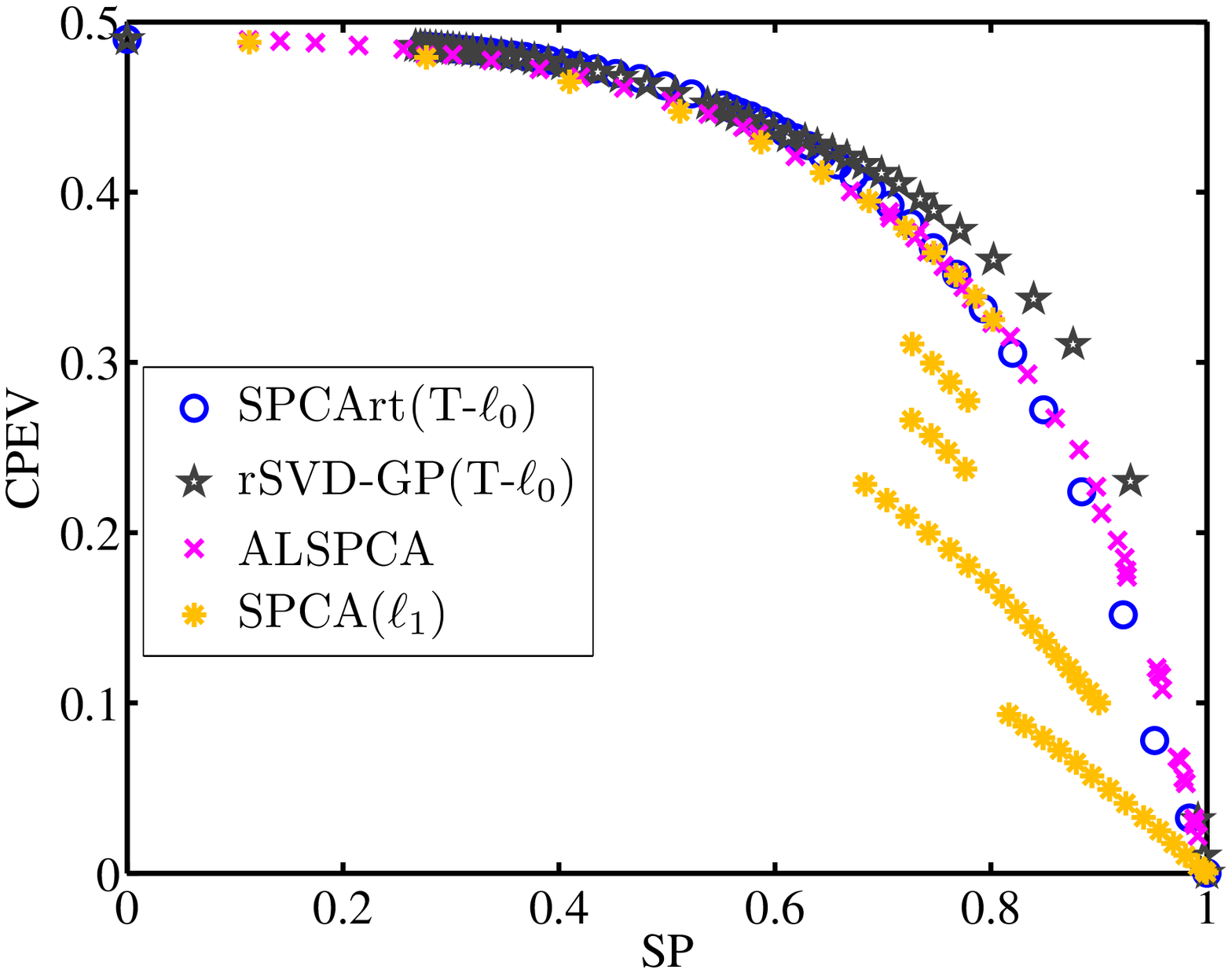}}
\subfigure[NOR]{\label{fig:gene:or}\includegraphics[width=4cm]{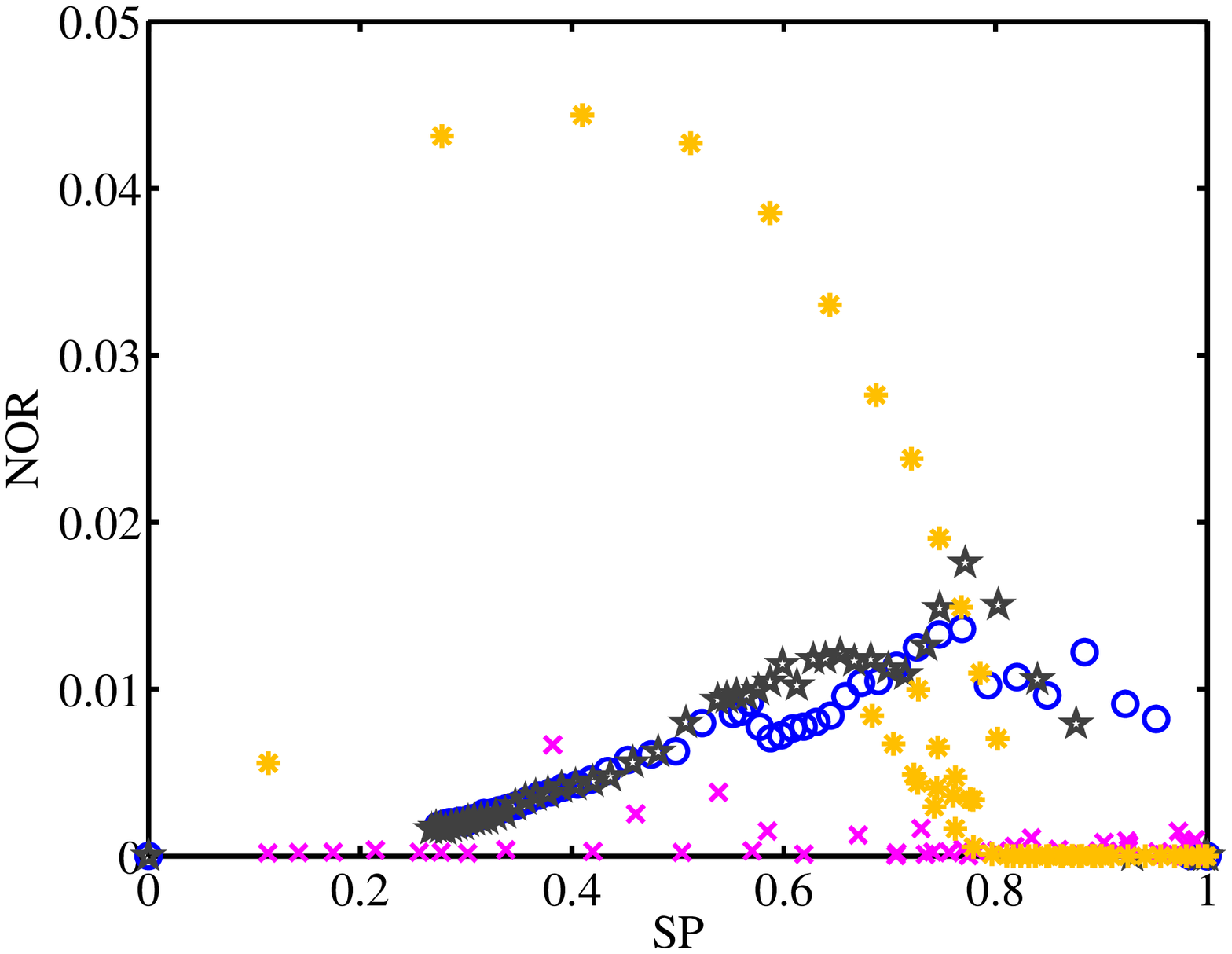}}
\subfigure[STD]{\label{fig:gene:std}\includegraphics[width=4cm]{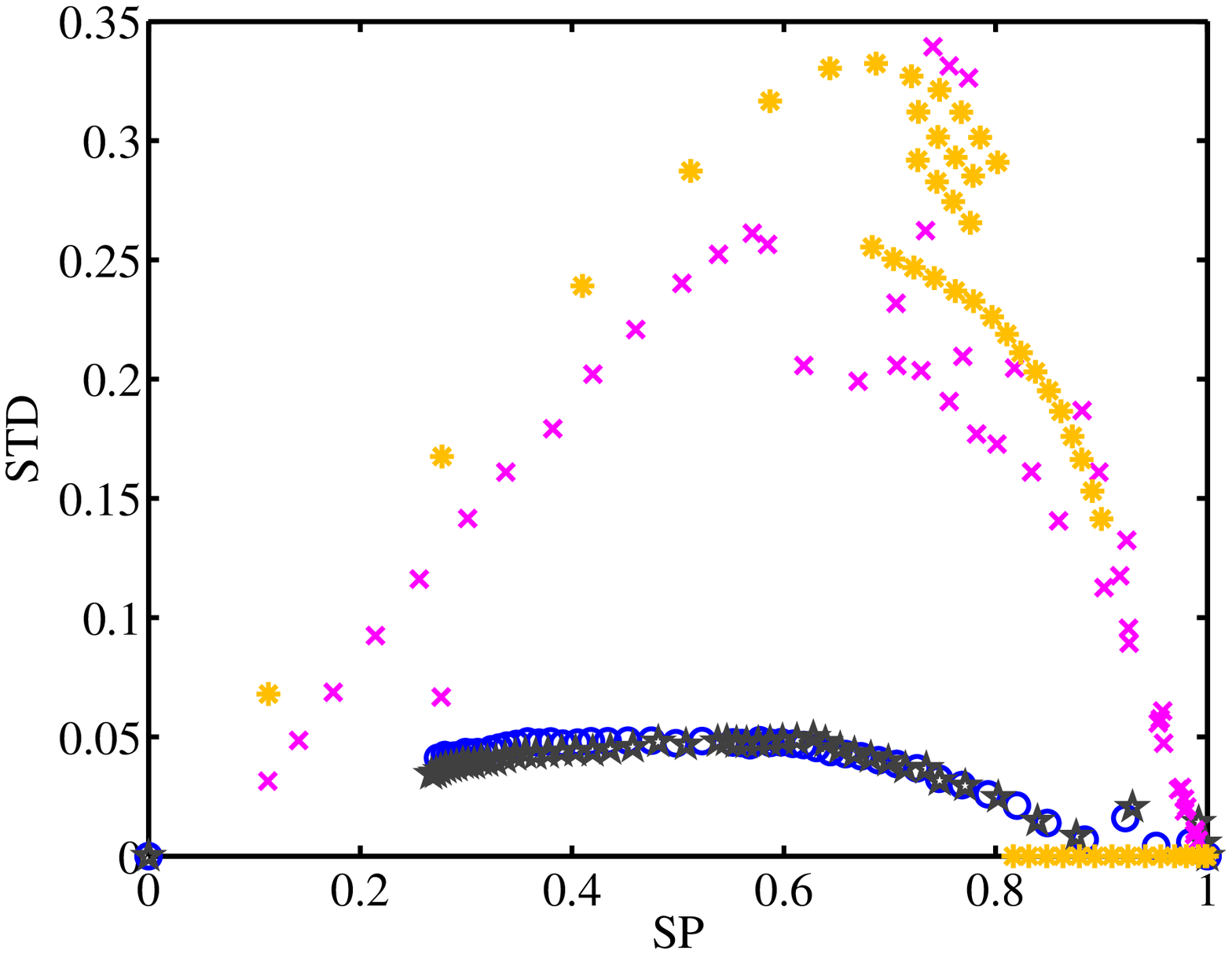}}
\subfigure[Time
cost]{\label{fig:gene:tm}\includegraphics[width=4cm]{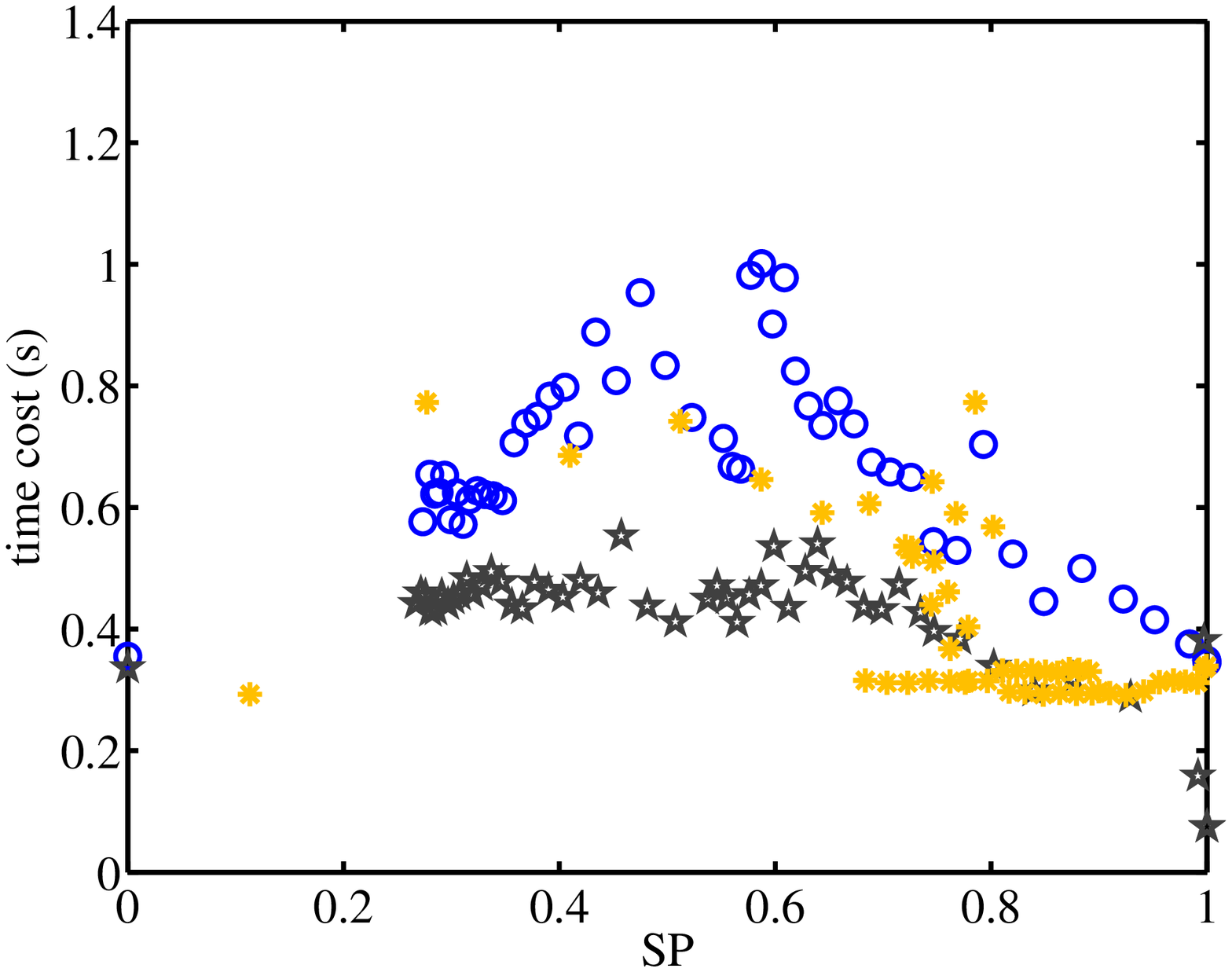}}
} \caption{SPCArt(T-$\ell_0$) vs. rSVD-GP(T-$\ell_0$), ALSPCA, and
SPCA on gene data, $r=6$. To be less messy, the other truncation
types are not shown. ALSPCA is much more costly so it is not shown
in (d). SPCArt(T-$\ell_0$) and rSVD-GP(T-$\ell_0$) perform best, and
both finish within 1 second in such high dimensional
data.}\label{fig:gene}
\end{figure*}

\subsubsection{Evolution of Solution as $r$ Increases}

Finally, we evaluate how the solution evolves as $r$ increases. $r$
is sampled so that CPEV(V) = [0.3 0.5 0.7 0.8 0.9 0.95 0.99 1]. For
simplicity, the $\lambda$'s are kept fixed, they are set as follows.
T-$\ell_0$: $1/\sqrt p$; T-sp: $\lfloor0.85p\rfloor$; T-en: $0.15$;
T-$\ell_1$: SPCArt $1/\sqrt p$, SPCA $4$, ALSPCA $0.7$. The results
are plotted in Figure~\ref{fig:varyr-other}. We can observe that:

(1) Using the same threshold, T-$\ell_1$ is always more sparse and
orthogonal than T-$\ell_0$, while explaining less variance.

(2) SPCArt is insensitive to parameter. A constant setting produces
satisfactory results across $r$'s. But it is not the case for
rSVD-GP.

(3) In contrast to the deflation algorithms (PathSPCA, rSVD-GP),
SPCArt is a block algorithm. Its solution evolves as $r$. The
sparsity, explained variance, orthogonality, and balance of sparsity
improve as $r$ increases, and it has the potential to get optimal
solution. This is evident for T-en and T-$\ell_1$ when $r$ becomes
the full dimension 169. T-$\ell_1$ perfectly recovers the natural
basis which is globally optimal; and T-en obtains similar results.
Visualized images of the loadings of the deflation and block
algorithm are shown in Figure~\ref{fig:Ximage}. {Due to the greedy
nature, the results obtained by deflation algorithm} are more
confined to those of PCA; and the first 10 loadings differ
significantly from the last 10 loadings.

\subsection{Gene Data ($n\ll p$)}
We now try the algorithms on the Leukemia dataset
\cite{golub1999molecular}, which contains 7129 genes and 72 samples,
i.e. $p\gg n$ data. This is a classical application that motivates
the development of sparse PCA. Because from the thousands of genes,
a sparse basis can help us to locate a few of them that determines
the distribution of data. The results are shown in
Figure~\ref{fig:gene}. For this type of data, SPCA is run on the
$p\gg n$ mode \cite{zou2006sparse} for efficiency. PathSPCA is very
slow except when SP$\geq 97\%$, so it is not involved in the
comparison. SPCArt(T-$\ell_0$) and rSVD-GP(T-$\ell_0$) perform best
(the later is slightly better).

\subsection{Random Data ($n>p$)}
Finally, we test the computational efficiency on a set of random
data with increasing dimensions $p$ = [100 400 700 1000 1300].
Following
\cite{aspremont2007direct,lu2009augmented,journee2010generalized},
zero-mean, unit-variance Gaussian data is used for the test. To make
how the computational cost depends on $p$ clear, we let $n=p+1$. For
fair comparison, only T-sp with $\lambda=\lfloor0.85p\rfloor$ are
tested. $r$ is set to 20. The results are shown in
Figure~\ref{fig:varyp}. rSVD-GP and PathSPCA increase nonlinearly
against $p$, while SPCArt grows much slowly. Remember in
Figure~\ref{fig:SPCArt-other:tm}, we already showed that the time
complexity of PathSPCA increases nonlinearly against the
cardinality, and from Figure~\ref{fig:varyr-tm:tm2}, we saw SPCArt
increases nonlinearly against $r$. All these are consistent with
Table~\ref{tab:timeO}. When dealing with high dimensional data and
pursuing a few loadings, SPCArt is advantageous.

\begin{figure}[thpb]
\centering{ \subfigure[Time
cost]{\label{fig:varyp:tm}\includegraphics[width=4cm]{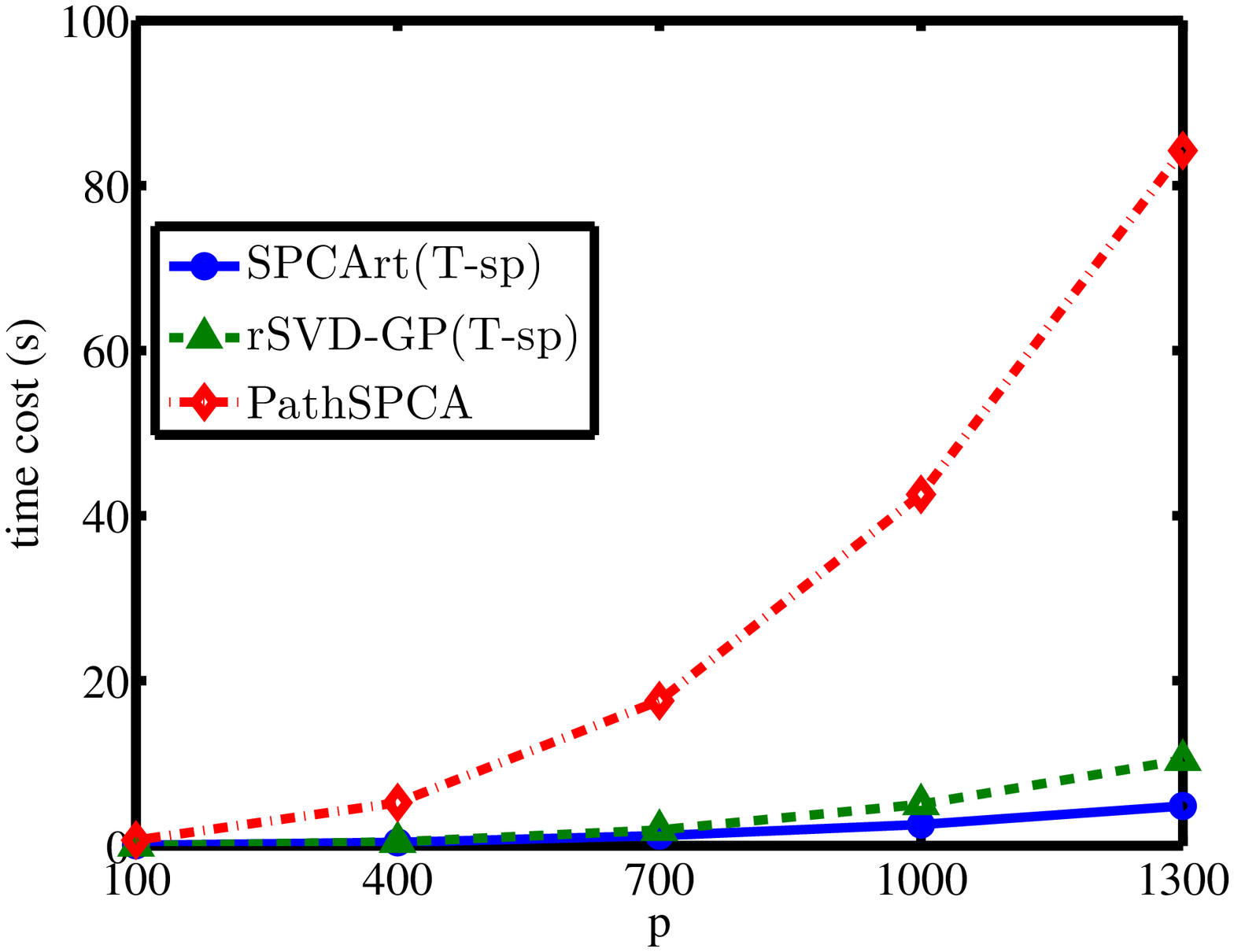}}
\subfigure[Time cost per
iteration]{\label{fig:varyp:tm-it}\includegraphics[width=4cm]{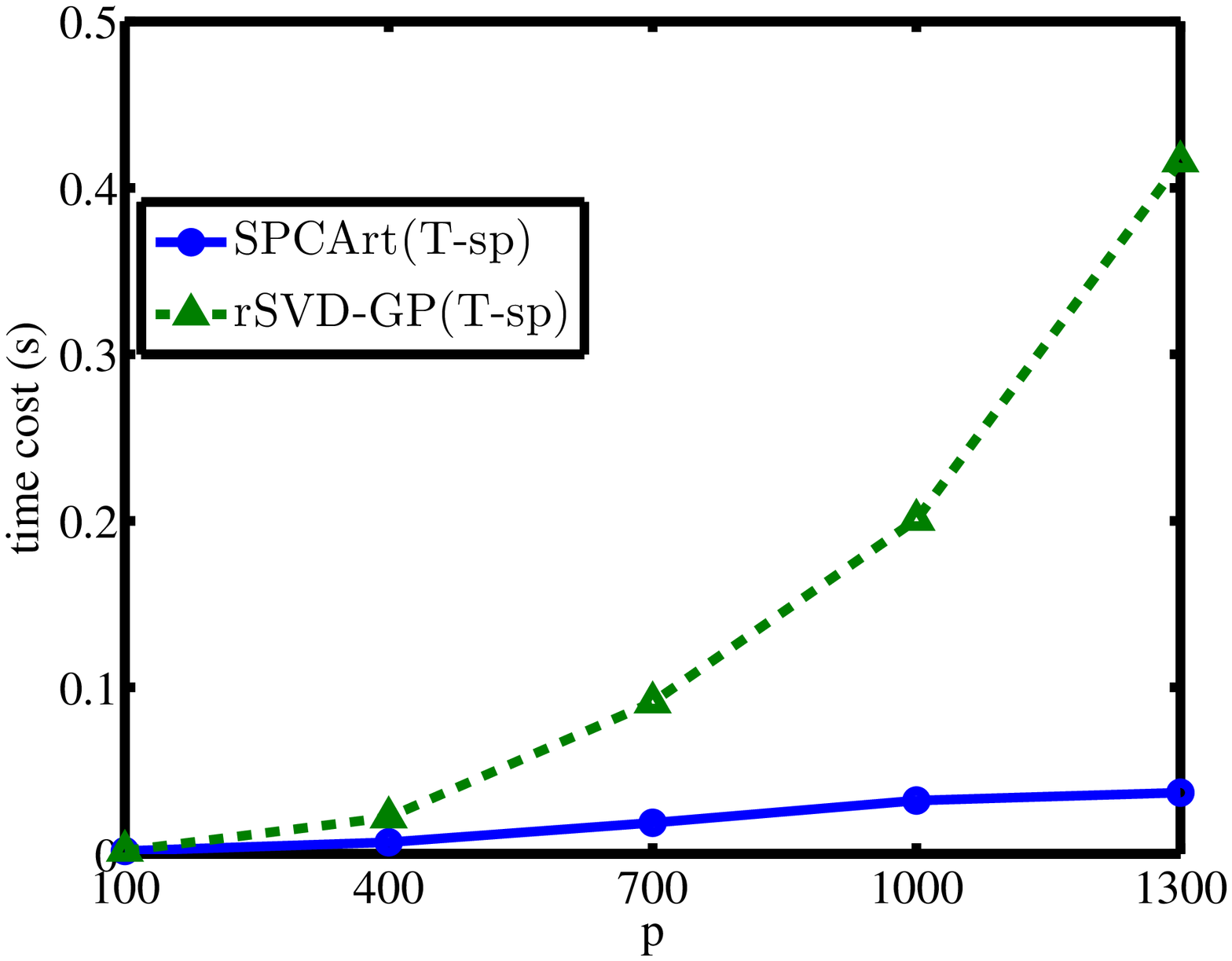}}
} \caption{Speed test on random data with increasing dimension $p$.
SPCArt grows much slowly as $p$.}\label{fig:varyp}
\end{figure}

\section{Conclusion} \label{sec:conclusion}
According to the experiments, SPCArt significantly improves simple
thresholding. rSVD-GP(B) improves GPower(B). rSVD-GP obtains
loadings more orthogonal than rSVD-GPB. SPCArt, rSVD-GP, and
PathSPCA generally perform well. PathSPCA consistently explains most
variance, but it is the most time-consuming among the three. rSVD-GP
and SPCArt perform similarly on sparsity, explained variance,
orthogonality, and balance of sparsity. However rSVD-GP is more
sensitive to parameter setting (except rSVD-GP(T-sp), i.e. TPower),
and it is a greedy deflation algorithm. SPCArt belongs to the block
group, its solution improves with the target dimension, and it has
the potential to obtain globally optimal solution.

When the sample size is larger than the dimension, the time cost of
PathSPCA and rSVD-GP go nonlinearly with the dimension, while SPCArt
increases much slowly. They can deal with high dimensional data
under different situations, SPCArt: the number of loadings is small;
rSVD-GP: the sample size is small; PathSPCA: the target cardinality
is small.

The four truncation types of SPCArt work well in different aspects:
T-$\ell_0$ hard thresholding performs well overall; T-$\ell_1$ soft
thresholding gets best sparsity and orthogonality; T-sp hard
sparsity constraint directly controls sparsity and has zero sparsity
variance; T-en truncation by energy guarantees explained variance,
and the performance bound is tight.

There are two open questions unresolved. (1) Under what conditions
can SPCArt, with each truncation type, recover the underlying sparse
basis? Efforts have been made recently on this problem
\cite{amini2009high, paul2012augmented, yuan2013truncated,
ma2013sparse}. (2) Is there any explicit objective formulation for
T-en?

\appendices
\section{Proof of the solution of T-sp}\label{sec:app T-sp}
When $R$ is fixed, define $Z=VR^T$, (\ref{equ:SPCArtsp}) becomes $r$
independent subproblems:
\begin{equation}\label{equ:SPCArtsp-X}
\min_{X_i}\;\|Z_i-X_i\|^2_F,\;s.t.\,\|X_i\|_0\leq
p-\lambda,\;\|X_i\|_2=1.
\end{equation}

\begin{proposition}
$X_i^*=P_\lambda(Z_i)/\|P_\lambda(Z_i)\|_2$ is the solution of
(\ref{equ:SPCArtsp-X}).
\end{proposition}

\begin{proof}
The problem is equivalent to $\max_{X_i} Z_i^TX_i$, s.t.
$\|X_i\|_0\leq p-\lambda, \|X_i\|_2=1$. We first prove that the
non-zeros of $X_i^*$ are the normalized entries of $Z_i$ in the same
support as $X_i^*$, then prove $\|X_i^*\|_0=p-\lambda$ and the
support corresponds to the largest entries of $Z_i$. Assume the
support of $X_i^*$ is $\mathcal{S}$. Divide $Z_i$ into two parts
$Z_i=\tilde{Z}_i+\bar{Z}_i$, where $\tilde{Z}_i$ has the same
support as $X_i^*$, and $\bar{Z}_i$ has the remaining support. The
problem is reduced to $\max_{X_i} \tilde{Z}_i^TX_i$, s.t.
$\text{support}(X_i)=\mathcal{S}$, $\|X_i\|_2=1$. The solution is
$X_i^*=\tilde{Z}_i/\|\tilde{Z}_i\|_2$. Next, since
$Z_i^T\tilde{Z}_i/\|\tilde{Z}_i\|_2=\|\tilde{Z}_i\|_2$, to achieve a
minima, $\|\tilde{Z}_i\|_2$ should be as large as possible. That is
the largest $p-\lambda$ entries of $Z_i$.
\end{proof}

\section{Proofs of Performance Bounds of SPCArt}\label{sec:app
bounds} Many of the results can be proven by studying the special
case $z=(1,0,\dots,0)^T$ and $z=(1/\sqrt p,\dots,1/\sqrt p)^T$. We
mainly focus on the less straightforward ones.

\subsection{Sparsity and Deviation}\label{sec:app sp and dv}

\begin{l0}
For T-$\ell_0$, the sparsity bounds are
\begin{equation*}
\begin{cases}
  0\leq s(x) \leq 1-\frac{1}{p} & \text{, $\lambda<\frac{1}{\sqrt{p}}$,} \\
  1-\frac{1}{p\lambda^2} < s(x) \leq 1 & \text{, $\lambda\geq\frac{1}{\sqrt{p}}$.}
\end{cases}
\end{equation*}
Deviation $\sin(\theta(x,z))=\|\bar{z}\|_2$, where $\bar{z}$ is the
truncated part: $\bar{z}_i=z_i$ if $x_i=0$, and $\bar{z}_i=0$
otherwise. The absolute bounds are:
\begin{equation*}
0\leq \sin(\theta(x,z)) \leq
\begin{cases}
  \sqrt{p-1}\lambda & \text{, $\lambda<\frac{1}{\sqrt{p}}$,} \\
  1 & \text{, $\lambda\geq\frac{1}{\sqrt{p}}$.}
\end{cases}
\end{equation*}
All the above bounds are achievable.
\end{l0}

\begin{proof}
We only prove $1-\frac{1}{p\lambda^2}\leq s(x)$, if
$\lambda\geq\frac{1}{\sqrt{p}}$. The others are easy to obtain. Let
$\tilde{z}=z-\bar{z}$, i.e. the part above $\lambda$, and let
$k=\|\tilde{z}\|_0$, then $k\lambda^2<\|\tilde{z}\|^2_2\leq 1$. So
$k<1/\lambda^2$. Since $\|x\|_0=\|\tilde{z}\|_0$,
$s(x)=1-\|x\|_0/p>1-1/(p\lambda^2)$.
\end{proof}

\begin{l1}
For T-$\ell_1$, the bounds of $s(x)$ and lower bound of
$\sin(\theta(x,z))$ are the same as T-$\ell_0$. In addition, there
are relative deviation bounds
\begin{equation*}
\|\bar{z}\|_2\leq\sin(\theta(x,z)) <
\sqrt{\|\bar{z}\|^2_2+\lambda^2\|x\|_0 }.
\end{equation*}
\end{l1}

\begin{proof}
Let $\tilde{z}=z-\bar{z}$, $\hat{z}=S_\lambda(z)$ and
$y=\tilde{z}-\hat{z}$. Note that the absolute value of nonzero entry
of $y$ is $\lambda$, and $\|y\|_2=\lambda
\sqrt{\|\tilde{z}\|_0}=\lambda \sqrt{\|x\|_0}$. Then,
\begin{equation}\label{equ:1}
\cos(\theta(x,z))=\cos(\theta(\hat{z},z))=\hat{z}^Tz/\|\hat{z}\|_2.
\end{equation}
Expand $z=\hat{z}+y+\bar{z}$ and note that $\bar{z}$ is orthogonal
to $\tilde{z}$ and $\hat{z}$, since their support do not overlap. We
have,
\begin{equation}\label{equ:2}
\hat{z}^Tz=\|\hat{z}\|^2_2+\hat{z}^Ty.
\end{equation}
By the soft thresholding operation,
\begin{equation}\label{equ:3}
0<\hat{z}^Ty\leq\|\hat{z}\|_2\|y\|_2.
\end{equation}
Combining (\ref{equ:1}), (\ref{equ:2}) and (\ref{equ:3}), we have
$\|\hat{z}\|_2< \cos(\theta(x,z)) \leq \|\hat{z}\|_2+\|y\|_2$. Note
that the upper bound of (\ref{equ:3}) is achieved when $\hat{z}$ and
$y$ are in the same direction, and in this case,
$\|\hat{z}\|_2+\|y\|_2=\|\tilde{z}\|_2$. So $\|\hat{z}\|_2<
\cos(\theta(x,z)) \leq \|\tilde{z}\|_2$. Then
$1-\|\tilde{z}\|^2_2\leq \sin^2(\theta(x,z)) <1-\|\hat{z}\|^2_2$.
The upper bound is approached when $\hat{z}$ becomes orthogonal to
$y$, in this case
$\|\hat{z}\|^2_2+\|y\|^2_2+\|\bar{z}\|^2_2=\|z\|^2_2=1$. Hence,
$1-\|\hat{z}\|^2_2=\|\bar{z}\|^2_2+\|y\|^2_2=\|\bar{z}\|^2_2+\lambda^2
\|x\|_0$. Besides, $1-\|\tilde{z}\|^2_2=\|\bar{z}\|^2_2$. The final
result is $\|\bar{z}\|_2\leq\sin(\theta(x,z))
<\sqrt{\|\bar{z}\|^2_2+\lambda^2\|x\|_0 }$.
\end{proof}

\vspace{2ex} Proposition~\ref{theo:sp} can be proved in a way
similar to T-en.

\begin{en}
For T-en, $0\leq \sin(\theta(x,z))\leq \sqrt{\lambda}$. In addition
\begin{equation*}
\lfloor{\lambda p}\rfloor/p\leq s(x)\leq 1-1/p.
\end{equation*}
If $\lambda<1/p$, there is no sparsity guarantee. When $p$ is
moderately large, $\lfloor{\lambda p}\rfloor/p\approx \lambda$.
\end{en}

\begin{proof}
Sort squared elements of $z$ in ascending order, and assume they are
$\hat{z}^2_1\leq \hat{z}^2_2 \leq \cdots \leq \hat{z}^2_p$ and the
first $k$ of them are truncated. If $z$ is uniform, i.e.
$\hat{z}^2_1=\hat{z}^2_p=1/p$, then the number of truncated entries
is $k_0= \lfloor{\lambda p}\rfloor$. Suppose $\exists z$ achieves
$k<k_0$, then $\sum^{k_0}_{i=1}\hat{z}^2_i$ is greater than that of
uniform case i.e. $\sum^{k_0}_{i=1}\hat{z}^2_i>k_0/p$. By the
ordering, $\hat{z}^2_{k_0}$ is above the mean of the first $k_0$
entries, $\hat{z}^2_{k_0}\geq 1/k_0
\sum^{k_0}_{i=1}\hat{z}^2_i>1/p$. But on the other hand,
$\hat{z}^2_{k_0+1}$ is below the mean of the remaining part,
$\hat{z}^2_{k_0+1}\leq 1/(p-k_0) \sum^{p}_{i=k_0+1}\hat{z}^2_i<
1/(p-k_0) (1-k_0/p)=1/p<\hat{z}^2_{k_0}$, i.e.
$\hat{z}^2_{k_0+1}<\hat{z}^2_{k_0}$ which is a contradiction. Thus,
$\lfloor{\lambda p}\rfloor/p\leq s(x)$.
\end{proof}

\subsection{Explained Variance}\label{sec:app ev}

\begin{evdmin}
Let rank $r$ SVD of $A\in\mathbb{R}^{n\times p}$ be $U\Sigma V^T$,
$\Sigma\in\mathbb{R}^{r\times r}$. Given $X\in\mathbb{R}^{p\times
r}$, assume SVD of $X^TV$ is $WDQ^T$, $D\in\mathbb{R}^{r\times r}$,
$d_{min}=\min_i D_{ii}$. Then
\begin{equation*}
d_{min}^2\cdot EV(V)\leq EV(X),
\end{equation*}
and $EV(V)=\sum_i \Sigma^2_{ii}$.
\end{evdmin}

\begin{proof}
Let SVD of $A^TA$ be $[V,\,V_2]
\begin{bmatrix}
\Lambda&\\&\Lambda_2
\end{bmatrix}
\begin{bmatrix}
V^T\\V^T_2
\end{bmatrix}
$, where $\Lambda=\Sigma^2$ and subscript 2 associates with the
remaining loadings. Then
\begin{equation*}
\begin{aligned}
tr(X^TA^TAX) &=tr(X^T[V,\,V_2]
\begin{bmatrix}
\Lambda&\\&\Lambda_2
\end{bmatrix}
\begin{bmatrix}
V^T\\V^T_2
\end{bmatrix}X)\\
&=tr(X^TV\Lambda V^TX)+ tr(X^TV_2\Lambda_2V^T_2X)\\
&\geq tr(X^TV\Lambda V^TX)\\
&=tr(WDQ^T\Lambda QDW^T)\\
&=tr(Q^T\Lambda QD^2)\\
&\geq tr(Q^T\Lambda Q)d^2_{min}\\
&=d^2_{min}\sum_i \Lambda_{ii}.
\end{aligned}
\end{equation*}
\end{proof}

\begin{evcos}
Let $C=Z^TX$, i.e. $C_{ij}=\cos(\theta(Z_i,X_j))$, and let $\bar{C}$
be $C$ with diagonal elements removed. Assume
$\theta(Z_i,X_i)=\theta$ and $\sum^r_{j}C_{ij}^2\leq 1$, $\forall
i$, then
\begin{equation*}
(\cos^2(\theta)-\sqrt{r-1}\sin(2\theta))\cdot EV(V)\leq EV(X).
\end{equation*}
When $\theta$ is sufficiently small,
\begin{equation*}
(\cos^2(\theta)-O(\theta))\cdot EV(V)\leq EV(X).
\end{equation*}
\end{evcos}

\begin{proof}
Following the notations of the previous theorem,
\begin{equation*}
\begin{aligned}
&\quad \; \,tr(X^TA^TAX)\\
&\geq tr(X^TV\Lambda V^TX)\\
&=tr(X^TVR^TR\Lambda R^TRV^TX)\\
&=tr(C^TR\Lambda R^TC)\\
&=tr(R\Lambda R^TCC^T)\\
&=tr(R\Lambda R^T(I\cos (\theta)+\bar C)(I\cos (\theta)+\bar C^T))\\
&=tr(R\Lambda R^T(I\cos^2 (\theta)+(\bar C+\bar C^T)\cos (\theta)+\bar C\bar C^T))\\
&\geq tr(\Lambda)\cos^2 (\theta)+tr(R\Lambda R^T(\bar C+\bar
C^T))\cos (\theta).
\end{aligned}
\end{equation*}

We estimate the minimum eigenvalue $\lambda_{min}$ of the symmetric
matrix $S=\bar C+\bar C^T$. By Gershgorin circle theorem,
$|\lambda_{min}|\leq \sum^r_{j\neq i} |S_{ij}|$, $\forall i$, since
$S_{ii}=0$.
\begin{equation*}
\begin{aligned}
\sum^r_{j\neq i} |S_{ij}| &=\sum^r_{j\neq
i}|\cos(\theta(Z_i,X_j))+\cos(\theta(X_i,Z_j))|\\
&\leq \sum^r_{j\neq i}|\cos(\theta(Z_i,X_j))|+\sum^r_{j\neq
i}|\cos(\theta(X_i,Z_j))|\\
&\leq \sqrt{r-1}\bigl(\sum^r_{j\neq
i}|\cos(\theta(Z_i,X_j))|^2\bigr)^{-1/2}\\
&\quad+\sqrt{r-1}\bigl(\sum^r_{j\neq
i}|\cos(\theta(X_i,Z_j))|^2\bigr)^{-1/2}.
\end{aligned}
\end{equation*}

The last inequality holds since, $\forall x\in\mathbb{R}^p$,
$\|x\|_1\leq \sqrt p\|x\|_2$. Because $Z$ is the $r$ orthonormal
vectors, $\|Z^TX_i\|_2\leq\|X_i\|_2=1$, and
$Z^T_jX_i=\cos(\theta(X_i,Z_j))$, hence $\sum^r_{j\neq
i}|\cos(\theta(X_i,Z_j))|^2\leq 1-\cos^2(\theta)=\sin^2(\theta)$.
And by assumption, $\sum^r_{j}C_{ij}^2\leq 1$, so we also have
$\sum^r_{j\neq i}|\cos(\theta(Z_i,X_j))|^2\leq \sin^2(\theta)$.
Thus, $\sum^r_{j\neq i} |S_{ij}| \leq2\sqrt{r-1}\sin(\theta)$, and
$\lambda_{min}\geq -2\sqrt{r-1}\sin(\theta)$. Finally,
\begin{equation*}
\begin{aligned}
tr(X^TA^TAX)&\geq tr(\Lambda)\cos^2 (\theta)+tr(R\Lambda R^T(\bar
C+\bar C^T))\cos(\theta)\\
&\geq EV(V)\cos^2 (\theta)+EV(V)\lambda_{min}\cos (\theta)\\
&= \bigl(\cos^2 (\theta)-2\sqrt{r-1}\cos
(\theta)\sin(\theta)\bigr)EV(V)\\
&=\bigl(\cos^2 (\theta)-\sqrt{r-1}\sin(2\theta)\bigr)EV(V).
\end{aligned}
\end{equation*}

When $\theta$ is sufficiently small, such that $\sin(2\theta)\approx
2\theta$, we have $tr(X^TA^TAX)\geq (\cos^2
(\theta)-O(\theta))EV(V)$.
\end{proof}

\section{Deducing Original GPower from Matrix Approximation
Formulation}\label{sec:app gpower}

First, we give the original GPower. Fixing $Y$,
(\ref{equ:ori-GPower-l0}) and (\ref{equ:ori-GPower-l1}) have
solutions $W_i^*=H_{\sqrt {\lambda_i}}(A^TY_i)/\|H_{\sqrt
{\lambda_i}}(A^TY_i)\|_2$ and $W_i^*=S_{
\lambda_i}(A^TY_i)/\|S_{\lambda_i}(A^TY_i)\|_2$ respectively.
Substituting them into original objectives, the $\ell_0$ problem
becomes
\begin{equation}\label{equ:ori-GPower-l0-maxconv}
\max_{Y} \sum_i\sum_j\;[(A_j^TY_i)^2-\lambda_{i}]_+,\;s.t.\,Y^TY=I,
\end{equation}
and the $\ell_1$ problem becomes $\max_{Y}
\sum_i\sum_j\;[|A_j^TY_i|-\lambda_{i}]_+,\;s.t.\,Y^TY=I$. Actually,
it is to solve
\begin{equation}\label{equ:ori-GPower-l1-maxconv}
\max_{Y} \sum_i\sum_j\;[|A_j^TY_i|-\lambda_{i}]^2_+,\;s.t.\,Y^TY=I.
\end{equation}
Now the problem is to maximize two convex functions,
\cite{journee2010generalized} approximately solves them via a
gradient method which is generalized power method. The $t$th
iteration is provided by
\begin{equation}\label{equ:Y(t)-l0}
Y^{(t)}=Polar(AH_{\sqrt {\lambda_i}}(A^TY^{(t-1)})),
\end{equation}
and
\begin{equation}\label{equ:Y(t)-l1}
Y^{(t)}=Polar(AS_{ \lambda_i}(A^TY^{(t-1)})).
\end{equation}

We now see how these can be deduced from the matrix approximation
formulations (\ref{equ:GPower-l0}) and (\ref{equ:GPower-l1}). Split
$X$ into $X=WD$, $s.t.\|W_i\|_2=1$, $\forall i$ and $D>0$ is
diagonal matrix whose diagonal element $d_i$ in fact models the
length of the corresponding column of $X$. Then they become
\begin{equation}\label{equ:GPower-l0-D}
\begin{split}
\min_{Y,D,W}\;
&\|A-YDW^T\|^2_F+\sum_i\lambda_{i}\|W_i\|_0,\\
=&\|A\|_F^2+\sum_id^2_i-2\sum_id_iY_i^TAW_i+\sum_i\lambda_{i}\|W_i\|_0,\\
&s.t.\,Y^TY=I,\;D>0\text{ is diagonal},\;\|W_i\|_2=1,\;\forall i,
\end{split}
\end{equation}
and
\begin{equation}\label{equ:GPower-l1-D}
\begin{split}
\min_{Y,D,W}\;
&\frac{1}{2}\|A-YDW^T\|^2_F+\sum_i\lambda_{i}\|d_iW_i\|_1,\\
=&\frac{1}{2}\|A\|_F^2+\frac{1}{2}\sum_id^2_i-\sum_id_iY_i^TAW_i+\sum_i\lambda_{i}d_i\|W_i\|_1,\\
&s.t.\,Y^TY=I,\;D>0 \text{ is diagonal},\;\|W_i\|_2=1,\;\forall i.
\end{split}
\end{equation}

Fix $Y$ and $W$, and solve $D$. For the $\ell_0$ case,
$d_i^*=Y_i^TAW_i$. Substituting it back, we get
(\ref{equ:ori-GPower-l0}).

For the $\ell_1$ case, $d_i^*=Y_i^TAW_i-\lambda_i\|W_i\|_1$. Assume
$\lambda_i$ is sufficiently small, then $d_i>0$ is satisfied.
Substituting it back we get $\max_{Y,W}
\sum_i\big(Y_i^TAW_i-\lambda_{i}\|W_i\|_1\big)^2,\;
s.t.\,Y^TY=I,\,\forall i,\, \|W_i\|_2=1$. When we fix $Y$ and solve
$W$, under the previous assumption it is equivalent to
(\ref{equ:ori-GPower-l1}). Substituting $W_i^*=S_{
\lambda_i}(A^TY_i)/\|S_{\lambda_i}(A^TY_i)\|_2$ back, we obtain
(\ref{equ:ori-GPower-l1-maxconv}).

Finally, we can see that the solutions (\ref{equ:Y(t)-l0}) and
(\ref{equ:Y(t)-l1}) literally combine the two solution steps of
(\ref{equ:GPower-l0}) and (\ref{equ:GPower-l1}) respectively.

\section*{Acknowledgment}
This work was partly supported by National 973 Program
(2013CB329500), National Natural Science Foundation of China (No.
61103107 and No. 61070067), and Research Fund for the Doctoral
Program of Higher Education of China (No. 20110101120154).


\bibliographystyle{IEEEtran}
\bibliography{SPCArt}
\end{document}